\documentclass[11pt]{article}
\pdfoutput=1
\usepackage[letterpaper, left=1in, right=1in, top=1in,bottom=1in]{geometry}

\usepackage{etoolbox}
\newtoggle{acmformat}
\togglefalse{acmformat}

\usepackage{verbatim}
\usepackage{booktabs} 
\usepackage[utf8]{inputenc}
\usepackage{bm,subfigure}
\usepackage{epsfig,amstext,xspace}
\usepackage{algorithm}
\usepackage[noend]{algpseudocode}
\usepackage[utf8]{inputenc} 
\usepackage[T1]{fontenc}    
\usepackage{url}            
\usepackage{booktabs}       
\usepackage{amsfonts}       
\usepackage{bbm}
\iftoggle{acmformat}
{

}
{
  \usepackage{inconsolata}
  \usepackage{geometry}
  \usepackage{amsthm}
  \usepackage{hyperref}
  \usepackage{microtype}
  \usepackage{times}

}
\usepackage{upgreek}
\let\widebar\undefined
\input{widebar}

\usepackage{mathtools}
\usepackage{thm-restate}
\usepackage{cleveref}

\crefname{lemma}{lemma}{lemmas}
\Crefname{lemma}{Lemma}{Lemmas}

\usepackage{amsmath}
\usepackage{amsfonts}

\DeclarePairedDelimiter{\ceil}{\lceil}{\rceil}

\let\Pr\undefined

\DeclareMathOperator*{\En}{\mathbb{E}}

\DeclareMathOperator{\Pr}{\mathbb{P}}

\DeclareMathOperator*{\argmax}{arg\,max}

\iftoggle{acmformat}
{
    \newtheorem{remark}{Remark}[section]
    \newtheorem{claim}{Claim}[section]
    \newtheorem{fact}{Fact}[section]
}
{

\theoremstyle{plain}
\newtheorem{theorem}{Theorem}[section]
\newtheorem{proposition}[theorem]{Proposition}
\newtheorem{lemma}[theorem]{Lemma}
\newtheorem{fact}[theorem]{Fact}
\newtheorem{claim}[theorem]{Claim}

\newtheorem{corollary}[theorem]{Corollary}

\newtheorem*{theorem*}{Theorem}
\newtheorem*{lemma*}{Lemma}
\newtheorem{claim*}{Claim}

\theoremstyle{definition}
\newtheorem*{definition*}{Definition}
\newtheorem{definition}[theorem]{Definition}

\newtheorem{example}[theorem]{Example}

\theoremstyle{remark}
\newtheorem{remark}[theorem]{Remark}

\newtheorem*{remark*}{Remark}
\newtheorem*{property*}{Property}

\newtheorem*{problem*}{Problem formulation}

}

 \newcommand{\R}{\mathbb{R}}
 \newcommand{\N}{\mathbb{N}}

\newcommand{\ind}{\mathbbm{1}}    

\newcommand{\eps}{\epsilon}

\usepackage{tikz}
\usetikzlibrary{arrows}

\usepackage{nicefrac}

\newcommand{\traversal}{{\normalfont{\texttt{TRAVERSAL}}}\xspace}
\newcommand{\sfr}{\mathsf{r}}
\newcommand{\alg}{\mathsf{alg}}

\newcommand{\Exphall}{\Exp_{\mathrm{hal};\ell}}
\newcommand{\Prhall}{\Pr_{\mathrm{hal};\ell}}
\newcommand{\rhoprog}{\rho_{\mathrm{prog}}}
\newcommand{\Egoodl}{\mathcal{E}_{\mathrm{good};\ell}}
\newcommand{\deltafail}{\delta_{\mathrm{fail}}}
\newcommand{\Evisit}{\scrE_{\calU_{\ell}}}
\newcommand{\modgoodl}{\modclass_{\mathrm{good};\ell}}
\newcommand{\progress}[1][\ell]{\mathtt{PROGRESS}_{#1}} 
\newcommand{\Eexplorel}[1][\ell]{\progress[#1]}

\newcommand{\Prmucan}{\Prop_{\modvarhat \sim \mathrm{can}}}
\newcommand{\Expmucan}{\Expop_{\modvarhat \sim \mathrm{can}}}

\newcommand{\Prcanl}{\Pr_{\mathrm{can};\,\ell}}   
\newcommand{\gapcanl}{\gap_{\mathrm{can};\,\ell}} 

\usepackage[round]{natbib}

\usepackage{times}
\usepackage{nicefrac}
\usepackage{color}              

\newcommand{\OPT}{\mathtt{OPT}} 
\newcommand{\REG}{\ttbf{regret}}

\newcommand{\qpunish}{q_{\mathrm{pun}}}
\newcommand{\ralt}{r_{\mathrm{alt}}}
\newcommand{\Expsimhh}{\Expop_{\ledhall \sim \mathrm{hh}}}
\newcommand{\Prsimhh}{\Prop_{\ledhall \sim \mathrm{hh}}}

\newcommand{\Dhalnol}{\mathcal{D}_{\mathrm{hal}}}
\newcommand{\ledcens}{\ledger_{\mathrm{cens}}}
\newcommand{\Utot}{\mathcal{U}_{\mathrm{all}}}
\newcommand{\modpunish}{\mathcal{M}_{\mathrm{pun}}}

\newcommand{\reach}{\mathsf{Reach}}
\newcommand{\mdphh}{{\normalfont \texttt{HiddenHallucination}}\xspace} 
\newcommand{\HHandExploit}{{\normalfont \texttt{HH\&Exploit}}\xspace} 

\newcommand{\gapcan}{\gap_{\mathrm{can}}}
\newcommand{\Expcan}{\Exp_{\mathrm{can}}}

\newcommand{\ledvar}{\bm{\ledger}}

\newcommand{\ledrawl}[1][\ell]{\ledger_{\mathrm{raw};#1}}
\newcommand{\ledcensl}[1][\ell]{\ledger_{\mathrm{cens};#1}}
\newcommand{\ledul}[1][\ell]{\ledger_{\mathcal{U}_{\ell};\ell}}
\newcommand{\modhall}{\model_{\mathrm{hal};\ell}}
\newcommand{\ledrevk}[1][k]{\ledger_{#1}}

\newcommand{\ledhonl}{\ledger_{\mathrm{hon};\ell}}

\newcommand{\initOneLiners}{%
         \setlength{\itemsep}{0pt}
        \setlength{\parsep }{0pt}
          \setlength{\topsep }{0pt}
}
\newenvironment{OneLiners}[1][\ensuremath{\bullet}]
    {\begin{list}
        {#1}
        {\initOneLiners}}
    {\end{list}}

\newcommand{\EqComment}[1]{\text{\emph{(#1)}}}

\newcommand{\ie}{{\em i.e.,~\xspace}}

\newcommand{\eg}{{\em e.g.,~\xspace}}
\newcommand{\Eg}{{\em E.g.,~\xspace}}

\newcommand{\rbr}[1]{\left(\,#1\,\right)}
\newcommand{\sbr}[1]{\left[\,#1\,\right]}
\newcommand{\cbr}[1]{\left\{\,#1\,\right\}}

\newcommand{\LDOTS}{\, ,\ \ldots\ ,}     
\newcommand{\poly}{\operatornamewithlimits{poly}}

\numberwithin{equation}{section}

\newcounter{protocol}

\makeatletter



\DeclareMathAlphabet\mathbfcal{OMS}{cmsy}{b}{n}

\Crefname{claim}{Claim}{Claims}
\Crefname{equation}{Eq.}{Eqs.}
\Crefname{protocol}{Protocol}{Protocols}
\Crefname{asm}{Assumption}{Assumptions}

\Crefname{condition}{Condition}{Conditions}

\newcommand{\calC}{\mathcal{C}}
\newcommand{\calX}{\mathcal{X}}

\newcommand{\nipsvspace}[1]{\iftoggle{nips}{\vspace{1}}{}}

\renewcommand{\Pr}{\mathbb{P}}
\newcommand{\Exp}{\mathbb{E}}

\newcommand{\iidsim}{\overset{\mathrm{i.i.d.}}{\sim}}

\newcommand{\I}{\mathbf{1}}

\newcommand{\calF}{\mathcal{F}}

\newcommand{\algcomment}[1]{\textcolor{blue!70!black}{\footnotesize{\texttt{\textbf{#1}}}}}

\newcommand{\supp}{\mathrm{supp}}
\newcommand{\xah}{(x,a,h)}
\newcommand{\rtil}{\tilde{r}}

\newcommand{\ttbf}[1]{{\normalfont \texttt{\textbf{#1}}}}

\newcommand{\opfont}[1]{\ttbf{#1}}

\newcommand{\bluepar}[1]{\vspace{1mm}\noindent\textbf{\textcolor{blue!70!black}{#1}}}

\newcommand{\modvar}{\bm{\upmu}}

\newcommand{\EvPun}{\calE_{\mathrm{pun},\,\ell}}

\newcommand{\gap}{\mathrm{Gap}}
\newcommand{\valuename}{\ttbf{value}}
\newcommand{\scrE}{\texttt{E}}

\newcommand{\nlearn}{n_{\mathrm{lrn}}}
\newcommand{\nphase}{n_{\mathrm{ph}}} 

\newcommand{\prior}{\ttbf{p}}
\newcommand{\posteriorhall}{\prior_{\mathrm{hal},\ell}} 

\newcommand{\orac}{\opfont{orac}}

\newcommand{\cens}{\opfont{cens}}
\newcommand{\Phase}{\opfont{phase}}

\newcommand{\modclass}{\mathcal{M}}
\newcommand{\modtotal}{\modclass_{\mathrm{mdp}}}
\newcommand{\modelst}{\model_{\star}}

\newcommand{\valuef}[2]{\valuename(#1;#2)}

\newcommand{\Pitotal}{\Pi_{\mathrm{mkv}}}

\newcommand{\modst}{\model_{\star}}
\newcommand{\sfE}{\mathsf{E}}

\newcommand{\calE}{\mathcal{E}}

\newcommand{\calK}{\mathcal{K}}

\newcommand{\ledhall}[1][\ell]{\ledger_{\mathrm{hal};#1}}

\newcommand{\Expop}{\operatornamewithlimits{\ensuremath{\mathbb{E}}}}   
\newcommand{\Prop}{\operatornamewithlimits{\ensuremath{\mathbb{P}}}}    

\newcommand{\calD}{\mathcal{D}}

  \newcommand{\Prcan}{\Pr_{\mathrm{can}}}

  \newcommand{\ledgespace}{\mathcal{L}}
  \newcommand{\ledgespacetot}{\widebar{\ledgespace}}

\newcommand{\modclasshall}[1][\ell]{\modclass_{#1}}
\newcommand{\kexpl}[1][\ell]{k^{\mathrm{hal}}_{#1}}

\newcommand{\censicon}{\mathrm{cs}}

\newcommand{\ledgercens}{\ledger_{\censicon}}

\newcommand{\ledgespacecenstotm}{\ledgespacetot_{\censicon;m}}

\newcommand{\ledgecens}{\ledgercens}

\newcommand{\epsp}{\varepsilon_{\mathsf{p}}}
\newcommand{\epsr}{\varepsilon_{r}}

\newcommand{\Klearn}{\mathcal{K}_{\mathrm{lrn}}}

\newcommand{\bmx}{\mathbf{x}}

\newcommand{\bma}{\mathbf{a}}

\newcommand{\pmodst}{\sfp_{\modst}}
\newcommand{\rmodst}{r_{\modst}}
\newcommand{\lonenorm}[1]{\|#1\|_{\ell_1}}
\newcommand{\calN}{\mathcal{N}}

\newcommand{\sfP}{\mathsf{P}}

\newcommand{\modvarhat}{\hat{\model}}

\newcommand{\sfp}{\mathsf{p}}

\newcommand{\thetar}{\theta_r}
\newcommand{\thetap}{\theta_{\sfp}}
\newcommand{\epspunish}{\varepsilon_{\mathrm{pun}}}

\newcommand{\ledgespacetotm}{ \ledgespacetot_m}

\newcommand{\Pimarkov}{\Pi_{\mathrm{mkv}}}

\newcommand{\subsets}{\mathrm{subsets}}
\newcommand{\censinv}{\cens^{-1}}
\newcommand{\calU}{\mathcal{U}}
\newcommand{\sfR}{\mathsf{R}}
\newcommand{\kh}{_{k;h}}
\newcommand{\nn}{\nonumber}

\newcommand{\Prmech}{\mathbb{P}_{\mathrm{mech}}}

\newcommand{\Ecubed}{\texttt{E3}}

\newcommand{\ledger}{\uplambda} 
\newcommand{\signal}{\upsigma}  
\newcommand{\traj}{{\uptau}}    
\newcommand{\model}{\upmu}      
\newcommand{\signals}{\Upsigma} 

\newcommand{\underexplored}[1][\ell]{\calU^{\mathrm{und}}_{#1}} 

\renewcommand{\refeq}[1]{Eq.~(\ref{#1})}

\newlength\tindent
\usepackage[suppress]{color-edits}
\addauthor{ms}{orange}
\addauthor{as}{red}

\title{Exploration and Incentives in Reinforcement Learning%
\footnote{Compared to the initial version, this version refactors presentation and spells out implications for regret (Section~\ref{sec:exploit}).}}

\author{Max Simchowitz\thanks{MIT, \texttt{msimchow@berkeley.edu}. Research conducted while the author was a graduate student at UC Berkeley and an intern at Microsoft Research NYC.} \and Aleksandrs Slivkins\thanks{Microsoft Research NYC, \texttt{slivkins@microsoft.com}}}

\date{First version: February 2021\\This version: February 2023}

\begin{document}

\maketitle
\vspace{-8mm}

\begin{abstract}
How do you incentivize self-interested agents to \emph{explore} when they prefer to \emph{exploit}? We consider complex exploration problems, where each agent faces the same (but unknown) MDP. In contrast with traditional formulations of reinforcement learning, agents control the choice of policies, whereas an algorithm can only issue recommendations. However, the algorithm controls the flow of information, and can incentivize the agents to explore via information asymmetry. We design an algorithm which explores all reachable states in the MDP. We achieve provable guarantees similar to those for incentivizing exploration in static, stateless exploration problems studied previously. To the best of our knowledge, this is the first work to consider mechanism design in a stateful, reinforcement learning setting.





\end{abstract}

\addtocontents{toc}{\protect\setcounter{tocdepth}{0}}



\section{Introduction}
How do you incentivize self-interested agents to \emph{explore} when they prefer to \emph{exploit}? We revisit the tradeoff between exploration and exploitation, \ie between costly acquisition of information and using this information to make good near-term decisions. While algorithmic aspects of this tradeoff have been extensively studied in machine learning and adjacent disciplines, we focus on its economic aspects. We consider a population of self-interested agents which collectively face the exploration-exploitation tradeoff. The agents should explore for the sake of the common good, yet any given agent is not inherently incentivized to suffer the costs of exploration for the sake of helping others. As a result, exploration may happen very slowly, or not at all. This can be remedied by an online platform such as a recommendation system, which may wish to balance exploration and exploitation. Such platform can only provide information and recommendations, but cannot force the agents to comply. However, the platform may aggregate information from many agents in the past, and control what is disclosed to agents in the future. This information asymmetry provides the platform with leverage to incentivize exploration. Designing recommendation algorithms that incentivize exploration has been studied previously, starting from \citep{Kremer-JPE14,Che-13}; we will refer to this work as \emph{incentivized exploration}. All prior work focuses on a basic exploration problem in which each agent acts once, and her actions do not affect the outcomes for actions chosen in the future.%
(If all agents are controlled by an algorithm, this problem is known as \emph{multi-armed bandits}.)

We initiate the study of incentives in much more complex exploration problems that arise in reinforcement learning (RL). Each agent has multiple interactions with the environment. The chosen actions have a persistent effect: there is \emph{state} which is probabilistically affected by the agent's actions and in turn affects the agent's rewards for these actions. A standard abstraction for such interactions is a Markov Decision Process (\emph{MDP}), which posits that the effect of the past is completely summarized by the current state. We follow the paradigm of \emph{episodic RL}: each agent faces a fresh copy of the same MDP, constituting one ``episode" of the problem. This MDP is not known in advance, but may be approximately learned over multiple episodes. RL algorithms choose the course of action in a particular episode (usually called a \emph{policy}), and adjust this policy from one episode to another so as to balance exploration and exploitation. However, in our problem, called \emph{incentivized RL}, it is the agents themselves who control the choice of policies, whereas an algorithm can only issue recommendations, leveraging the information asymmetry discussed above.

For motivation, let us consider three stylized examples in which agents have repeated, MDP-like interactions with the environment, and need to be incentivized to explore.
First, consider automated market proxies, such as bidding agents in ad auction,  algorithmic traders on a stock exchange, or pricing agents for retail buyers or sellers. Their actions affect their own endowment or inventory, and possibly also the state of the market (\eg for thin markets). The parameter settings for such proxies can be viewed as an MDP policy. Many/most market participants tend to be ``small", participating infrequently or in rare bursts, and therefore lacking ``inherent" incentives to explore. Thus, an online market may wish to incentivize exploration among such participants. Second, a navigation app suggests driving directions, whereas the driver chooses which route to take. The driver makes multiple decisions as he is driving, which change his location. This can be modeled as an MDP, with driver's location as ``state" and routes as ``policies". Exploration is needed to find out which routes are better at a particular time, but each driver just wants to get to this destination soon. Third, a medical treatment received by a patient may consist of multiple steps which depend on the patient's current ``state". This scenario can be viewed as an MDP, where the medical treatment is a ``policy". Exploration, a.k.a. clinical trials, is commonly used to learn which treatments work best. However, patients tend to prefer treatments that work best, and need to be incentivized to explore for the sake of others (see Section 1.4 in \cite{ICexploration-ec15} for a discussion).


\bluepar{(Simplified) problem formulation.}
A problem formulation in incentivized exploration consists of three components: a machine learning problem collectively solved by the agents, strategic interactions between the algorithm and the agents, and a specific performance objective pursued by the algorithm.  As we delve into incentivized RL, we begin with simple and arguably fundamental versions of all three components.

We use finite, deterministic MDPs without time discounting.%
\footnote{Restriction to \emph{deterministic} MDPs, \ie MDPs with deterministic rewards and transitions, captures the gist of the problem and simplifies presentation. In Section~\ref{sec:extension}, we extend our setup and results to \emph{randomized} MDPs. }
We have $S$ states, $A$ actions, and $H$ stages, denoted $x\in[S]$, $a\in[A]$ and $h\in[H]$. The parameters $S,A,H$ are finite and fixed throughout. An agent interacts with the MDP in stages: in each stage $h\in[H]$, the agent observes the current state $x_h \in [S]$, selects an action $a_h\in [A]$, receives reward $r_h\in[0,1]$ and transitions to a new state $x_{h+1}$. The outcome is agent's \emph{trajectory}, a sequence of tuples
    $(x_h,\,a_h,\, r_h,\,h)_{h\in[H]}$.
An \emph{MDP model} $\model$ specifies the reward
    $\sfr_{\model}\xah$
and the next state for each $\xah$ triple, as well as the initial state $x_0$. A policy $\pi:[S]\times[H]\to [A]$ determines the agent's behavior in the MDP, specifying an action for each state $x$ and stage $h$. The value of a policy $\pi$ under model $\model$, denoted $\valuef{\pi}{\model}$, is defined as the total reward when this policy is executed under this model. Formally:
\begin{align*}
\valuef{\pi}{\model}
    := \textstyle \sum_{h\in[H]} r_h
    = \sum_{h\in[H]} \sfr_{\model}(x_h,\,a_h,\, r_h),
\end{align*}
where the next action is $a_{h+1} = \pi(x_h, h)$ for each stage $h\in [H]$. In \emph{Episodic RL}, there is a fixed but unknown MDP model $\modst$ and $K$ \emph{episodes}. In each episode $k\in [K]$, an algorithm chooses a policy $\pi_k$, this policy is executed in the MDP, and the resultant trajectory is observed. Each episode starts from the same initial state $x_0$.
\footnote{So, an execution of a given policy in a given episode does not depend on the policies from the previous episodes.}

On the economics side, we have a new agent in each episode $k$, and the algorithm must ensure that following its chosen policy $\pi_k$ is in the agent's self-interest. The formal requirement is
\begin{align}\label{eq:greedy-intro}
\pi_k \in \argmax_{\text{policies $\pi$}}\;
    \En_{\modst\sim\prior}\sbr{ \valuef{\pi}{\modst} \mid \pi_k},
\end{align}
where the (true) MDP model $\modst$ is initially drawn from a Bayesian prior $\prior$ over MDP models.
\footnote{The Bayesian framing here is (merely) a way to endow agents with well-defined incentives.}
This is a version of Bayesian incentive-compatibility (\emph{BIC}), a standard requirement in economic theory. The intuition behind \refeq{eq:greedy-intro} is that the agent observes policy $\pi_k$ before her episode starts, and makes a one-time decision whether to follow this policy. In this decision, the agent chooses among all policies, evaluates them based on their Bayesian-expected value, and breaks ties in favor of $\pi_k$. The agent knows the algorithm, the prior $\prior$ and her episode $k$, but she does not know anything about the previous episodes. The algorithm must ensure that the agent chooses $\pi_k$.

The performance objective is \traversal: visit all $\xah$ triples in the smallest number of episodes.%
\footnote{Naturally, we are only interested in $\xah$ triples that are \emph{reachable} by some RL algorithm without incentives constraints.} This  provides sufficient data for policy optimization, possibly with a different reward metric \citep{jin2020reward}.
While \traversal does not attempt to optimize agents' welfare, the resulting data can be used afterwards for exploitation. To elucidate this point, we derive specific corollaries for regret minimization in Section~\ref{sec:exploit}.

Incentivized exploration in bandits is a special case with $H=1$ stages. Specialization of \traversal (visit each arm in the smallest number of rounds) has been studied in this context  \citep[\eg][]{ICexploration-ec15,Selke-PoIE-ec21}. Some initial samples of each arm are required to bootstrap all other published algorithms for randomized rewards. Therefore, all these algorithms are preceded by a ``warm-up stage" which collects the initial samples and, essentially, optimizes for \traversal.

\bluepar{Challenges.}
Incentivized exploration is challenging even in bandits. For one, why would an agent choose an arm she does not initially prefer, especially the first time it is recommended? Moreover, if the Bayesian prior is independent across arms, a given arm cannot be recommended before all arms with larger prior mean rewards are explored. Due to these and similar difficulties, algorithms for incentivized exploration require specially tailored algorithms, as standard algorithms from bandits are not compatible with agents' incentives. Revealing full data collected by the algorithm would not help convince the agents to follow  recommendations: instead, they would only exploit ({\ie choose} the actions with the greatest posterior mean reward). Even the most basic objective of exploring one other arm is not always achievable, requires non-trivial solutions, and is very subtle to optimize. (We expand on all these points in Related Work, see Section~\ref{sec:related}.) Returning to incentivized RL, we conclude that we should not expect  standard RL algorithms to be incentive-compatible, and \traversal is a natural first objective to consider.

RL is known to be much more difficult than multi-armed bandits as a machine learning problem, for many reasons. Below we list three specific challenges that are most relevant to incentivized RL. First, the effective number of alternatives from which an agent may choose -- the number of all Markovian policies -- is exponential in $S$, the number of states. Second, an algorithm cannot reliably explore a particular $\xah$ pair even without incentives issues, because it does not know a priori which policies will visit which states. In contrast, a bandit algorithm can just pull any desired arm. Third, expected rewards associated with different policies are necessarily highly correlated, even in the frequentist version, because many different policies may visit the same $\xah$ pair. Whereas rewards of different arms in bandits are not correlated in the frequentist version, and can be assumed mutually independent in the Bayesian version as a paradigmatic special case.

\newcommand{\nArmsMAB}{\tilde{A}} 

Incentivized RL can be trivially reduced to incentivized exploration in bandits by treating each MDP policy as an ``arm" and each episode as a ``round". This reduction is exponentially wasteful, as it creates $\nArmsMAB = A^{SH}$ arms. Moreover, the relevant state-of-art result for incentivized exploration in bandits with correlated priors \citep[Section 5]{ICexploration-ec15} does not provide explicit guarantees for a super-constant number of arms, and requires $e^{\Omega(\nArmsMAB)}$ rounds to explore each arm even once in some natural examples.%
\footnote{Recently, \citet{Selke-PoIE-ec21} achieved $\poly(\nArmsMAB)$ dependence if the prior is mutually independent across arms. However, this result does not help the reduction at hand, because the mean rewards of policies are highly correlated. }
{Reduction} to this result requires $K$ episodes, where $K$ is exponential in the number of policies. This is \emph{doubly exponential} in problem parameters:
$K=\exp\rbr{ A^{SH} }$. Moreover, the prior-dependent parametrization in the guarantee from \citep[Section 5]{ICexploration-ec15} lacks a natural interpretation when `arms' correspond to MDP policies.

Let's calibrate our desiderata {for incentivized RL}.
The effective number of different actions in the MDP is $SAH$, \ie $A$ actions for each $(x,h)$ pair. So, $\exp(SAH)$ episodes in incentivized RL would be in line with the exponential dependence in prior work, and would vastly improve over the trivial reduction {outlined above}.


\bluepar{Results and techniques.}
We design an algorithm for incentivized RL which achieves \traversal
in $K$ episodes, where $K$ is bounded in terms of parameters $S,A,H$ and the prior. We obtain $K$ with at most exponential dependence on $SAH$ when the rewards $\sfr_{\model}\xah$ are independent across the $\xah$ triples and jointly independent of the transitions (but the transitions can still be arbitrarily correlated among themselves). Note that policy values are still highly correlated.%
\footnote{This holds because (i) different policies may visit the same $\xah$ triples and (ii) the transitions may be correlated.}
We extend this result to randomized MDPs and (with a more abstract guarantee) to arbitrary priors without the independence assumption.

In terms of techniques, it is helpful to distinguish between the challenges of exploration in RL, that is, learning which policies visit which states in an MDP, and those of  \emph{incentivized} exploration in RL. We handle  RL exploration by building on the classic \Ecubed{} algorithm \citep{kearns2002near}. This algorithm encourages exploration of new $\xah$ triples by ``punishing" the previously explored $\xah$ triples and ``promoting" the unexplored ones, \ie pretending that the rewards from former are small and rewards from the latter are large, and computing a reward-optimizing policy in the resulted MDP. Our main technical contribution is to ``implement" a version of the \Ecubed{} algorithm within the constraints of incentivized RL.

Our results are theoretical, focusing on the exponential dependence on $SAH$ (and vastly improving over the prior work, as discussed above). Note that $\exp(SAH)$ episodes may be practical for sufficiently small MDPs.
Regarding implementation, each individual episode can be made computationally efficient under suitable independence assumptions. However, the fact that computation in Bayesian MDPs is not well-understood impedes efficient implementation in the general case.


\bluepar{Discussion and limitations.}
We focus on \emph{Markovian} policies: functions that input the state and the stage, but ignore what happened in the past stages. This is traditional in the study of MDPs and does not lose generality in the planning problem when the MDP is fully known. With a Bayesian posterior, the optimal MDP policy may be non-Markovian; equivalently, an agent may want to revise the policy mid-episode. However, committing to follow a Markovian policy may be behaviorally justified. Indeed,
it may reflect the lack of sophistication or resources to optimize beyond Markovian policies,
exogenous constraints (\eg the available medical treatments are Markovian and must be chosen in advance), or considerations of convenience (\eg following the driving directions vs. charting one's own course). In all these cases, agents may be content to know that the suggested policy is better for them than any other Markovian policy.


We make standard assumptions from theoretical economics: the principal has the ``power to commit" to the declared algorithm, and the agents are endowed with Bayesian rationality, \ie they maximize their expected utility given available information, and have sufficient knowledge and computational power to do so. Almost all prior work on incentivized exploration makes these assumptions.%
\footnote{One exception is \citep{Jieming-unbiased18}, which restricts the algorithm's structure so that weaker economic assumptions suffice.}
Likewise, our algorithm requires a detailed knowledge of the Bayesian prior; this is a standard assumption in incentivized exploration, and more generally in the literature on information design in economics.%
\footnote{The only exception in prior work on incentivized exploration is one of the algorithms in \citep{ICexploration-ec15}, which summarizes the prior with just two numbers. However, this result requires the prior to be independent across arms (and only applies to bandits).}

While we study a basic model of incentivized RL, all three components thereof can be extended in various directions (see \Cref{sec:conclusions}). We view our work as a foundation towards further study.



\bluepar{Some notation.}
Let $\Pimarkov$ be the set of all policies as defined above, \ie all Markovian policies.
The support of the prior $\prior$, \ie the class of of all feasible MDP models, is denoted $\modtotal$. Let us recap the ``time units" used throughout this paper: in each \emph{episode} an MDP is executed for $H$ \emph{stages} (standard terminology in episodic RL); our algorithm operates in \emph{phases} of a fixed number of episodes each; we use \emph{rounds} only when we discuss specialization to multi-armed bandits.

\bluepar{Map of the paper.}
To simplify presentation, we first specify our algorithm for the special case of  deterministic MDPs (Section~\ref{sec:alg}), and analyze it for the special case of deterministic MDPs and independent priors (Section~\ref{sec:analysis-basic}). This case is already highly non-trivial and captures our main ideas. Then in Section~\ref{sec:extension} we extend our algorithm and analysis to the general case, with randomized MDPs and arbitrary priors. Regret implications are spelled out in Section~\ref{sec:exploit}. All proofs are (only) outlined in the body of the paper, with lengthy details deferred to the appendices.

\subsection{Related Work}\label{sec:related}
Incentivized exploration was introduced in \citet{Kremer-JPE14} and \citet{Che-13}. All prior work focused on multi-armed bandits as the underlying machine learning problem. The problem is quite rich even under a basic economic model adopted in our paper: prior work has studied
optimal policies for deterministic rewards \citep{Kremer-JPE14,Cohen-Mansour-ec19},
regret-minimizing policies for stochastic rewards \citep{ICexploration-ec15}, and
exploring all arms \citep{ICexploration-ec15,ICexplorationGames-ec16}. Extensions of this model included
heterogenous agents \citep{Jieming-multitypes18,Kempe-colt18},
agents directly affecting each other's payoffs 
\citep{ICexplorationGames-ec16},
information leakage \citep{Bahar-ec16,Bahar-ec19},
relaxed economic assumptions
\citep{Jieming-unbiased18}, time-discounted utilities \citep{Bimpikis-exploration-ms17}, monetary incentives \citep{Frazier-ec14,Kempe-colt18}, and continuous information flow \citep{Che-13}. A survey on incentivized exploration can be found in
\citep[][Ch. 11]{slivkins-MABbook}. 

Standard approaches from bandits do not carry over to incentivized exploration without non-trivial modifications, major assumptions, substantial performance loss, and new analyses. Any bandit algorithm can be made BIC by ``hiding'' it among many rounds of exploitation \citep{ICexploration-ec15}.
Successive Elimination~\citep{EvenDar-icml06} carries over if the rule for eliminating suboptimal arms is revised to depend on the prior \citep{ICexploration-ec15}.
Thompson Sampling is BIC when primed with a certain amount of data that needs to be collected exogenously \citep{Selke-PoIE-ec21}. In the first two results, performance loss compared to bandits is exponential in the number of arms. The last two results assume that the prior is independent across arms, which really breaks for Incentivized RL.


The basic objective of exploring all arms is very subtle in incentivized exploration (while trivial in bandits). This objective is not always achievable, even for two arms: it may be impossible to explore arm $2$ when arm $1$ is preferred initially. 
(A simple example: two arms with mean rewards $\mu_1>\mu_2$ and a Bayesian prior on $(\mu_1,\mu_2)$ such that $\mu_1$ is independent of $\mu_1-\mu_2$ \citep{ICexploration-ec15}.) Exploring arm $2$ requires an assumption: that arm $2$ can appear optimal, with some positive probability, after seeing enough samples of arm $1$. Absent such assumptions, the objective can be refined to exploring all arms that the agents can possibly be incentivized to explore.  Achieving either version of the objective takes non-trivial techniques and analyses \citep{ICexploration-ec15,ICexplorationGames-ec16}. Achieving this objective \emph{optimally} is even more subtle. Exact optimality is achieved only for deterministic rewards and only up to three arms \citep{Kremer-JPE14,Cohen-Mansour-ec19}. For randomized rewards, one can achieve optimality up to a polynomial dependence on the lower bound, and only for independent priors \citep{Selke-PoIE-ec21}. The minimal number of rounds needed to visit arm $2$ (with any algorithm) can be arbitrarily large depending on the prior \citep{Selke-PoIE-ec21}.

Revealing the algorithm's full history to each agent implements the ``greedy" bandit algorithm which always exploits. This algorithm suffers from linear Bayesian regret, caused by herding on a suboptimal arm. This is a common case: it holds for any Bayesian prior and even for two arms \citep[][Ch. 11]{slivkins-MABbook}. However, the greedy algorithm performs well under strong assumptions on agents' heterogeneity and the structure of rewards \citep{kannan2018smoothed,bastani2017exploiting,Greedy-Manish-18,AcemogluMMO19}.

Incentivized exploration is closely related to two important topics in theoretical economics,
Bayesian Persuasion
\citep{BergemannMorris-survey19,Kamenica-survey19}
and social learning
\citep{Horner-survey16,Golub-survey16}. The former focuses on a single-round interaction between a principal and agent(s), and the latter studies how strategic agents learn over time in a shared environment. Connection between exploration and incentives arises in other domains:
dynamic auctions
    \cite[\eg][]{AtheySegal-econometrica13,DynPivot-econometrica10,Kakade-pivot-or13},
ad auctions
    \cite[\eg][]{MechMAB-ec09,DevanurK09,Transform-ec10-jacm},
human computation
    \cite[\eg][]{RepeatedPA-ec14,Ghosh-itcs13,Krause-www13},
and competition between firms
    \cite[\eg][]{bergemann2000experimentation,keller2003price,CompetingBandits-merged}.


We consider a standard paradigm of reinforcement learning (RL): episodic RL with tabular MDPs. Other RL paradigms study MDPs with specific helpful structure such as linearity, MDPs with partially observable state (POMDPs), algorithms that interact with a single MDP throughout, and algorithms that learn over different MDPs (\emph{meta-learning}). While we focus on \emph{exploration}, considerable attention has also been devoted to \emph{planning}
\ie policy optimization given the oracle access to the MDP or full knowledge thereof. The literature on RL is vast and rapidly growing, see the book draft \citep{RLTheoryBook-20} for background.

Tabular episodic RL has been studied extensively over the past two decades; standard references include
\cite{kakade2003sample,kearns2002near,brafman2002r,strehl2006pac,strehl2009reinforcement}.
Optimal solutions have recently been obtained for unknown MDPs, both for policy optimization \citep{dann2017unifying} and for regret minimization \citep{jaksch2010near,azar2017minimax,dann2015sample}.
Our objective of exploring all reachable states is studied in \citep{kearns2002near,brafman2002r}. Moreover, it is essentially a sub-problem in \emph{reward-free RL} \citep{jin2020reward}, where one collects enough data to enable policy optimization relative to any given matrix of rewards. Sample complexity guarantees in prior work are primarily derived in frequentist settings (even when analyzing Bayesian algorithms like Thompson Sampling, as in \citet{pmlr-v40-Gopalan15}). Nevertheless, Bayesian framework, standard for modeling agents' incentives, also informs many practical RL algorithms \citep{ghavamzadeh2016bayesian}.

Multi-armed bandits can be seen as a special case of episodic RL with $H=1$ stages. They received much attention as a basic model for explore-exploit tradeoff, \eg see books \citep{Gittins-book11,Bubeck-survey12,slivkins-MABbook,LS19bandit-book}.



\section{Our Algorithm: Hidden Hallucination}\label{sec:alg}

Let us describe \mdphh, our algorithm for \traversal objective (\Cref{alg:MDP_HH}). 
Recall that we focus on deterministic MDPs for now (and randomized MPDs are treated in Section~\ref{sec:extension}).

\bluepar{Preliminaries: signals.} Our algorithm is more naturally described in a slightly more general model of incentivized RL (which reduces to the one described in the Introduction). In this model, each episode $k$ proceeds as follows:
\begin{OneLiners}
\item[1.] the algorithm chooses a signal $\signal_k$ from some fixed set $\signals$ of possible signals;
\item[2.] agent $k$ arrives, observes $k$ and $\signal_k$, and chooses a policy $\pi_k$;
\item[3.] this policy is executed in the MDP;
the algorithm observes the resultant trajectory $\traj_k$.
\end{OneLiners}
No other information is revealed to the principal or the agents. The algorithm for choosing the signals is known to the agents.%
\footnote{This algorithm is often called \emph{signaling policy}; we avoid this term to prevent confusion with ``MDP policies".}
The set $\signals$ of feasible signals can be an arbitrary countable set.

Each agent $k$ chooses the policy so as to maximize her conditional expected reward given what she knows: the round $k$, the observed signal $\signal_k$, the algorithm, the prior, and the selection rule for the previous agents. More precisely, we make an inductive definition: for each agent $k=1,2, \ldots$
\begin{align}\label{eq:greedy}
\pi_k \in \argmax_{\pi \in \Pimarkov}
    \Exp\sbr{ \valuef{\pi}{\modst} \mid \signal_k},
\end{align}
with some fixed tie-breaking, and this selection rule is known to all agents. In this definition, the signal $\signal_k$ is treated as a random variable over $\signals$, whose distribution is determined by the algorithm, the prior, and the selection rule for the previous agents. We condition on a particular realization of this random variable, as chosen by the principal and observed by the agent. This conditioning yields a conditional distribution over models $\modst$, which is the distribution that $\valuef{\pi}{\cdot}$ in \eqref{eq:greedy} is integrated over.
\footnote{The selection rule \eqref{eq:greedy} merely specifies an agent's rational response to a given signal under Bayesian uncertainty over $\modst$, when the alternatives are $\pi\in\Pimarkov$ and the rewards are $\valuef{\pi}{\modst}$. How these objects are defined is unimportant to this definition. On this level of abstraction, this is a very standard setup in theoretical economics.}

To reduce this model to the version from the Introduction (where each signal $\signal_k$ must be a policy), the principal can compute the policy $\pi_k$ in \eqref{eq:greedy} and recommend this policy directly, changing the signal to $\pi_k$. Then \refeq{eq:greedy-intro} would hold,
by a standard ``direct revelation argument" \citep{Kremer-JPE14}.

We use the model with signals because this is what our algorithm directly specifies, whereas it is unclear how explicitly specify the resulting policies $\pi_k$ in \eqref{eq:greedy}.

\bluepar{Key ideas.}
We would like to  ``punish" already-visited $\xah$ triples by making their rewards appear very small, so as to incentivize exploration of the remaining $\xah$ triples. Such ``punishment" traces back to \Ecubed{} algorithm \citep{kearns2002near}, the classic RL algorithm mentioned in the Introduction. Our algorithm implements this ``punishment" in a way consistent with agents' incentives.

Essentially, in some episodes agents should observe histories where the true rewards are replaced with small ``hallucinated" rewards. Such \emph{hallucination episodes} should be randomly hidden among many \emph{honest episodes}, in which the rewards are revealed faithfully. The agents should not know which type of episode they are in, and the probability of the hallucination episode should be low enough for the agents to (mostly) trust the hallucinated rewards.

A naive implementation of these ideas would reveal the entire history to each agent (replacing the rewards in the hallucination episodes, as discussed above). Unfortunately, this approach may reveal more than intended: agents in hallucination episodes may be able to use the trajectories from the honest episodes to infer something about rewards. This is because the agents would second-guess why the algorithm selected particular policies observed from trajectories in the past. Such second-guessing might make the agents suspect that they are in a hallucination episode, and therefore distrust the observed rewards. It appears difficult to prove that this issue does not prevent the agents from exploring.

Therefore, we only reveal \emph{partial} histories (which we call \emph{ledgers}), and design them to be interpretable in a particularly simple way. Generically, a Bayesian update for any given subset of observations may depend both on the observations and the algorithm used to collect these observations. We enforce a formal property that the Bayesian update on the revealed ledger does not depend on the algorithm. This property, called \emph{ledger hygiene}, is spelled out in Section~\ref{sec:canon}. For its sake, the ledgers only describe the past hallucination episodes, while all honest episodes are left out. The rewards in the ledger are true or hallucinated, depending on the current episode.


Another issue, which comes up for correlated priors (and also for randomized MDPs), is that the hallucinated rewards must look \emph{plausible}. For example, if an $\xah$ triple has a low-mean Bernoulli reward, then a high reward is likely to be observed if this triple is seen sufficiently often. Moreover, if two $\xah$ triples have strongly correlated mean rewards, then the hallucinated rewards should be correlated similarly. For these reasons, one cannot just hallucinate a low reward for each $\xah$ triple in the ledger.
Therefore, we hallucinate \emph{an entire MDP} with low mean rewards, and then resample the rewards according to this MDP, for all $\xah$ triples in the ledger. The hallucinated MDP must be consistent with the transition data revealed in the ledger, so we sample it from the appropriately defined conditional distribution.

Ledger hygiene and MDP hallucination ensure that agents in hallucination episodes take their observed rewards at face value. Since the hallucinated ledger actively punishes rewards from already-visited $\xah$ triples, it incentivizes the new agent towards the remaining $\xah$ triples. Consequently, the agent selects \emph{some} policy which advances exploration.


\bluepar{Algorithm specification: phases and ledgers.}
The algorithm proceeds in phases of $\nphase$ episodes each. In each phase $\ell$, it selects the \emph{hallucination episode} $\kexpl$ uniformly at random from all episodes. 
The algorithm proceeds indefinitely; we bound a sufficient number of phases in the analysis.

In each episode in phase $\ell$,  the algorithm reveals trajectories from all past-phase hallucination episodes.%
\footnote{An episode is called \emph{past-phase} if it has occurred in one of the preceding phases.}
The set of all these episodes is denoted
    $\calK_{\ell} := \{\kexpl[\ell']: \ell' \in [\ell-1]\}$.
A triple $\xah$ is called \emph{fully-explored} at phase $\ell$  if it is visited at least once
in the past-phase hallucination episodes $\calK_{\ell}$, and \emph{under-explored} otherwise. Reward information for fully-explored triples is either revealed faithfully (in honest episodes), or hallucinated. Reward information for under-explored triples is \emph{censored}: not revealed to the agents.


To make this precise, we define a \emph{ledger}: the sequence of policy-trajectory pairs from all past-phase hallucination episodes $k\in \calK_{\ell}$, in which some of the reward information may be altered. In a formula, a ledger for phase $\ell$ is a tuple
    $ \rbr{ (\pi_k,\traj'_k):\; k\in \calK_{\ell}} $,
where $\pi_k$ is a policy chosen in episode $k$, and $\traj'_k$ is some trajectory which has the same transitions as the true episode-$k$ trajectory $\traj_k$, but may modify or remove some of the rewards.%
\footnote{Removing a reward for a given stage of the trajectory can be implemented by replacing it with a special symbol such as $\bot$.}
We consider four types of ledgers, depending on how the rewards are treated:
\begin{itemize}
\item the \emph{raw ledger} $\ledrawl$ retains all information on rewards;
\item the \emph{censored ledger} $\ledcensl$ removes all reward information;
\item the \emph{honest ledger} $\ledhonl$
 retains reward information for fully-explored $\xah$ triples,
 but removes the rewards for all under-explored $\xah$ triples.
\item the \emph{hallucinated ledger} $\ledhall$
     removes reward information from all under-explored $\xah$ triples,
    and modifies (\emph{hallucinates}) the rewards for all fully-explored $\xah$ triples.
\end{itemize}


\noindent On hallucination episodes $\kexpl$, we reveal the hallucinated ledger $\ledhall$. All other episodes implement \emph{exploitation} by revealing the honest ledger $\ledhonl$. 

The remaining crucial ingredient is how to generate the hallucinated rewards. As we discussed in ``Key Ideas", directly hallucinating low rewards for specific $\xah$ pairs might appear non-plausible to the agents. Instead, we hallucinate the entire MDP. Specifically, we define the \emph{punish-event} $\EvPun$ that for all triples $\xah$ that are fully explored at phase $\ell$, the rewards are smaller than a given parameter
$\epspunish>0$. That is:
\begin{align}\label{eq:punish-event-defn}
\EvPun = \cbr{ \sfr_{\modst}\xah \le \epspunish: \;
                \text{all triples $\xah$ that are fully explored at phase $\ell$}}.
\end{align}
Consider \refeq{eq:algo-halposterior-defn}, the posterior distribution of the true model $\modst$ given $\EvPun$ and the censored ledger $\ledcensl$, \ie the transition data. We draw from this posterior to hallucinate an MDP model $\modhall\in \modtotal$. Finally, we use $\modhall$ to hallucinate rewards for each fully-explored $\xah$ triple.

\begin{algorithm}[h]
  	\begin{algorithmic}[1]
  	\State{}\textbf{Input: }
        phase length $\nphase$,
        punishment parameter $\epspunish>0$
    \For{each phase $\ell= 1,2,\;\dots$}
    \State{}$\Phase_{\ell} = (\ell-1) \nphase  + [\nphase] \subset \N$
        ~~~\algcomment{\% the next $\nphase$ episodes}
    \State{}Draw ``hallucination episode" $\kexpl$ uniformly from $\Phase_{\ell}$
    \For{each episode $k \in \Phase_{\ell}$}
    \If{$k = \kexpl$}
    \algcomment{\qquad\% hallucination episode}
    \State{}\algcomment{\% censored ledger $\ledcensl$,
        punish-event $\EvPun$ with parameter $\epspunish$}
    \State{}Define distribution $\posteriorhall$ over MDP models by
        \algcomment{\quad\% hallucinated posterior}
        \begin{align}\label{eq:algo-halposterior-defn}
         \posteriorhall(\modclass) := \Pr\sbr{\modst \in \modclass \mid \ledcensl,\; \EvPun},
            \quad\forall \text{ measurable }\modclass\subset\modtotal.
         \end{align}
    \State{}Draw MDP model $\modhall$ at random from $\posteriorhall$\label{line:modhall}
        \algcomment{\quad\% hallucinated MDP}
     \State{}For each fully-explored $\xah$ triple,
        \algcomment{\quad\% hallucinate rewards}
     \State{}\hspace{2em}each time this triple appears in the ledger,
     \State{}\hspace{2em}draw its reward as prescribed by $\modhall$.
     \State{}Form $\ledhall$ by inserting the hallucinated rewards into $\ledcensl$
        \algcomment{\quad\% hallucinated ledger}
     \State{}Reveal hallucinated ledger: $\ledrevk \gets \ledhall$
    \Else{} \algcomment{\qquad\% exploitation}
     \State{}Reveal honest ledger: $\ledrevk \gets \ledhonl$.
    \EndIf
    \State{}Observe the trajectory $\traj_k$ from this episode.
  \EndFor
  \EndFor
  \end{algorithmic}
  \caption{\mdphh}
  \label{alg:MDP_HH}
	\end{algorithm}




\bluepar{Efficient computation.} The only computationally intensive step in \Cref{alg:MDP_HH} is the hallucination of $\modhall$ from distribution  \eqref{eq:algo-halposterior-defn}. This step can be made computationally efficient if the Bayesian prior is independent across the $\xah$ triples (jointly for rewards and transitions). This is because one need only condition on observed transitions from each $\xah$ triple independently to sample rewards. Similarly, if (joint) rewards are independent from (joint) transitions (even though the rewards themselves can be arbitrarily correlated across $\xah$ triples, and the same for transitions), then rewards can be sampled directly from the prior. We do not provide an efficient implementation for the general case, mainly because computation Bayesian MDPs is not yet well-understood (\eg for policy optimization, Bayesian update, and posterior sampling).

\bluepar{Comparison to prior work.}
Each hallucination episode (which implements exploration) is hidden among many honest episodes which implement exploitation. This extends the ``hidden exploration" technique from incentivized exploration in bandits
(\citep{ICexploration-ec15}, also see \citep[][Ch. 11]{slivkins-MABbook}), where each exploration round is randomly hidden among many exploitation rounds. (In fact, this technique underlies the trivial reduction discussed in the Introduction.) The difference is that an exploration round can directly recommend any desired action. In contrast, a hallucination episode recommends a very particular policy, constructed indirectly as the agent's reaction to a carefully designed ledger. This indirect construction is what drives the exponential improvement over the trivial reduction.

/
\section{Analysis: Deterministic MDPs and Independent Priors}
\label{sec:analysis-basic}

We present analysis for the basic case of deterministic MDPs and reward-independent priors (defined below), retaining the main ideas while avoiding some heavy technicalities. The general case is deferred to Section~\ref{sec:extension}.

\begin{definition}[Reward-Independence]\label{defn:reward_independence} A prior $\prior$ over $\modtotal$ is called \emph{reward-independent} if mean rewards $r_{\modst}\xah$ are independent across $\xah$ triples, and jointly independent of the transition probabilities. (The transition probabilities can be arbitrarily correlated across the $\xah$ triples.)
\end{definition}

We capture the dependence on the prior via these two parameters:
\begin{align}\label{eq:fmin_argmin}
f_{\min}(\eps) &:= \min_{x,a,h}\quad
    \Pr\sbr{ \sfr_{\modst}\xah \le \eps}
\quad\text{and}\quad
r_{\min} := \min_{x,a,h}\quad
    \Exp\sbr{ \sfr_{\modst}\xah }.
\end{align}
Both parameters take the worst case across all $\xah$ triples. Here, $r_{\min}$ is simply a uniform lower bound on the prior mean rewards. Because of reward-independence, the probability of the punish-event is lower-bounded as
    $\Pr\sbr{\EvPun} \geq \rbr{f_{\min}(\epspunish)}^{-SAH}$.

Our main result guarantees \traversal in the number of episodes that is exponential in $SAH$.

\begin{theorem}\label{thm:det_mdp}
Consider a reward-independent prior $\prior$ supported on deterministic MDPs. Assume that $r_{\min}>0$ and
    $\calC  := f_{\min}\rbr{ \epspunish}>0$,
where
    $\epspunish = r_{\min}/2H$.
Consider \Cref{alg:MDP_HH} with punishment parameter $\epspunish$,  and
    $\nphase = \ceil{ 6H\,r_{\min}^{-1}\,\calC^{-SAH} }$.
The algorithm visits a new $\xah$ triple in every phase, until all reachable triples are visited. This takes at most this many episodes:
    \[ K = \calC^{-SAH}\cdot O\rbr{ SAH^2\, r_{\min}^{-1} }.\]
\end{theorem}

\begin{remark}The two assumptions in \Cref{thm:det_mdp} are in line with prior work. Indeed, let us specialize to multi-armed bandits with independent priors, and let $\sfr(a)$ denotes the mean reward of arm $a$. Then $\xah$ triples specialize to arms, which are trivially $\rho$-explorable. Exploring all arms in this setting requires $r_{\min}>0$ and $f_{\min}(r_{\min})>0$, according to the characterization in \citet{Selke-PoIE-ec21}.
\end{remark}

The rest of this section proves \Cref{thm:det_mdp}. On a high level, an agent is much more likely to face an exploitation episode than a hallucination one. So, even when shown the hallucination ledger, the agent would believe that most likely it is the honest ledger. Therefore, the agent would believe that the rewards from fully-explored $\xah$ triples are indeed small, and consequently select policies which aim to explore under-explored $\xah$ triples.

The formal proof is structured as follows. In \Cref{sec:canon}, we discuss \emph{ledger hygiene}, a crucial property of our ledgers which underpins the rest of the analysis. \Cref{sec:one_step} analyzes a single phase of \mdphh (\Cref{prop:mdp_hh}). Building on that, \Cref{sec:det_transitions} provides a self-contained proof of \Cref{thm:det_mdp}. The details are deferred to the appendices.

\subsection{Canonical Posteriors and Ledger Hygiene}
\label{sec:canon}

We want legers to be interpretable on the face value, regardless of the process used to generate it. We capture this property via the notion of \emph{canonical posterior} $\Prcan\sbr{ \cdot \mid \ledger }$, whereby we pretend that a given ledger $\ledger$ is constructed by an algorithm which deterministically chooses the policies in $\ledger$. Then, we turn to \emph{random ledgers}, \ie ledger-valued random variables. We define \emph{ledger hygiene}, a  property which asserts that the posterior given a random ledger is in fact the canonical posterior. While this is a non-trivial property, we show that the censored and honest ledgers in \mdphh satisfy this property.

We introduce an abstract notion of a \emph{partially-censored ledger} $\ledger$. Formally,
it is any dataset of a particular shape: a sequence of (policy, partial-trajectory pairs) $(\pi_k,\tau_k)$,  along with a \emph{censoring set} $\calU_{\ledger} \subset [S]\times[A]\times[H]$. The reward information is censored out from all trajectories for all triples $\xah\in \calU_{\ledger}$; that is, each $\tau_k$ is of the form $(x_h,\,a_h,\, \tilde{r}_h,\,h)_{h\in[H]}$, where $\tilde{r}_h \in [0,1]$ records a reward for uncensored triples $(x_h,a_h,h) \notin \calU_{\ledger}$, but $\tilde{r}_h = \bot$ replaces the reward value with a special censoring symbol for triples in $\calU_{\ledger}$.
 The ledger is called \emph{totally-censored} if
    $\calU_\ledger = [S]\times[A]\times[H]$, so that no reward information is recorded.

Consider a partially-censored ledger $\lambda$ which stores $n$ policy-trajectory pairs, and let $\pi_1 \LDOTS \pi_n$ be the respective policies. Consider a new, non-adaptive algorithm which interacts with MDP model $\modst$, proceeds for $n$ episodes, chooses policy $\pi_k$ in each episode $k\in [n]$, and observes some trajectory, denoted $\traj_k$. Let $\Lambda$ be a random ledger with the same censoring set $U_\lambda$ and with policy-trajectory pairs
    $(\pi_k,\traj'_k)$, $k\in[n]$,
where $\traj'_k$ denotes trajectory $\traj_k$ in which all reward data for triples $\xah\in U_\lambda$ is censored out.

We define $\Prcan\sbr{ \cdot \mid \ledger }$ as a distribution over $\modtotal$ obtained  by conditioning the prior $\prior$ on the event
$\cbr{\Lambda = \ledger}$:
\begin{align}\label{eq:canon-posterior-defn}
\Prcan\sbr{\modclass \mid \ledger }
    &:= \Prop \sbr{ \modst \in \modclass \mid \Lambda = \ledger }
    \quad\forall \text{ measurable }\modclass\subset\modtotal.
\end{align}
For a particular event $\modclass$,
    $\Prcan\sbr{\modclass \mid \ledger }$
is called the \emph{canonical probability} given $\ledger$. This probability is ``canonical'' in the sense that it corresponds to the distribution of a non-adaptive algorithm which places all of its past trajectories (appropriately censored) into the ledger $\Lambda$. Given a measurable function $f:\modtotal\to \R$, we define the \emph{canonical expectation}
    $\Expcan\sbr{f(\modst)\mid \ledger}$
given $\ledger$ as the expectation of $f(\cdot)$ over distribution $\Prcan\sbr{ \cdot \mid \ledger }$.
Now we are ready to define ledger hygiene.

\begin{definition}\label{def:hygiene}
A random ledger $\ledger$ is called \emph{hygienic} if it satisfies
\begin{align}\label{eq:hygiene}
 \Pr\sbr{\modst \in \modclass \mid \ledger}  = \Prcan[\modclass \mid \ledger ]
 \quad\forall \text{ measurable }\modclass\subset\modtotal.
\end{align}
\end{definition}

One can construct numerous examples of non-hygenic ledgers, see \Cref{app:canonical_vs_mechanism}. \Eg, policies in the ledger could encode more information about $\modst$ than the canonical posterior can extract.

\begin{lemma}\label{lem:data_hygiene}
Censored ledger $\ledcensl$ and honest ledger $\ledhonl$ are both hygienic.
\end{lemma}

\begin{proof}[Proof Sketch]
The essential property we use is that the policies in $\ledcensl$, and the censoring set $\calU_{\ell}$ for $\ledhonl$, are determined exactly by $\ledcensl[\ell-1]$, that is, the visited triples $\xah$ from previous hallucination episodes. Thus, $\ledcensl$ and $\ledhonl$ do not depend on transition data which are not in their own ledgers, and do not communicate reward information (because reward information is only explicitly used on non-hallucination episodes). The formal proof is given in \Cref{proof:lem_data_hygiene}.
\end{proof}

\subsection{Analysis for a Single Phase}
\label{sec:one_step}


We derive conditions under which an agent selects a policy from some desired subset $\Pi \subset \Pitotal$.


To state these conditions, we introduce the notion of \emph{canonical gap}. The \emph{canonical value} of policy $\pi$ given ledger $\ledger$ is defined as
    $\Expcan[\valuef{\pi}{\modst} \mid \ledger]$.
The canonical gap for policy set $\Pi$ is the difference in maximal canonical value between $\Pi$ and its complement.

\begin{definition}[Canonical Gap]\label{defn:canonical_gap}
The canonical gap of policy set $\Pi \subset \Pimarkov$ given ledger $\ledger$ is
\begin{align*}
\gapcan[\Pi \mid \ledger] :=
    \max_{\pi \in \Pi}\; \Expcan[\valuef{\pi}{\modst} \mid \ledger]
    \;-\;
    \max_{\pi \not\in \Pi}\; \Expcan[\valuef{\pi}{\modst} \mid \ledger]
\end{align*}
\end{definition}


The meaning  behind this definition is that if an agent were to observe a hygienic ledger $\ledger$ with a positive canonical gap, this agent would choose some policy in $\Pi$.

Our algorithm construct hallucinated ledgers $\ledger = \ledhall$ so as to yield a positive canonical gap, and we'd like to conclude that an agent would choose some policy in $\Pi$ in the hallucination episodes (which would constitute progress towards our exploration goal). Unfortunately, random ledgers $\ledger_k$ revealed by our algorithm are not hygienic, precisely because of hallucination episodes. We circumvent this issue if
    $\gapcan\sbr{\Pi \mid \ledhall}$
is not only positive, but much larger than $1/\nphase$, where $\nphase$ is the phase duration. We formulate a condition which also depends on the punish-event $\EvPun$:
\begin{align}\label{eq:HH_good_condition}
3H/\nphase \le \Prcan\sbr{\EvPun \mid \ledcensl} \cdot \gapcan\sbr{\Pi \mid \ledhall}.
\end{align}
The essence of Hidden Hallucination as a technique is that this condition suffices.


\begin{proposition}[Hidden Hallucination]\label{prop:mdp_hh} Let $\Pi \subset \Pimarkov$ be any subset of policies. Fix phase $\ell$ in the algorithm. A policy in $\Pi$ is chosen in the hallucination episode if condition \eqref{eq:HH_good_condition} holds.
\end{proposition}

Full proof is in \Cref{proof:prop_mdp_hh}. To use this proposition, we will establish uniform (\ie data-independent) lower bounds on both terms on the right-hand side of \eqref{eq:HH_good_condition}, and set $\nlearn$ accordingly.

\begin{proof}[Proof Sketch]
Fix episode $k$. The proof readily reduces to bounding the conditional probability that $k$ is the hallucination episode, given the revealed ledger $\ledger_k$. Intuitively, we need to bound the agents' belief that they are facing a hallucinated ledger rather than an honest one. Our argument is inspired by the analysis of hidden exploration in \cite{ICexploration-ec15}, but differs in a number of respects; notably, the role of ledger hygiene.

Recall that the censored ledger $\ledcensl$ comprises all data observed by agent $k$ that is known to be faithfully transmitted by the algorithm. We condition on $\ledcensl$, and verify that
\begin{align}\label{eq:Prhall_ledhall}
\Pr[\; k=\kexpl \mid \ledcensl \;] = \Pr[\; k\neq\kexpl \mid \ledcensl,\, \EvPun\;].
\end{align}
In words, agent $k$ finds it equiprobable that she is in a hallucination episode and that she is shown an honest ledger but the true model lies in the punish-event.

When $1/\nphase$ is very small, the probability of $k = \kexpl$ is dominated by the probability of $\EvPun$, and \Cref{eq:HH_good_condition} quantifies exactly how small $1/\nphase$ must be. Note that this condition is stated in terms of canonical conditionals, which suffices because the censored and honest ledgers $\ledhonl$ and $\ledcensl$ are hygienic.
\end{proof}

\subsection{Full proof of \Cref{thm:det_mdp}}
\label{sec:det_transitions}


Let us start with some notation. Recall that fully-explored (resp., under-explored) $\xah$ triples are now simply the ones that have (resp., have not) been visited at least once during the past-phase hallucination episodes. Let $\underexplored$ denote the set of all $\xah$ triples that are under-explored in a given phase $\ell$. Let $\Pi_{\ell}$ be the set of all policies $\pi\in \Pimarkov$ which deterministically visit some triple $\xah\in \underexplored$ under the true model $\modst$. Note that $\Pi_\ell$ is non-empty if and only if some under-explored $\xah$ triple is reachable.



We apply \Cref{prop:mdp_hh} with policy set $\Pi = \Pi_\ell$, so as to guarantee that an agent selects some policy in $\Pi_\ell$ in the hallucination episode, and therefore visits some under-explored $\xah$ triple. Then all reachable $\xah$ triples will be visited after at most $SAH$ phases, \ie in at most $SAH \cdot \nphase$ episodes.

We lower-bound the two terms on the right-hand side of \refeq{eq:HH_good_condition}. First, 
reward-independence implies that
\begin{align*}
\Prcan[\EvPun \mid \ledcensl] \ge f_{\min}(\epspunish)^{SAH}.
\end{align*}
Hence, it remains to bound the canonical gap:

\begin{claim}\label{claim:gap_size}
Suppose $\epspunish \le r_{\min}/2H$ and fix phase $\ell$. Then
    $\gapcan\sbr{\Pi_{\ell} \mid \ledhall} \geq r_{\min}/2$.
\end{claim}

Then \Cref{prop:mdp_hh} follows with the choice of $\nlearn$ as in \Cref{thm:det_mdp}, and we are done.

It remains to verify \Cref{claim:gap_size}. To this end, we analyze
    $\Prcan[\;\cdot \mid \ledhall]$,
the canonical posterior given the hallucinated ledger. We prove that any model in the support of this distribution can partially \emph{simulate} the trajectory of any given policy under the true model $\modst$.

\begin{lemma}[Deterministic Simulation Lemma]\label{lem:modclass_not}
Fix phase $\ell$ and policy $\pi \in \Pitotal$. Fix any model $\model$ in the support of
    $\Prcan[\;\cdot \mid \ledhall]$.
Let $h_\pi$ be the first stage at which policy $\pi$ visits any under-explored $\xah$ triple under the fake model $\model$; let $h_\pi=H+1$ this never happens. Then:
\begin{OneLiners}
\item[(a)]
the $\xah$ triples visited by policy $\pi$ in stages $h\leq h_\pi$ are identical under $\modst$ and $\model$.
\item[(b)]
all rewards collected by policy $\pi$ in stages $h<h_\pi$ under model $\model$ are at most $\epspunish$.
\end{OneLiners}
\end{lemma}

\begin{proof}
Since the prior $\prior$ is supported on deterministic MDPs, the fake model $\model$ must have identical transitions and rewards to those in the hallucinated ledger $\ledhall$. By construction, $\ledhall$ contains transitions for all fully-explored $\xah$ triples, and they come from the true model $\modst$. Hence the transitions in models $\model$ and $\modst$ are the same for all stages $h < h_\pi$. Part (a) follows by induction over stages $h \le h_\pi$.

To prove part (b), fix any $\xah$ triple visited by policy $\pi$ in stage $h\leq h_\pi$ under the fake model $\mu$. This triple is fully-explored by part (a). Consequently, this triple is assigned expected reward at most $\epspunish$ in the hallucinated posterior \eqref{eq:algo-halposterior-defn}. Since we only deal this deterministic MDPs, this expected reward propagates as the observed reward in the hallucinated ledger $\ledhall$, and then into the fake model $\model$.
\end{proof}

\begin{proof}[Proof of \Cref{claim:gap_size}]
Denote the hallucinated ledger as $\ledger = \ledhall$. For brevity, we write \emph{fake models} to refer to all models $\model$ in the support of
    $\Prcan[\;\cdot \mid \ledger]$.

First, fix any policy $\pi\not\in \Pi_\ell$. By definition, its trajectory under the true model $\modst$ never runs out of fully-explored $\xah$ triples. By \Cref{lem:modclass_not}, its total reward under any fake model $\model$ is at most $\epspunish H$, which is at most $r_{\min}/2$ by our choice of $\epspunish$. It follows that
\begin{align}\label{eq:claim:gap_size-1}
(\forall \pi\not\in\Pi_\ell) \qquad
    \Expcan[\valuef{\pi}{\modst} \mid \ledger] \leq r_{\min}/2.
\end{align}

Second, fix any policy $\pi \in \Pi_{\ell}$. By definition, its trajectory under the true model $\modst$ visits some under-explored $\xah$ triple; let's focus on the first such triple. Take any fake model $\model$. By \Cref{lem:modclass_not}, policy $\pi$ visits $\xah$ under $\model$, too. Consequently,
    $\valuef{\pi}{\model} \geq \sfr_{\model}(x,a,h)$.
It follows that
\begin{align}
(\forall \pi\in\Pi_\ell)\qquad
\Expcan[\valuef{\pi}{\modst} \mid \ledger]
    &\geq \Expcan[\sfr_{\model}(x,a,h) \mid \ledger] \nonumber \\
    &\geq \ralt
    &\EqComment{by definition of $\ralt$} \nonumber\\
    &= r_{\min}
    &\EqComment{by \Cref{rem:reduction-to-independence}}.
    \label{eq:claim:gap_size-2}
\end{align}
The claim follows by plugging \eqref{eq:claim:gap_size-1} and \eqref{eq:claim:gap_size-2} into the definition of the canonical gap.
\end{proof}

\section{Extensions: Randomized MDPs and Arbitrary Priors}
\label{sec:extension}

\bluepar{Preliminaries.}
In a randomized MDP, rewards and transitions are randomized as follows. An \emph{MDP model} $\model$ specifies, for each $\xah$ triple,
the reward distribution
    $\sfR_{\model}\xah$
over $[0,1]$ with mean reward
    $\sfr_{\model}\xah$,
and transition probabilities
    $\sfp_{\model}\rbr{\cdot \mid x,a,h} \in \Delta(S)$.
In each stage $h\in[H]$, reward $r_h$ is drawn independently from
    $\sfR_{\model}(x_h,\, a_h,\, h)$,
and the new state $x_{h+1}\in[S]$ is drawn independently from distribution
    $\sfp_{\model}\rbr{\cdot \mid x_h,a_h,h}$.
The initial state $x_1$ is drawn independently from distribution
    $\sfp_{\model}\rbr{\cdot \mid 0}$.
For ease of exposition, we posit that all reward distributions are supported on the same countable set; the paradigmatic example is Bernoulli rewards. The shape of reward distributions does not matter for this paper, only the mean rewards do.

As the MDP is non-deterministic, we introduce $\sfP_{\model}^{\pi}$ and $\sfE^{\pi}_{\model}$ to denote expectation and operators corresponding to trajectories $\traj$ generated by policy $\pi$ and model $\model$. The policy value is defined as the expected total reward of this policy under $\sfE^{\pi}_{\model}$,
\begin{align*}
\valuef{\pi}{\model} := \textstyle \sfE_{\modst}^{\pi}\sbr{\sum_{h\in[H]} r_h}.
\end{align*}
The \traversal objective needs to be modified as well, because some $\xah$ triples might only be reachable with some (small) probability. Given $\rho\in [0,1]$, we say a triple $\xah$ is called \emph{$\rho$-reachable} under the true MDP $\modst$,  if some policy $\pi \in \Pitotal$ reaches it with probability at least $\rho$, \ie
    $\sfP_{\modst}^{\pi}\sbr{ (\bmx_h,\bma_h) = (x,a)} \ge \rho$.
Note that whether a particular triple $\xah$ is $\rho$-reachable is determined only by the state-step pair $(x,h)$.
Our objective is parameterized by this $\rho$, as well as multiplicity $n\in\N$ and confidence $\delta\in(0,1)$.

\begin{definition}
We say that an algorithm $\alg$  has $(\rho,n)$-traversed $\modst$ by episode $K$ if each $\rho$-reachable $\xah$ triple (under $\modst$) is visited in at least $n$ episodes $k \le K$. Here, an $\xah$ triple is called \emph{visited} in a given episode $k$ if $ (x_{k;h},a_{k;h},h) = \xah$. We say algorithm $\alg$ satisfies $(\rho,n,\delta,K)$-\traversal \ if
\begin{align}
\Pr[\alg \text{ has $(\rho,n)$-traversed } \modst \text{ by episode } K] \ge 1 -\delta.
\end{align}
\end{definition}

\noindent The objective is to achieve $(\rho,n,\delta,K)$-\traversal in smallest number of episodes $K$. Multiple visits ($n>1$) are usually desired in order to estimate mean rewards or transition probabilities.

\bluepar{Our algorithm.} \Cref{alg:MDP_HH} carries over with minor modifications. There is an additional parameter, target number of samples $\nlearn$. A triple $\xah$ is called \emph{fully-explored} at phase $\ell$  if it is visited at least $\nlearn$ times in the past-phase hallucination episodes $\calK_{\ell}$, and \emph{under-explored} otherwise. Recall that reward information for under-explored triples is always \emph{censored}: not included in the ledgers revealed to the agents. The punish-event is defined via \eqref{eq:punish-event-defn}, same as before, but $\sfr_{\model}\xah$ in this definition now refers to mean rewards rather than deterministic rewards. To define hallucinated rewards, each time any fully-explored $\xah$ triple appears in the ledger, we draw its reward independently from the reward distribution specified by the $\modhall$, the hallucinated MDP model.

\bluepar{Results: reward-independent priors.}
Our guarantees are most lucid for reward-independent priors, as per \Cref{defn:reward_independence}. As in the previous section, we guarantee \traversal in the number of episodes that is exponential in $SAH$. While the guarantees become more complex to account for randomization in the MDP, they use the same parameters $f_{\min}$ and $r_{\min}$, see \refeq{eq:fmin_argmin}.

\begin{theorem}\label{thm:main_indep}
Consider a reward-independent prior $\prior$. Fix parameters $\rho, \delta\in (0,1]$. Assume that $r_{\min}>0$ and
    $\calC_{\rho}  := f_{\min}\rbr{ \epspunish}>0$,
where
    $\epspunish = r_{\min}\,\rho\,/\, 18\,H$.
Consider \Cref{alg:MDP_HH} with punishment parameter $\epspunish$, appropriately chosen phase length $\nphase$, and large enough target $n=\nlearn$. This algorithm is guaranteed to $(\rho,n)$-explore with probability at least $1-\delta$ by episode $K_{\rho,n}$, where $n$ and $K_{\rho,n}$ are specified below.


For some absolute constants $c_1,c_2$, it suffices to take
\begin{align*}
n = \nlearn &\ge
     c_1\cdot \rho^{-2}\, r_{\min}^{-2}\,
     H^4 \rbr{ SAH \log \calC_{\rho}^{-1} +
        \log \frac{ SAH}{\delta\rho\, r_{\min}} }, \\
K_{\rho,n} &= c_2\cdot n\cdot \calC_{\rho}^{-SAH} \cdot \rho^{-3}\, r_{\min}^{-3}\, SAH^4.
\end{align*}
In particular, for any $n\geq 1$, one can obtain
    $K_{\rho,n} \leq n\cdot  \calC_{\rho}^{-SAH}
        \cdot \poly\rbr{\rho^{-1}\, r_{\min}^{-1}\, SAH  }
        \cdot \log\rbr{\delta^{-1}\,\calC_{\rho}^{-1}}$.
\end{theorem}

\bluepar{Results: arbitrary priors.}
Our analysis in fact extends to arbitrary priors, and yields \Cref{thm:main_indep} as a corollary. The prior-dependent parameters now need to be defined using the machinery developed in \Cref{sec:canon}: conditionally on a given partially-censored ledger and via canonical posterior/expectation (to avoid the dependence on how this larger was constructed).



\begin{align}
\qpunish(\eps) &:=
    \min_{\text{totally-censored ledgers $\ledger$}}\quad
        \Prcan\sbr{\sfr_{\modst}\xah \le \eps \;\;
            \text{for all $\xah$ triples} \mid \ledger},
            \label{eq:qpun-defn}\\
\ralt &:=
    \min_{\text{partially-censored ledgers $\ledger$}}\quad
    \min_{\xah \in \calU_{\ledger}}
    \Expcan\sbr{ \sfr_{\modst}\xah \mid \ledger }.
        \label{eq:ralt-defn}
\end{align}
In particular, the canonical posterior probability of the punish-event given the censored ledger in the algorithm,
    $\Prcan\sbr{ \EvPun\mid \ledcensl}$,
is uniformly lower-bounded by $\qpunish(\epspunish)$.
Likewise, $\ralt$ lower-bounds the posterior mean reward
    $\Expcan\sbr{ \sfr_{\modst}\xah \mid \ledger }$
given any ledger $\ledger$ that censors out this particular $\xah$ pair.

\begin{remark}\label{rem:reduction-to-independence}
These parameters can be easily related to those in \refeq{eq:fmin_argmin} under reward-independence, essentially because the conditioning on $\ledger$ vanishes. First, $\ralt=r_{\min}$. Second, the probability in \eqref{eq:qpun-defn} factorizes as
    $\prod_{\xah}\; \Prop_{\modst \sim \prior} \sbr{\sfr_{\modst}\xah \le \eps}$.
It follows that
$\qpunish(\epspunish) \geq f_{\min}^{-SAH}(\epspunish)$.
\end{remark}

The main (and most general) result of this paper, stated below, has the same shape as \Cref{thm:main_indep} for reward-independence, but
    $r_{\min}$ and $f_{\min}^{-SAH}(\epspunish)$
are replaced with, resp., $\ralt$ and $\qpunish(\epspunish)$. Accordingly, \Cref{thm:main_indep} follows immediately by Remark~\ref{rem:reduction-to-independence}. The significance of this result is that it reduces the task of designing and analyzing  mechanisms for  incentivized exploration to the task of analyzing parameters $\ralt$ and $\qpunish(\epspunish)$ for a particular prior.

\begin{theorem}\label{thm:main_prob_mdp}
Consider an arbitrary prior $\prior$. Fix parameters $\rho, \delta\in (0,1]$. Assume that $\ralt>0$ and
    $\qpunish = \qpunish\rbr{\epspunish}>0$,
where
    $\epspunish = \ralt\,\rho\,/\, 18\,H$.
Consider \Cref{alg:MDP_HH} with punishment parameter $\epspunish$, appropriately chosen phase length $\nphase$, and large enough target $n=\nlearn$. This algorithm is guaranteed to $(\rho,n)$-explore with probability at least $1-\delta$ by episode $K_{\rho,n}$, where $n$ and $K_{\rho,n}$ are specified below.


For some absolute constants $c_1,c_2$, it suffices to take
\begin{align}
n=\nlearn &\ge
     c_1\cdot \rho^{-2}\, \ralt^{-2}\,H^4
        \rbr{ S + \log \frac{SAH}{\delta\rho\cdot\ralt\cdot \qpunish} },
        \label{eq:thm:main_prob_mdp-n} \\
K_{\rho,n} &= c_2\cdot n\cdot \qpunish \cdot \rho^{-3}\, \ralt^{-3}\, SAH^4.
    \label{eq:thm:main_prob_mdp-K}
\end{align}
In particular, for any $n\geq 1$ one can obtain
 $K_{\rho,n} \leq n\cdot \qpunish
        \cdot \poly\rbr{\rho^{-1}\, \ralt^{-1}\, SAH  }
        \cdot \log\rbr{\delta^{-1}\,\qpunish^{-1}}$.
 \end{theorem}

\begin{remark}\label{rem:replace-qpun}
{Parameter $\qpunish(\epspunish)$ in \Cref{thm:main_prob_mdp} can be replaced with any number $L>0$ which lower-bounds
    $\Prcan\sbr{ \EvPun\mid \ledcensl}$
for all phases $\ell$.}
\end{remark}

\subsection{Proof overview for \Cref{thm:main_prob_mdp}}
\label{sec:probab_proof}




The ledger hygiene (\Cref{sec:canon}) and the single-step analysis (\Cref{sec:one_step}) carry over as is. We strive to mimic the rest of the analysis from the previous section (\ie \Cref{sec:det_transitions}): we carry over the major steps and deal with various complications that arise due to randomness. Some of the details are deferred to the appendices.

Let us start with some notation. Fix phase $\ell$. Let $\calF_{\ell}$ denote the $\sigma$-algebra generated by all randomness (in rewards, transitions, and the algorithm) in all phases up to and including this phase.  As before, let $\underexplored$ denote the set of all $\xah$ triples that are under-explored in this phase. Let $\Pi_{\ell,q}$ be the set of policies $\pi\in \Pimarkov$ which visit some triple $\xah\in \underexplored$ with probability at least $q$ (under the true model $\modst$).

We are interested in the event that the algorithm makes progress in a given phase $\ell$. Specifically, let
$\progress$ be the event that the algorithm visits some triple $\xah\in\underexplored$ in the hallucination episode in this phase. The progress is only probabilistic, expressed via a rather subtle statement:
\begin{align}\label{eq:genCase-progress}
\Pr\sbr{\Eexplorel \mid \calF_{\ell-1}} \geq \rhoprog
\qquad\text{with probability at least $1-\deltafail$ over $\calF_{\ell-1}$}.
\end{align}
Given the parameters from the theorem statement, the parameters in \eqref{eq:genCase-progress} are chosen as follows:
\begin{align}\label{eq:genCase-params}
\text{$\deltafail = \nphase/K_{\rho,n}$
and
$\rhoprog=\Delta_0^2/6H^2$, where $\Delta_0 = \rho\cdot\ralt/2$}.
\end{align}

\begin{lemma}[Progress]\label{lem:main_prob_lemma}
\refeq{eq:genCase-progress} holds in each phase $\ell > \nlearn$ such that some triple in $\underexplored$ is $\rho$-reachable.
\end{lemma}

The proof of \Cref{thm:main_prob_mdp} follows directly from \Cref{lem:main_prob_lemma} via a martingale-Chernoff argument which is common in the study of tabular MDPs (this is spelled out \Cref{proof:main_prob_mdp}).

We ensure that a policy from a given policy set $\Pi$ is chosen in the hallucination episode, under suitable conditions. We use a corollary of \Cref{prop:mdp_hh} which holds under a weaker condition compared to  \eqref{eq:HH_good_condition}. This new condition is somewhat subtle to define. We treat
    $\gapcan\sbr{\Pi \mid \ledger}$,
the canonical gap given ledger $\lambda$, as a random variable with randomness induced by the ledger. In a given phase $\ell$, take a conditional expectation of this random variable given the censored ledger $\ledcensl$:
\begin{align}\label{eq:mean-canonical-gap-defn}
\gapcanl[\Pi\mid \ledger]
    &= \Exp\sbr{ \gapcan\sbr{\Pi \mid \ledger} \;\mid\; \ledcensl }
    &\EqComment{mean-canonical gap}.
\end{align}
Essentially, we average out all randomness in ledger $\ledger$ that comes from the current phase; the only remaining randomness comes from $\ledcensl$. The new condition is also given ledger $\ledger = \ledhall$, but it only requires the mean-canonical gap to be bounded, not the realization of the canonical gap. The new condition is stated as follows, for some deterministic parameter $\Delta$:
\begin{align}\label{eq:HH-stronger-condition}
\gapcanl\sbr{\Pi \mid \ledhall}
    \geq \Delta
    \geq \frac{6H}{\nphase \cdot \Prcan\sbr{\EvPun \mid \ledcensl}}.
\end{align}
We prove that this condition suffices to make progress, in a probabilistic sense.

\begin{lemma}[Hidden Hallucination via mean-canonical gap]\label{cor:mdp_hh}
Fix phase $\ell$ and policy set $\Pi \subset \Pimarkov$ such that \eqref{eq:HH-stronger-condition} holds for some deterministic parameter $\Delta$. Then, with probability at least $\Delta/2H$ over the draw of $\ledhall$ (conditioned on $\calF_{\ell-1}$), an agent in the hallucination episode chooses a policy in $\Pi$.
\end{lemma}

We obtain \eqref{eq:HH-stronger-condition} with a policy set $\Pi = \Pi_{\ell,q}$ (for some $q$) and parameter $\Delta = \Delta_0$. As before, we lower-bound
    $\Prcan\sbr{\EvPun \mid \ledcensl}$
with  $\qpunish(\epspunish)$, by definition of the latter (\Cref{eq:qpun-defn}). The key is to lower-bound the mean-canonical gap, which we accomplish next. The statement is also probabilistic: the mean-canonical gap being a random variable with randomness coming from the censored ledger $\ledcensl$, we obtain a high-probability statement over the randomness in $\ledcensl$.

\begin{lemma}[Probabilistic Gap Bound]\label{lemma:gap_bound_prob}
Fix phase $\ell > \nlearn$. Recall parameters $\deltafail$ and $\Delta_0$ from \eqref{eq:genCase-params}. With probability at least $1 - \deltafail$ over the randomness in ledger $\ledcensl$, it holds that
\[ \gapcanl\sbr{\Pi_{\ell,q} \mid \ledhall}
    \geq \Delta_0,
    \quad\text{where $q = \Delta_0/3H$}.\]
\end{lemma}

These two lemmas about the canonical gap imply \Cref{lem:main_prob_lemma}, if one plugs in all the parameters. (We omit the tedious but straightforward details.)

\subsubsection*{Proof Sketch of \Cref{lemma:gap_bound_prob}}\label{ssec:sketch:lemma:gap_bound_prob}

To analyze the mean-canonical gap, we define a similar version of the canonical posterior \eqref{eq:canon-posterior-defn}. The \emph{mean-canonical posterior} $\Prcanl[\,\cdot\mid\ledger]$ given a random ledger $\lambda$ is a distribution over MDP models given by
\begin{align}\label{eq:hal-posterior-defn}
\Prcanl[\modclass \mid\ledger]
    &:= \Exp\sbr{ \Prcan\sbr{ \modclass \mid \ledger }
        \;\mid\; \ledcensl }
    \quad\forall \text{ measurable }\modclass\subset\modtotal.
\end{align}

Paralleling the deterministic analysis (\ie the proof of \Cref{claim:gap_size}), we consider
    $\Prcanl[\,\cdot \mid\ledhall]$,
the mean-canonical posterior given the hallucinated ledger $\ledhall$. We verify that, with a ``good enough" probability, a model $\model$ drawn at random from this posterior satisfies the following:
\begin{itemize}
\item[(a)] for all fully-explored $\xah$ triples, $\model$ has small mean rewards, $\sfr_{\model}\xah$.
\item[(b)] for all fully-explored $\xah$ triples, the transition probabilities, $\sfp_{\model}(\cdot \mid x,a,h)$, are close to those for the true model, $\sfp_{\modst}(\cdot \mid x,a,h)$ (and similarly for the initial state distributions, provided $\ell > \nlearn$).
\end{itemize}
We show that these properties imply a small mean-canonical gap, relying on a probabilistic version of the  simulation lemma (\Cref{lem:modclass_not}). This version is stated and proved in \Cref{lem:visitation_comparison_general}.

To establish properties (a,b), we must address several sources of randomness: (i) randomness in realized rewards and transitions, (ii) randomness in the draw of the hallucinated rewards and (iii) randomness in the agent's posterior given the hallucinated ledger. We address them simultaneously by constructing a set of ``good models'' under which points (a) and (b) hold, and using Bayesian concentration inequalities to show that with high probability under \emph{all} randomness, any model under the agents posterior given a hallucinated ledger lies in this ``good set''.

We account for the fact that the mean-canonical posterior
    $\Prcanl[\,\cdot \mid\ledhall]$
is not the posterior formed by an agent given ledger $\ledhall$ (because agents know that rewards may be hallucinated).
However, in view of \Cref{eq:Prhall_ledhall},
it is a conditional posterior given the punish-event $\EvPun$. We argue that the strength of concentration and the rarity of hallucinations (which occur only once per phase) overwhelm the effect of this conditioning. See \Cref{sec:lemma:gap_bound_prob} for the full argument.




\section{Regret implications}\label{sec:exploit}

In this section, we investigate how our algorithm can be used to balance exploration and exploitation. We run \mdphh for $K_{\rho,n}$ episodes, as specified in \Cref{thm:main_prob_mdp}; we collectively refer to these episodes as \emph{exploration epoch}. Then, we use the collected data for exploitation. Specifically, we reveal all MDP trajectories from the exploration epoch in all subsequent episodes (which are called \emph{exploitation epoch}). Call this algorithm \HHandExploit.

We investigate \emph{regret}, a standard objective in RL which compares the algorithm's reward to the best policy given the true model $\modst$. Formally, we define
\begin{align*}
    \OPT(\modst) = \max_{\pi\in\Pimarkov}\; \sbr{\valuef{\pi}{\modst}},
\end{align*}
and Bayesian regret in $K$ episodes as
\begin{align*}
\REG(K) :=
    {\textstyle \sum_{k\in [K]}}\;
        \Expop_{\model\sim\prior}\sbr{
            \OPT(\model) - \valuef{\pi_k}{\model}}.
\end{align*}
The first-order issue in regret analyses in the literature is how it scales with $K$.

We prove that for \HHandExploit, $\REG(K)$ scales as $K^{2/3}$ when the reachability parameter $\rho$ is sufficiently small and the number $K$ of episodes is sufficiently large. This scaling is in line with explore-then-exploit approach in multi-armed bandits, which achieves $T^{2/3}$ regret in $T$ rounds.%
\footnote{Explore-then-exploit is indeed an appropriate comparison for \HHandExploit, because both algorithms proceed in two epochs of predetermined duration, where in the ``exploration epoch" one explores without regard to exploitation, and in the ``exploitation epoch" one exploits without regard to exploration. This separation of ``pure exploration" and ``pure exploitation" is known to be inefficient; in particular, optimal algorithms achieve regret rates that scale as $\sqrt{T}$ (in bandits) and $\sqrt{K}$ (in RL).}

This result holds when the reachability parameter $\rho$ is low enough to guarantee that all $\xah$ triples are $\rho$-reachable, except those that cannot ever be reached. More precisely, a triple $\xah$ is called \emph{never-reachable} if no policy can reach it under any feasible model, \ie
    \begin{align*} \sfP^\pi_{\model}\sbr{ (\bmx_h,\bma_h,h) =\xah} = 0
        \quad \forall\;\text{policy  $\pi\in\Pimarkov$, model $\model\in \mathtt{support}(\prior)$}.
    \end{align*}
Thus, the threshold value of $\rho$ is defined as follows:
    \begin{definition}
$\rho_{\min} \ge 0$ is  the smallest $\rho\geq 0$ such that each $\xah$ triple is either $\rho$-reachable for any model in $\mathtt{support}(\prior)$, or is never-reachable.
\end{definition}

Further, we require this threshold value to be strictly positive. Thus, our result is stated as follows:

\begin{theorem}\label{cor:Bayesian-regret-nice}
Suppose $\rho_{\min}>0$. Recall from \eqref{eq:thm:main_prob_mdp-K} that
    $\Phi_\rho := K_{\rho,n}/ n$
is determined by $\rho$, the prior, and $S,A,H$. Consider algorithm \HHandExploit which runs for $K$ episodes. Choose reachability parameter $\rho\in (0,\rho_{\min}]$, failure parameter $\delta = \frac{1}{KH}$, and target number of samples $\nlearn = (K/\Phi_\rho)^{2/3}$. Assume that $K^{2/3}$ is large enough to upper-bound the right-hand side in \eqref{eq:thm:main_prob_mdp-n}.
Then
\[ \REG(K) \leq \tilde{O}(K^{2/3})\cdot \Phi_\rho^{1/3}\cdot H \sqrt{H(S + \log(KHS))}.\]
\end{theorem}

In the rest of this section we derive \Cref{cor:Bayesian-regret-nice} as a corollary of \Cref{thm:main_prob_mdp}, and also obtain a (weaker) bound on Bayesian regret that does not rely on the condition that $0<\rho\leq \rho_{\min}$.

We need a generic proposition on exploitation in MDPs. While similar propositions have appeared in prior work \cite[\eg][]{jin2020reward}, we use a slightly non-standard version which is Bayesian and involved $\rho_{\min}$. We provide a self-contained proof in Appendix~\ref{app:exploit}.

\newcommand{\vepslearn}{\varepsilon_{\mathrm{lrn}}}

\begin{restatable}[Exploitation]{proposition}{proprevelation}\label{prop:revelation}
Let $\alg$ be an algorithm which satisfies $(\rho,n,\delta,K_0)$-\traversal\ for some reachability parameter $\rho>0$, target of $n$ samples, failure probability $\delta$, and $K_0\geq n$ episodes. Let $\hat\pi$ be any ``exploitation policy" after $K_0$ episodes, \ie
\begin{align*}
\hat \pi \in \argmax_{\pi\in \Pimarkov}
    \Exp\sbr{ \valuef{\pi}{\modst} \mid \cbr{\text{full history of the first $K_0$ episodes}}}.
    \end{align*}
Then, with probability $1 - 2\delta$ under all sources of randomness, the per-episode regret of policy $\hat \pi$ can be upper-bounded as follows:
\begin{align}\label{eq:cor:simple-regret_hi_prob}
\OPT(\modst) - \valuef{\hat{\pi}}{\modst}     \leq {\textstyle \mathcal{O}(H^2)\cdot\left(S \rho \cdot \ind_{\{\rho > \rho_{\min}\}}+ H\sqrt{\frac{S + \log(SAHn/\delta)}{n}}\right)}.
\end{align}
\end{restatable}

\begin{remark}
Put differently, one can upper-bound the left-hand side of \eqref{eq:cor:simple-regret_hi_prob} by a given $\epsilon>0$ if
\begin{align}
\rho \le c\max\cbr{\epsilon/(SH^2),\,\rho_{\min}}
    \quad \text{and}\quad
n \ge c\cdot H^6\epsilon^{-2} \rbr{S+\log \tfrac{SAH}{\epsilon\delta}},
\end{align}
for a large enough absolute constant $c$.
\end{remark}


The generic implication on Bayesian regret (which implies \Cref{cor:Bayesian-regret-nice}) is as follows:

\begin{corollary}\label{cor:regret}
Consider algorithm \HHandExploit with a fixed reachability parameter $\rho>0$. Use the assumptions and parametrization in \Cref{thm:main_prob_mdp}. Then each episode $k$ of the exploitation epoch satisfies \eqref{eq:cor:simple-regret_hi_prob} with $\hat{\pi} = \pi_k$ and $n=\nlearn$. Thus, Bayesian regret over $K>K_{\rho,n}$ episodes satisfies
 \begin{align}\label{eq:cor:regret}
\REG(K) \leq H K_{\rho,n} + (K-K_{\rho,n}) \rbr{\Psi_{\rho,n}+ 2\delta H},
\end{align}
where $\Psi_{\rho,n}$ is the right-hand side of \eqref{eq:cor:simple-regret_hi_prob}.
\end{corollary}


\section{Conclusions and Open Questions}\label{sec:conclusions}

This paper combines reinforcement learning (RL) and incentivized exploration, and advances both. From the RL perspective, we design RL mechanisms, \ie RL algorithms which interact with self-interested agents and are compatible with their incentives. This is the first paper on ``RL mechanisms", \ie \textbf{the first paper on any scenario that combines RL and incentives}, to the best of our knowledge. From the incentivized exploration perspective, we extend the learning model in several important ways discussed in the Introduction. We adopt a relatively simple model that captures salient features present in multiple motivating scenarios: namely, agents that have repeated, MDP-like interactions with the environment and need to be incentivized to explore. However, as is quite common in both RL theory and economic theory, we do not attempt to capture all the particularities of any given scenario, and instead focus on understanding the basic model. Aside from the explicit theoretical contributions, we hope that our ``hidden hallucination" technique would be useful as a general principle, both theoretically and heuristically.


One concrete follow-up question within our model of incentivized RL concerns exploring the MDP in $K=\poly(SAH)$ episodes. However, this question is not resolved even in the ``easier" setting of incentivized exploration in bandits with correlated priors.%
\footnote{However, \cite{Selke-PoIE-ec21} achieves this for incentivized exploration in bandits with \emph{independent} priors.}

More broadly, we view our work as a ``beachhead" for further investigation. As defined in the Introduction, a model of incentivized RL consists of three components: RL problem,  strategic interactions, and performance objective. Each component can be extended in several different ways (and a given motivating scenario may require several such extensions). For the RL component, one could consider incorporating large action/state sets (with structure such as linearity), partial observations (\ie POMDPs), contextual MDPs (where a \emph{context} is observed before each episode), or non-Markovian dynamics. For the ``strategic" component, one could allow the agents to revise their choices before each stage (rather than once per episode), and support heterogenous agents with public and/or private idiosyncratic signals.%
\footnote{\citet{Jieming-multitypes18} accommodate idiosyncratic signals for incentivized exploration in bandits.} For the ``performance objective", one could study scenarios when some reachable $\xah$ triples cannot be explored in our framework, and redefine \traversal objective to explore all $\xah$ triples that \emph{can} be explored.%
\footnote{Such extensions have been studied in \cite{ICexplorationGames-ec16,Jieming-multitypes18} for ``stateless" incentivized exploration.} Moreover, a mechanism for incentivized RL could, in principle, take advantage of agents' inherent incentives to explore their own MDP, when and if they exist. Finally, regret might improve if exploration is made more adaptive: essentially, it may suffice to explore some of the $\xah$ triples more rarely (or not at all).%
\footnote{Recall that optimal regret in $K$ episodes scales as $\sqrt{K}$, whereas we only achieve $K^{2/3}$ scaling (which is optimal if one separates ``pure exploration" and ``pure exploitation"). It is unclear if $\sqrt{K}$ scaling is achievable in incentivized RL.}

\medskip

\bluepar{Subsequent work.}
Two closely related papers have appeared subsequent to the initial publication of our work on \texttt{arxiv.org}. We discuss these papers below.

\citet{IncentivizedRL-ec22} study another model that combines RL and Bayesian Persuasion.%
\footnote{Our initial publication on \texttt{arxiv.org} predates theirs by a full year.}
Their model is similar to ours in that an algorithm for episodic RL needs to incentivize self-interested agents to follow its recommendations. The key difference is that their model focuses on payoff-relevant `outcomes' that are drawn IID before each episode and observed by the algorithm but not to the agents. These `outcomes' are the only source of information asymmetry used by the algorithm to create incentives; the guarantees in \citep{IncentivizedRL-ec22} appear vacuous when the `outcomes' are not observed. In contrast, the only source of information asymmetry in our model (and all prior work on incentivized exploration) is the history of interactions with other agents. An intriguing open question is to combine these two sources so that they reinforce one another.

\citet{CombiIE-neurips22} study a model of incentivized exploration with large, structured action set and correlated priors.%
\footnote{Our initial publication on \texttt{arxiv.org} predates theirs by more than a year, and acknowledged by theirs as prior work.}
Their learning problem is combinatorial semi-bandits, a version of bandits where the learner chooses among feasible \emph{subsets} of arms in each round, and rewards of all chosen arms are observable. This problem is stateless, and therefore ``easier" than ours. They  achieve stronger results on regret-minimization: essentially, Thompson Sampling is BIC for their model when initialized with $N$ samples of each atom, for some $N$ determined by the prior. In particular, their algorithm attains regret which scales as $\tilde{O}(\sqrt{T})$ in the number of rounds $T$. However, their solution for collecting one sample of each atom takes time exponential in the number of atoms, similar to how the number of episodes for our algorithm is exponential in $SAH$.

\bibliographystyle{plainnat}
\bibliography{bib-abbrv,bib-AGT,bib-bandits,bib-slivkins,bib-competition,bib-RL}

\appendix

\newpage
\addtocontents{toc}{\protect\setcounter{tocdepth}{2}}
\tableofcontents
\newpage

\section{Comparing Canonical and Mechanism Conditionals \label{app:canonical_vs_mechanism}}
\begin{example}[Canonical v.s. Mechanism Conditionals: Fabricated Rewards] Consider an extremely simplified setting with $X = H = A = 1$, and where the reward is deterministic. Thus, each model is specified by a single numerical reward value $r_{\model} \in [0,1]$. There is just one policy, say $\pi_0$, and $\orac(\pi_0;\model) = r_{\model}$; every trajector is specified by an $r \in [0,1]$, and any ledger with one trajectory is of the form $\ledger = (r,\pi_0)$.

Fix $r_1 \in \supp(\prior)$, and consider a mechanism whose signal to the agent at $k = 1$ is a ledger $\hat{\bm{\ledger}}_1$. Suppose that the mechanism  deterministically chooses this to be equal to $\ledger_1 := (r_1,\pi_0)$. This  is done without any information from the true instance $\modst$, and hence conditioning on the fact that the mechanism reveals the $\ledger_1$ is uninformative: $r' \in [0,1]$ $\Pr[r_{\modst} = r' \mid \ledger_1] = \prior(r')$ just coincides with the prior mass on $r'$. However, the canonical posterior, which pretends $\ledger_1$ was actually derived from observed data, is a dirac mass on $r_1$: $\Prcan[ r_{\modvarhat} = r' \mid \ledger_1] = \I( r' = r_1)$.
\end{example}
More subtly, the canonical probabilities also prevent the agent from gleaning information from which policies were used in the ledger, or which censoring set $\calU$ was used:
\begin{example}[Canonical v.s. Mechanism Conditionals: Policy Selection] Consider a setting with $X = H = 1$ and $A = 2$ and deterministic rewards; i.e., two-armed bandits, where $r_{\models}(a) \in \{0,1\}$ for both actions $a \in \{1,2\}$ and all models, and the prior is uniform on resulting four possible combinations. Again, policies correspond to selecting actions.

Consider a mechanism which always takes action $a_1 = 1$ on the first episode, and then selects action $a_2 = 1$ if the first observed reward is $r_{\modst}(a_1) = 1$, and action $a_2 = 2$ otherwise. The mechanism then reveals the ledger $\ledger_2 = (a_2, r_{\modst}(a_2))$ at episode $2$. Observe that if the revealed ledger $\ledger_2$ is of the form  $(a_2,r_2) = (2,x)$, $x \in \{0,1\}$, then the true model must have $r_{\modst}(2) = x$, but also $r_{\modst}(1) = 0$, since otherwise action $a_2 = 1$ would have been selected. Hence, whenever $r_{\modst}(1) = 0$, $\ledger_2$ uniquely determines $\modst$, and  $\Pr[ \cdot \mid \ledger_2]$ is a delta-mass on the true model. However, under canonical posterior $\modvarhat \sim \Prcan[\cdot\mid \ledger_2]$, $r_{\modvarhat}(1)$ is always uniform on $\{0,1\}$, since the principal's actions (i.e. policies) are treated as deterministic in the conditional, and thus no information about action $a = 1$ is revelead. A similar discrepency between mechanism and canonical posteriors can be obtained by considering a mechanism which always selects action $a_1 = 1$, and $a_2 = 2$, but censors the reward from $a_2 = 2$ if $r_{\modst}(a_1) = 1$.
\end{example} 

\section{Proofs of Hallucination Properties \label{sec:proofs_for_hhh}}

\subsection{Proof of \Cref{lem:data_hygiene} \label{proof:lem_data_hygiene}}
	Here, we prove the data-hygiene principle. We begin with the following general fact, which clarifies the essential properties of the hidden hallucination algorithm required for $\ledhonl$ and $\ledcensl$ to be hygenic. Throughout, we abbreviate sequences of variable $(x_1,\dots,x_{\ell})$ as $x_{1:\ell}$.
	\begin{fact}\label{fact:conditional_fact} Consider a collection of random `model' variable $\modvar$, `trajectory' random variables $\hat{\bm{x}}_1,\dots,\hat{\bm{x}}_{\ell} \in \mathcal{X}$, and `policy' random variables $\hat{\bm{y}}_1,\dots,\hat{\bm{y}}_{\ell} \in \mathcal{Y}$. Assume that $\mathcal{X}$ and $\mathcal{Y}$ are countable, and that the above random variables have joint law $\Pr$. Suppose that
	\begin{itemize}
	\item[(a)] There is function $\Pr_0[\cdot \mid \model,y]$ whose values are probability distributions of $\calX$ such that 
	\begin{align*}
	\Pr[\hat{\bm{x}}_{i} \in \cdot \mid \hat{\bm{x}}_{1:i-1},\hat{\bm{y}}_{1:i}, \modvar]  \quad=\quad \Pr_0[\hat{\bm{x}}_{i}  \in \cdot \mid \modvar, \hat{\bm{y}}_i]
	\end{align*}
	\item[(b)] the ``policy'' variable  $\hat{\bm{y}}_{i}$ and ``model'' variable $\modvar$ are conditionally independent  given the past ``trajectory'' variables $\hat{\bm{x}}_{1:i-1}$.
\end{itemize}
	Now, for a given sequence $y_{1:\ell} \in \mathcal{Y}^\ell$, define a distribution  $\tilde{\Pr}$ over random variables, $(\tilde{\modvar}, \tilde{\bm{x}}_{1:\ell})$ by lettting $\tilde{\modvar}$ have the same marginal distribution as $\modvar$, and letting $\tilde{\bm{x}}_{i} \sim \Pr_0[\cdot \mid \tilde{\modvar},y_i]$ are drawn independent from the law $\Pr_0$, under the fixed $y_i$ and $\tilde{\modvar}$. Then, we have the following equality of distributions:
	\begin{align}
	\Pr[\modvar \in \cdot \mid \hat{\bm{x}}_{1:\ell} = x_{1:\ell}, \hat{\bm{y}}_{1:\ell}= y_{1:\ell}] = \tilde{\Pr}[\tilde{\modvar} \in \cdot \mid \tilde{\bm{x}}_{1:\ell} = x_{1:\ell}] \label{eq:useful_fact}
	\end{align}
	\end{fact}
 \Cref{fact:conditional_fact} can be verified directly by writing out the relevant conditionals.

Let us now apply above the fact to the proof of \Cref{lem:data_hygiene}, beginning with the censored ledger $\ledcens$. We take $\modvar$  to be $\modst$, and in phases $i$, let $\hat{\bm{x}}_i := \cens(\traj_{\kexpl[i]})$ denote the random variable corresponding to the total censoring of the trajectory recieved on the (random) $i$-th hallucination episode, and $\hat{\bm{y}}_i := \pi_{\kexpl[i]}$ the random variable for the corresponding policy; we take $\Pr$ be the measure yielding their joint distribution. Then,
\begin{itemize}
	\item Condition (a) in  \Cref{fact:conditional_fact} satisfied since $\traj_{\kexpl[i]} \sim \sfP^{\pi}_{\modelst}$ for $\pi = \pi_{\kexpl[i]}$, even when conditioning on all past policies and trajectories.
	\item Condition (b) also holds. This is for because the choice of policy on a hallucination episode $i$ is generated purely based on the hallucination mechanism, which (1) only uses information from past hallucination episodes $1:i-1$ and (2) the only information from those past episodes are transition information (not rewards), and these are specificed by the totally censored trajectories $\hat{\bm{x}}_j := \cens(\traj_{\kexpl[j]})$, $j < i$.
\end{itemize}
  Thus, \Cref{eq:useful_fact} holds. The left-hand side of \Cref{eq:useful_fact} is precisely $\Pr[\modst \in \cdot \mid \ledcensl]$, since $\ledcensl$ comprises of precisely the trajectories and policies on past hallucination episodes. And, the right hand side of \Cref{eq:useful_fact} is precisely the canonical probability $\Prcan[\modst \in \cdot \mid \ledcensl]$ which regards the policies as fixed and non-random; hence, $\ledcensl$ is hygenic.

  The proof that $\ledhonl$ is hygenic is nearly identical. Here, we take $\hat{\bm{x}}_i := \cens(\traj_{\kexpl[i]};\calU_i)$ to be the $\calU_i$-censored trajectories, and $\hat{\bm{y}}_i = (\calU_i,\pi_{\kexpl[i]})$ to be the pair consisting of the hallucination policy at phase $i$, and the censoring set $\calU_i$ for that phase. Again, condition (a) follows directly, and condition (b) follows since both the policy on the hallucination episodes $\pi_{\kexpl[i]}$ and the censoring set $\calU_i$ are determined by algorithmic randomness and past totally censored ledgers from hallucination episodes, $\cens(\traj_{\kexpl[j]})$, $j < i$. Since $\hat{\bm{x}}_i := \cens(\traj_{\kexpl[i]};\calU_i)$ only censors triples in $\calU_i$, $\hat{\bm{x}}_{1:i-1}$ uniquely determines $\cens(\traj_{\kexpl[j]}), ~j < i$. Thus, given $\hat{\bm{x}}_{1:i-1}$, , $\hat{\bm{y}}_i = (\calU_i,\pi_{\kexpl[i]})$ is independent of the model $\modst$. The desired conclusion follows.
	\qed

\subsection{Proof of \Cref{prop:mdp_hh} \label{proof:prop_mdp_hh}}

\newcommand{\ledvarhat}{\hat{\ledvar}}
\newcommand{\Zhal}{Z_{\mathrm{hal}}}
\newcommand{\phal}{p_{\mathrm{hal}}}
\newcommand{\Bernoulli}{\mathrm{Bernoulli}}
\newcommand{\khatexpl}{\widehat{\bm{k}}_{\ell}^{\mathrm{\ell}}}

Thourought, we fix an episode $k \in \Phase_{\ell}$ for a given phase $\ell \in \N$. Recall that $\Pi \subset \Pitotal$ is our target set of policies. We have to show that, given that if the revealed ledger $\ledger_k$ satisfies
\begin{align}
\frac{1}{\nphase} \le \frac{\gapcan[\Pi \mid \ledger_k] \cdot \Prcan[\EvPun\mid \ledcensl]}{3H},
\label{eq:hh_condition}
\end{align}
it holds that any maximizer of $\Exp[\valuef{\pi}{\modst} \mid \ledvarhat_k = \ledger_k]$ lies in $\Pi$.

The central objects in our analysis is the following $\Prmech$-measurable random variable $\Zhal$ and its conditional expectation:
\begin{align}
\Zhal = \ind_{\{k = \kexpl\}}, \quad \phal := \Pr[\Zhal = 1\mid \ledger_k]
\end{align}
In words, $\phal$ captures the  agent's suspicion that the ledger $\ledger_k$ they are shown at episode $k$ was hallucinated, and not the honest ledger with the true rewards. First we show that if this probability is sufficiently small, then any Bayes-Greedy policy lies in $\Pi$:
\begin{claim}\label{claim:phal_gap} Fix a ledger $\ledger$ satisfying $\phal < \gapcan[\Pi \mid \ledger_k ]/2H$. Then any maximizer of $\Exp[\valuef{\pi}{\modst} \mid  \ledger_k = \ledger]$ lies in $\Pi$.
\end{claim}
\begin{proof}[Proof of \Cref{claim:phal_gap}] Fix two policies $\pi_1 \in \Pi,\pi_2 \in \Pitotal \setminus \Pi$. Fix an abitrary ledger $\ledger$ in the support of $\ledger_k$; it will be clearner to reason about conditionals $\{\ledger_k = \ledger\}$. It suffices to show that for arbitrary $\ledger$, if $\phal(\ledger) := \Pr[\Zhal = 1\mid \ledger_k = \ledger] <  \gapcan[\Pi \mid \ledger]/2H$, then
\begin{align*}
\Exp[\valuef{\pi_1}{\modst} - \valuef{\pi_2}{\modst} \mid \ledger_k = \ledger] > 0.
\end{align*}
To this end, we lower bound the above difference
\begin{align*}
&\Exp[\valuef{\pi_1}{\modst} - \valuef{\pi_2}{\modst} \mid \ledger_k = \ledger]\\
&\quad= \Exp[\ind_{\{\Zhal = 1\}} \left(\valuef{\pi_1}{\modst} - \valuef{\pi_2}{\modst}\right) \mid \ledger_k = \ledger]\\
&\qquad+ \Exp[\ind_{\{\Zhal = 0\}} \left(\valuef{\pi_1}{\modst} - \valuef{\pi_2}{\modst}\right) \mid \ledger_k = \ledger]\\
&\overset{(i)}{\ge}  \Exp[\ind_{\{\Zhal = 0\}} \left(\valuef{\pi_1}{\modst} - \valuef{\pi_2}{\modst}\right) \mid  \ledger_k = \ledger] - H\phal(\ledger)\\
&\overset{(ii)}{=}   (1-\phal)\Exp[ \valuef{\pi_1}{\modst} - \valuef{\pi_2}{\modst} \mid \ledhonl = \ledger] - H\phal(\ledger)\\
&\overset{(iii)}{=}   (1-\phal)\gapcan[ \Pi \mid \ledger] - H\phal(\ledger),
\end{align*}
where $(i)$ uses that values are upper bounded by $H$, and $(ii)$ uses that $\Zhal$ is selected using independent randomness, and when $\Zhal = 0$, the revealed ledger is the honest ledger: $\ledger_k= \ledhonl$. Equality $(iii)$ is precisely the definition of the canonical gap, \Cref{defn:canonical_gap}
Now, since the honest ledger satisfies the data hygiene guarantee (\Cref{lem:data_hygiene}), we obtain
\begin{align*}
\Exp[\valuef{\pi_1}{\modst} - \valuef{\pi_2}{\modst}\mid \ledhonl = \ledger] &= \Expcan[\valuef{\pi_1}{\modst} - \valuef{\pi_2}{\modst} \mid \ledger],
\end{align*}
 Hence, combining the two displays, and using $\gapcan \le H$ (since all values are bounded by $H$),
\begin{align*}
\Exp[\valuef{\pi_1}{\modst} - \valuef{\pi_2}{\modst} \mid  \ledger_k = \ledger] &\ge (1 - \phal)\gapcan[ \Pi \mid \ledger] - \phal(\ledger) H \\
&\ge \gapcan[ \Pi \mid \ledger] - 2\phal(\ledger) H. \qquad\qedhere
\end{align*}
\end{proof}
To apply \Cref{claim:phal_gap}, we need to further analyze the term $\phal$. To do so, we need the following intermediate claim, which says that the conditional distribution of the honest ledger on the event that model is in the punishing set coincides with the distribution of the conditional ledger:
\newcommand{\ledcs}{\ledger_{\mathrm{cens}}}

\begin{claim}\label{claim:ledhonl_cannonical} The following equality of distributions holds:
\begin{align*}
\Pr[\ledhonl = \cdot \mid \EvPun, \ledcensl ] = \Pr[\ledhall = \cdot \mid \ledcensl ]
\end{align*}
\end{claim}
\begin{proof}
Let $\calD_{1}$ denote the joint distribution of $(\ledhonl,\modst)$ conditioned on $ \EvPun $ and $\ledcensl$, and let $\calD_{2}$ denote the joint distribution of $(\ledhall,\modhall)$ conditioned on $\ledcensl$. By marginalizing, it suffices to show that $\calD_{1} = \calD_{2}$.

First, observe that since $\modhall \sim \Pr[ \modst \in \cdot \mid \ledhall,  \EvPun]$, $\modst$ and $\modhall$ have the same marginal distribution under $\calD_1$ and $\calD_2$. Now, consider the distribution of $\ledhonl \mid \modst$ and $\ledhall \mid \modhall $ under $\calD_1,\calD_2$, respectively. The first conditional is equal to the distribution of $\ledhonl \mid \modst, \ledcensl$, and the second $\ledhonl \mid \modhall, \ledcensl$. Since all transition data in $\ledhonl$ and $\ledhall$ are determined by $\ledcensl$, it suffices to show that the rewards in both ledgers have the same distribution.

Similar to the proof of the hygiene guarantee (\Cref{lem:data_hygiene}), we observe that the data in the censored ledger $\ledcensl$ are independent of the rewards in $\ledhonl$. Hence, the rewards $r\kh$ for each triple $(x\kh,a\kh,h)$ that appears in each constituent trajectory $\traj_k$ in $\ledcensl$ are independent draws from the corresponding reward distribution under $\modst$ - namely, $r\kh\sim \mathcal{R}_{\modst}(x\kh,a\kh,h)$. Moreover the rewards in the hallucinated ledger are constructed in the same way, but with $\mathcal{R}_{\modst}$ replaced with $\mathcal{R}_{\modhall}$. Hence, equality of distribution follows.
\end{proof}
\Cref{claim:ledhonl_cannonical} forms the corner cornerstone of our upper bound on $\phal$ (that is, upper bound on the agents belief that the revealed ledger arose from hallucination):
\begin{claim}\label{claim:phal_claim} Setting  $p_0 := 1/\nphase$, the following holds for any realization of $\ledger_k$:
\begin{align*}
\phal \le \frac{1 }{1 + \frac{(1-p_0)}{p_0}\Prcan[ \EvPun \mid \ledcensl] }.
\end{align*}
\end{claim}
\begin{proof}[Proof of \Cref{claim:phal_claim}] With $p_0 := 1/\nphase$, the marginal of $\Zhal$ satisfies $\Pr[\Zhal = 1] = p_0$. Fix a ledger $\ledger$ in the support of $\ledger_k$, and again set $\phal(\ledger) := \Pr[\Zhal = 1\mid \ledger_k = \ledger]$.
\begin{align}
\phal(\ledger) &= \Pr[\Zhal = 1 \mid  \ledger_k = \ledger ]\\
&= \frac{\Pr[\Zhal = 1 \text{ and }  \ledger_k = \ledger ]}{\Pr[\ledger_k = \ledger]} \nonumber\\
&= \frac{\Pr[\ledger_k  = \ledger \text{ and } \Zhal = 1 \mid  ]}{\Pr[\ledger_k = \ledger \text{ and } \Zhal = 1] + \Pr[\ledger_k = \ledger \text{ and } \Zhal = 0 ] } \nonumber\\
&\overset{(i)}{=} \frac{p_0\Pr[\ledhall = \ledger] }{p_0\Pr[\ledhall = \ledger] + (1-p_0)\Pr[\ledhonl = \ledger] }, \nonumber\\
&\overset{(ii)}{=} \frac{1 }{1 + \frac{(1-p_0)}{p_0}\cdot \frac{\Pr[\ledhonl = \ledger]}{\Pr[\ledhall = \ledger]} }, \label{eq:phal_first_bound}
\end{align}
whre $(i)$ uses that $\Pr[\ledger_k  = \ledger \text{ and } \Zhal = 1 \mid  ] = \Pr[\ledger_k  = \ledger \mid \Zhal = 1 ] \Pr[\Zhal = 1] = p_0 \Pr[\ledger_k  = \ledger \mid \Zhal = 1 ] = \Pr[\ledhall = \ledger]$, since $\Zhal$ is independent of all observed trajectories and determines whether the revealed trajectoried $\ledger_k$ is hallucinated or honest (and a similar computation for when $\Zhal = 0$). In $(ii)$, we have assumed $\Pr[\ledhall = \ledger]  > 0$, for otherwise the upper bound on $\phal$ is immediate from the previous line.

It remains to lower bound the ratio $\frac{\Pr[\ledhonl = \ledger]}{\Pr[\ledhall = \ledger]}$, again assuming the denominator is non-zero.
We bound bound
\begin{align*}
\frac{\Pr[\ledhonl = \ledger]}{\Pr[\ledhall = \ledger]} &= \frac{\Pr[\ledhonl = \ledger, \ledcensl = \ledcs]}{\Pr[\ledhall = \ledger, \ledcensl = \ledcs]}\\
&= \frac{\Pr[\ledhonl = \ledger \mid \ledcensl = \ledcs]}{\Pr[\ledhall = \ledger \mid \ledcensl = \ledcs]}\\
&\ge \frac{\Pr[\ledhonl = \ledger, \EvPun \mid \ledcensl = \ledcs]}{\Pr[\ledhall = \ledger \mid \ledcensl = \ledcs]}\\
&= \Pr[  \EvPun \mid \ledcensl = \ledcs] \cdot  \frac{\Pr[\ledhonl = \ledger \mid \EvPun, \ledcensl = \ledcs]}{\Pr[\ledhall = \ledger \mid \ledcensl = \ledcs]}
\end{align*}

 From \Cref{claim:ledhonl_cannonical}, it follows that $\Pr[  \EvPun \mid \ledcensl = \ledcs]  = \Prcan[ \EvPun \mid \ledcensl = \ledcs]$, and that
\begin{align*}
\frac{\Pr[\ledhonl = \ledger \mid \EvPun, \ledcensl = \ledcs]}{\Pr[\ledhall = \ledger \mid \ledcensl = \ledcs]} = 1.
\end{align*}
Thus, $\frac{\Pr[\ledhonl = \ledger]}{\Pr[\ledhall = \ledger]} \ge \Pr[ \EvPun \mid \ledcensl = \ledcs]$. The desired bound follows.
\end{proof}
\Cref{prop:mdp_hh} now follows from combining the above claims.
\begin{proof}[Proof of \Cref{prop:mdp_hh}]
Introduce the shorthand $q:= \Prcan[ \EvPun \mid \ledcensl]$, and recall $p_0 = 1/\nphase$. From \Cref{eq:hh_condition}, $p_0 \le p_{\star} := q \cdot \gapcan[\Pi \mid \ledger_k]/3H$.  From \Cref{claim:phal_claim}, we have the upper bound
\begin{align}
\phal \le \frac{1}{1 + q(1-p_0)/p_0} \label{eq:phal_fin_bound}.
\end{align}
Hence, \Cref{claim:phal_gap} ensures that all Bayes-greedy policies $\pi_k$ lie in $\Pi$ as soon as we can ensure that the RHS of \Cref{eq:phal_fin_bound} is strictly less than $\gapcan[\Pi \mid \ledger_k]/2H$. To this end, note that for $p_0 \le p_{\star}$ for $p_{\star}$ above, the bounds $\gapcan \le H$ and $q \le 1$ entail $p_0 \le 1/3$. Moreover, $p_0 \mapsto \frac{1}{1 + q(1-p_0)/p_0}$ is decreasing in $p_0$. Thus,
\begin{align*}
\frac{1}{1 + q(1-p_0)/p_0} \le \frac{1}{1 + q(1-p_\star)/p_\star} = \frac{1}{1 + \frac{2q}{3 q  \gapcan[\Pi \mid \ledger_k]/3H}} = \frac{1}{1 + \frac{2H}{\gapcan}} < \frac{\gapcan[\Pi \mid \ledger_k]}{2H}. \quad\qedhere
\end{align*}
\end{proof}

\subsection{Proof of \Cref{cor:mdp_hh}}
By assuming \Cref{eq:HH-stronger-condition},
\begin{align*}
\frac{1}{\nphase} \le \frac{(\Delta/2) \cdot \Prcan[ \EvPun \mid \ledcensl]}{3H}.
\end{align*}
Hence, by \Cref{prop:mdp_hh}, it holds that, whenever $\gapcan[\Pi \mid \ledhall] \ge \Delta/2$, $\pi_k \in \Pi$ for $k = \kexpl$. Thus, it suffices to show that  $\Prsimhh\left[\gapcan[\Pi \mid \ledhall] \ge \Delta/2 \mid \ledcensl\right] \ge \frac{\Delta}{2H}$. To this end, we observe that $\gapcan[\Pi \mid \ledhall] \in [-H,H]$ with probability $1$ over the draw of $\ledhall$. Hence,
\begin{align*}
\Delta &\le \Expsimhh\left[\gapcan[\Pi \mid \ledhall] \mid \ledcensl\right]\\
&\le \frac{\Delta}{2} +  H \Prsimhh\left[\gapcan[\Pi \mid \ledhall] \ge \Delta/2 \mid \ledcensl\right]. \qquad\qed
\end{align*}

\section{Proofs for the Probabilistic MDP Setting \label{sec:proofs_mdps_randomized}}

Throughout, we define
\begin{align}
&\calU_{\ell} := \{(x,a,h) \in [S] \times [A] \times [H] : N_{\ell}(x,a,h) < \nlearn\}, ~ \label{eq:calU_def}\\
&\text{ where } N_{\ell}(x,a,h) := \textstyle\sum_{k \in \calK_{\ell}}^{\ell-1} \I\{(x,a,h) \in \traj_{k}\} \nn.
\end{align}

\subsection{Proof of \Cref{lemma:gap_bound_prob} \label{sec:lemma:gap_bound_prob}}

Recall from \Cref{ssec:sketch:lemma:gap_bound_prob} the following notation:
\begin{itemize}
\item $\Prmucan[\cdot \mid \cdot]$ and $\Expmucan[\cdot \mid \cdot]$ to denote cannonical expectation distribution, which is viewed as measure over an abstract model random variable $\modvarhat$; e.g. $\Prmucan[\modvarhat \in \cdot \mid \ledger] = \Prcan[\modst \in \cdot  \mid \ledger]$.
\end{itemize}
Further, we define the \emph{random measures} $\Exphall$ and $\Prhall$:
\begin{align}
\Exphall\sbr{f(\modvarhat)}
    &:= \Expop\left[\Expmucan[f(\modvarhat) \mid \ledhall] \mid \ledcensl\right],\nonumber\\
\Prhall\sbr{\modvarhat \in \cdot}
    &:= \Expop\left[\Prmucan[\modvarhat \in \cdot \mid \ledhall] \mid \ledcensl\right]   \label{eq:hallucination_expectations},
\end{align}
where the outer expectation  the distribution of the hallucinated ledger $\ledhall$ conditioned on the totally censored ledger $\ledcensl$, and the inner expectation over the cannonical distribution given the hallucinated ledger. Note that $\Exphall$ and $\Prhall$ are functions of $\ledcensl$, though this is made
Reiterating the proof sketch, our aim is to ensure that
\begin{itemize}
\item[(a)] For highly visited triples $\xah \in \calU_{\ell}^c$, $\modvarhat$ has small rewards $\sfr_{\modvarhat}\xah$.
\item[(b)] For highly visited triples $\xah \in \calU_{\ell}^c$, the transition probabilities $\sfp_{\modvarhat}(\cdot \mid x,a,h)$ are close to those for the true model, $\sfp_{\modst}(\cdot \mid x,a,h)$ (and a similar closeness holds for the initial state distributions).
\end{itemize}

We assign the following definitions to the properties (a) - low reward,  and (b) - accurate transitions - described above.\footnote{definitions are stated for the \emph{complements} of subsets $\calU$ to remain consistent with how the definitions are used}.
\begin{definition}[Punished] Given $\calU \subset [S] \times [A] \times [H]$, we say that a model $\model$ is $\varepsilon$-punished on $\calU^c$, for all $\xah \in \calU^c$, $\sfr_{\model}\xah \le \varepsilon$.
\end{definition}
\begin{restatable}[Transition-Similar]{definition}{defnsimilar} \label{defn:similarity_transition} Let $\|\cdot\|_{\ell_1}$ denote the $\ell_1$-distance between probability distributions.  Given $\calU \subset [S] \times [A] \times [H]$, we say two models $(\model,\modst)$ are $\varepsilon$-transition-similar on $\calU^c$ if (i) $\|\sfp_{\model}(\cdot \mid 0 ) - \sfp_{\modst}(\cdot \mid 0)\|_{\ell_1} \le \varepsilon $ \emph{(closeness of initial state distribution)}, and (ii) for each $(x,a,h) \in \calU^c$, $\|\sfp_{\model}(\cdot \mid x,a,h) - \sfp_{\modst}(\cdot\mid x,a,h)\|_{\ell_1} \le \varepsilon$ \emph{(closeness of transitions on $\calU^c$)}.
\end{restatable}
 Note that transition-similarity concerns \emph{transition probabilities but \textbf{not} rewards}.
The various tolerances $\varepsilon$ for which we show these properties hold is in part determined by randomness in the transitions and rewards; these up up being quantified by terms arising from Azuma-Hoeffding's inequality:
\begin{definition}[Error Bounds]\label{defn:error_bounds} The reward and transition error bounds are
\begin{align}
\textstyle \epsr := \sqrt{ \frac{2\log(1/\delta_0)}{\nlearn}}, \quad \text{ and } \epsp:= 2\sqrt{ \frac{2(S\log(5) + \log(1/\delta_0))}{\nlearn}}, \quad \text{where } \delta_0 := \frac{\deltafail \cdot \qpunish \cdot \epspunish}{4SAH}.  \label{eq:delta_not}
\end{align}
\end{definition}
Here, $\epsr$ arises form the the two-sided concentration of the first $\nlearn$ samples from rewards at triple $\xah$ collected during the hallucination episodes; $\epsp$ reflects concentration of transitions in the $\ell_1$ norm, incurring an extra $S$ factor due to a covering argument (see \Cref{lem:conc_bounds} for details). Using these error bounds, we define the set of models:
\begin{definition}[Good Models] We define the \emph{good model set}, $\modgoodl \subset \modtotal$, as
\begin{align}
\modgoodl := \left\{\model: \begin{matrix} &\model \text{ is } (\epspunish + 2\epsr) \text{ punished  on } \calU_{\ell}^c \\
&(\model,\modst) \text{ are } 2\epsp \text{ transition-similar on } \calU_{\ell}^c\end{matrix}   \right\}. \label{eq:modgoodl_a}
\end{align}
We call models $\model \in \modgoodl$ ``good models''.
\end{definition}
Note that the above definition admits two interpretations: (a) a frequentist interpretation, where $\modgoodl$ is a fixed set depending on the unknown parameter $\modst$, and (b) a Bayesian interpretation, where $\modgoodl$ is a random set depending on the random model $\modst$.

The models in the $\modgoodl$ satisfy two key properties, for all sufficiently visited triples $\xah \in \calU_{\ell}^c$. First, the transition probabilities are $2\epsp$-close to those of the true model $\modst$, and second, the rewards from those triples are at most $\epspunish$, pluss a  $2\epsr$ error term. Crucially, the cannonical posterior given $\ledhall$, $\Prhall$, concentrates on $\modgoodl$.
\begin{lemma}\label{lem:hallucination_good_event}
For any phase $\ell > \nlearn$, the following event holds with probability $ 1- \deltafail$:
\begin{align*}
\Egoodl := \{\Prhall\left[ \modvarhat \in \modgoodl\right]\ge 1 - \epspunish\}
\end{align*}
The randomness in $\Egoodl $ is over the randomness in $\modst$ (determining $\modgoodl$) and the censored ledger $\ledcensl$ (determining $\Prhall[\cdot]$).
\end{lemma}
\begin{proof}[Proof Sketch] The proof builds on the Bayesian Chernoff bounds due to \citep{Selke-PoIE-ec21}. First, we define estimators $\theta_{\sfp}$ and $\theta_r$ of the empirical transitions and rewards of the true model $\modst$, which we show concentrate around the true transitions and rewards for sufficiently visited triples $\xah \in \calU^c$. We argue that this implies that the posteriors under $\Prhall$ must also concentrate around the truth. However, we modify \citep{Selke-PoIE-ec21} to adress that ledger $\ledhall$ with respect to which the posterior $\Prhall$ is defined is based on a samples from hallucinated model $\modhall$. This model is not an exact draw from the true posterior given $\ledcensl$ (which would be the analogue to \citep{Selke-PoIE-ec21}), but a  posterior which \emph restrict to $ \EvPun$ (that is, theevent that all rewards on $\xah \in \calU_{\ell}^c$ are at most $\epsilon$); refer back to \Cref{eq:hallucination_expectations} for the formal definition of $\Prhall$.

This difference from  \citep{Selke-PoIE-ec21} has two implications: first, we must account for the minimal probably that $\EvPun$ occurs given the censored ledger, which incurs a factor of $\qpunish$ (\Cref{eq:qpun-defn}) in our selection of $\delta_0$ . Second, this restriction allows us to argue that the posterior $\modvarhat \sim \Prhall$ is $\epspunish + 2\epsr$-punished, where the  $\epspunish$ is from restriction to the punishing class, and $2\epsr$ is from the Bayesian Chernoff. Combining with a union bound for all triples $\xah \in \calU^c_{\ell}$, and applying Bayesian Chernoff to the appropriate transitions and initial state distribution to verify $2\epsp$transition-similarity, we conclude the argument. The full proof is given in  \Cref{sec:proof_hall_good_event}.
\end{proof}
Next, let us use the punished and transition-similarity properties of $\modgoodl$ to control obtain the bound in \Cref{lemma:gap_bound_prob}. Our key technical tool is relating certain cumulative rewards between transition-similar models, whose proof is in the spirit of \cite{kearns2002near}:

\newcommand{\Evisitnol}{\scrE_{\calU}}
\newcommand{\omegast}{\omega_{\star}}
\begin{lemma}[Simulation Lemma]\label{lem:visitation_comparison_general}
Fix $\calU \subset [S] \times [A] \times [H]$ and $\varepsilon \ge 0$, and let $(\model,\modst)$ be two models which are $\varepsilon$-transition-similar on $\calU^c$.  For $h \in [H]$, introduce the events $\scrE_h := \{(\bmx_\tau,\bma_{\tau},\tau) \in \calU^c,~ \forall \tau < h\}$.\footnote{These events are measurable under the probability measures of the form $\sfP^{\pi}_{\model}[\cdot]$. } Then, for any reward function $\rtil: [S] \times [A] \times [H] \to [0,1]$, and policy $\pi \in \Pimarkov$,
\begin{align*}
\left|\sfE^{\pi}_{\model}\left[\sum_{h=1}^H \rtil(\bmx_h,\bma_h,h) \ind_{\{\scrE_h\}}\right] - \sfE^{\pi}_{\modst}\left[\sum_{h=1}^H\rtil(\bmx_h,\bma_h,h) \ind_{\{\scrE_h\}}\right]\right| \le \binom{H}{2}\varepsilon.
\end{align*}
In particular, defining $\Evisitnol := \{\exists h : (\bmx_h,\bma_{h},h) \in \calU\}$, then $|\sfP^{\pi}_{\model}\left[\Evisitnol \right] - \sfP^{\pi}_{\modst}\left[\Evisitnol \right]| \le \binom{H}{2}\varepsilon$
\end{lemma}
The above lemma is proven in \Cref{proof:visitation_comparison_general}, and its purpose is to relate visitations under good models $\model \in \modgoodl$ to visitations under the true model $\modst$.

We apply the bound in two ways. First, we upper bound on the value of policies under good models in terms of the probability of visiting under-visited triples $\xah \in \calU_{\ell}$:
\begin{claim}[Value Upper Bound for Good Models]\label{claim:value_upper_bound} For any good $\model \in \modgoodl$,
\begin{align*}
\valuef{\pi}{\model} &\le H \Pr^{\pi}_{\model}\left[ \Evisit \right] + H (2\epsr + \epspunish) \tag*{(punished rewards)}\\
&\le H \Pr^{\pi}_{\modst}\left[ \Evisit \right] + H (2\epsr + \epspunish) + H(H-1)\epsp \tag*{(similarity \& punished rewards)}
\end{align*}
In particular for $\pi \in \Pimarkov \setminus \Pi_{\ell}$, we have
\begin{align}
\valuef{\pi}{\model} \le H \rho_0 + H (2\epsr + \epspunish) + H(H-1)\epsp. \label{eq:value_bound_on_Pi_l_c}
\end{align}
\end{claim}
The proof is direct, and given in \Cref{sec:proof_of_prob_claims}, along with the proofs of the subsequent four claims. Notably, \Cref{eq:value_bound_on_Pi_l_c} upper bounds the value of policies  $\pi \in \Pimarkov \setminus \Pi_{\ell}$; that is, for good models, policies which do not explore $\calU_{\ell}$ do \emph{not} have high value. Next, we establish a lower bound on the policy values. To do so, we shall opt for the following representation of the exploration probability $\sfP_{\modst}^{\pi}[\Evisit],$:
\begin{claim}\label{claim:omega_star}Define $\omegast^\pi\xah :=  \sfP_{\modst}^{\pi}[(\bmx_h,\bma_{h}) = (x,a) \text{ and } (\bmx_\tau,\bma_{\tau},\tau) \in \calU_{\ell}^c, \quad \forall \tau < h]$ as the probability that the MDP visits $\xah$, but does not leave $\calU_\ell^c$ before step $h$. Then,
\begin{align*}
\sum_{\xah \in \calU_\ell}\omegast^\pi\xah = \sfP_{\modst}^{\pi}[\Evisit],
\end{align*}
Moreover, if $\rho \ge \rho_0$ and $\calU_{\ell} \cap \reach_{\rho}(\modst)$ is \emph{nonempty}, then there exists a policy $\pi \in \Pi_{\ell}$ for which 
$$\sum_{\xah \in \calU_\ell}\omegast^\pi\xah \ge \rho.$$
\end{claim}
The first part of the claim uses that the events in the definition of $\omegast^\pi\xah$ over $\xah \in \calU_{\ell}$ give a disjoint decomposition of the event $\Evisit$; the second par uses the first identity, together with the fact that if there $\calU_{\ell} \cap \reach_{\rho}(\modst)$ is non-empty, then there is policy $\pi$ which reaches a triple $\xah \in \calU_{\ell}$ with probabilty at least $\rho$; this policy therefore has $\sfP_{\modst}^{\pi}[\Evisit] \ge \rho \ge \rho_0$.  We combine \Cref{claim:omega_star} with the following value lower bound in terms of $\omegast$ and rewards on $\xah \in \calU_\ell$:
\begin{claim}[Value Lower Bound for Good Models] For any good model $\model \in \modgoodl$, \label{claim:value_lower_bound}
\begin{align*}
\valuef{\pi}{\model}  \ge \sum_{(x,a,h) \in \calU_{\ell}} \sfr_{\model}\xah \cdot \omegast^\pi\xah - H(H-1)\epsp.
\end{align*}
\end{claim}
In particular, from \Cref{claim:omega_star}, if $\calU_{\ell} \cap \reach_{\rho}(\modst)$, then there exists a policy $\pi \in \Pi_{\ell}$ for which $\valuef{\pi}{\model}   \ge \rho \min_{\xah \in \calU_{\ell}}  \sfr_{\model}\xah - H(H-1)\epsp$. However, using this inequality directly forces us to consider the minimal reward (on $\calU_{\ell}$) on each good model $\model \in \modst$.

Instead, by using the weighting (which depends only on the \emph{true} model $\modst$), we can commpute the lower bound in \Cref{claim:value_lower_bound} with expectations, permitting a weaker condition on the prior. Specifically, combining the above two claims, together with the fact that $\Pr^{\pi_2}_{\modst}\left[ \Evisit \right] \le \rho_0$ for $\pi_2 \in \Pi_{\ell}^c$ and the bound $\epsr \le \epsp$ by definition, we achieve the following synthesis:
\begin{claim}[Gap for Good Models]\label{claim:value_diff} Let $\model \in \modgoodl$, and $\pi \in \Pimarkov$. Then,
\begin{align*}
\valuef{\pi}{\model} - \max_{\pi_2 \in \Pimarkov \setminus \Pi_{\ell}}\valuef{\pi_2}{\model}  \ge \sum_{(x,a,h) \in \calU_{\ell}} \sfr_{\model}\xah \cdot \omegast^{\pi}\xah -H \left(\rho_0 + \epspunish  + 2H\epsp\right).
\end{align*}
\end{claim}
Now recall from \Cref{lem:hallucination_good_event} that, on the high-probability good event $\Egoodl$, the model $\modvarhat$ sampled from the posterior lies in the good set $\modgoodl$ with $(1-\epspunish)$-probability. Hence, we can convert \Cref{claim:value_diff} into the following guarantee in expectation under $\Exphall$:
\begin{claim}[Gap under Hallucination]\label{claim:value_diff_expectation}
Set $\varepsilon_0 :=  H \left(\rho_0 + 3\epspunish  + 2H\epsp\right)$.  Then, if $\Egoodl$ holds,
\begin{align*}
\Exphall[\valuef{\pi}{\modvarhat} - \max_{\pi' \in  \Pi^c} \valuef{\pi'}{\modvarhat} ] &\ge \textstyle\sum_{(x,a,h) \in \calU_{\ell}} \Exphall[\sfr_{\modvarhat}\xah] \cdot \omegast^{\pi}\xah  - \varepsilon_0.
\end{align*}
In particular, \Cref{claim:omega_star} ensures that if $\calU_{\ell} \cap \reach_{\rho}(\modst)$ is nonempty,
\begin{align}
\max_{\pi \in \Pi_{\ell}}\Exphall[\valuef{\pi}{\modvarhat} - \max_{\pi' \in  \Pi^c} \valuef{\pi'}{\modvarhat} ]\ge  \rho \cdot\min_{\xah \in \calU_{\ell}} \Exphall[\sfr_{\modvarhat}\xah] -  \varepsilon_0.  \label{eq:RHS_value_diff_expecation}
\end{align}
\end{claim}

We are now ready to conclude the proof of \Cref{lemma:gap_bound_prob}:
\begin{proof}[Concluding the proof of \Cref{lemma:gap_bound_prob}] It suffices to establish that the RHS of \Cref{eq:RHS_value_diff_expecation} is at least $\Delta_0 := \rho \ralt/2$, where  $\ralt$ was defined as in \Cref{eq:ralt-defn}. Then,
\begin{align*}
\text{(RHS of \Cref{eq:RHS_value_diff_expecation})} \ge \ralt \rho -  \varepsilon_0 = 2\Delta_0 - H \left(\rho_0 + 3\epspunish  + 2H\epsp\right),
\end{align*}
where we subsituted the definitions of $\varepsilon_0$ and $\Delta_0$ in above. Under the assumptions of the lemma, we immediately have $H (\rho_0 + 3\epspunish) \le 2\Delta_0/3$. Hence, it suffices to establish that, for our given condition on $\nlearn$, $2H^2\epsp \le \Delta_0/3$. Recall that $\epsp:= 2\sqrt{ \frac{2(S\log(5) + \log(1/\delta_0))}{\nlearn}}$. Then, we require $\epsp^2 \le \Delta_0^2/12 H^4$, so that it suffices that $\nlearn \ge \frac{96 H^4((S\log(5) + \log(1/\delta_0))}{\Delta_0}^2$.
\end{proof}


\subsection{Proof of \Cref{lem:hallucination_good_event} \label{sec:proof_hall_good_event}}
The proof of \Cref{lem:hallucination_good_event} builds on the Bayesian concentration technique due to \cite{Selke-PoIE-ec21}. We recall the ``good set'' of models, which we write with an explicit dependence on $\modst$.
\begin{align}
\modgoodl(\modst) := \left\{\model: \forall \xah \in \calU_{\ell}^c \begin{matrix} &\|\sfp_{\model}(\cdot \mid x,a,h) - \sfp_{\modst}(\cdot \mid x,a,h)\|_{\ell_1} \le 2\epsp(\delta_0) \\
&\text{ and }\lonenorm{\sfp_{\model}(\cdot \mid 0) - \pmodst(\cdot \mid 0) } \le 2\epsp(\delta_0)\\
&\text{ and }  \sfr_{\model}\xah \le \epspunish + 2\epsr(\delta_0)\\\end{matrix} \right\}. \label{eq:modgoodl}
\end{align}
Further, recall $\Exphall[f(\modvarhat)] := \Expop\left[\Expmucan[f(\modvarhat) \mid \ledhall] \mid \ledcensl\right]$ and similarly $\Prhall$, whcih are random measures dependening on $\ledcensl$. Our goal is to show that, with probability $1 - \deltafail$ over the randomness of $\modst$ and $\ledcensl$, $\Prhall\left[ \modvarhat\in \modgoodl(\modst)\right]\ge 1 - \epspunish$.

\paragraph{Chernoff Bounds} To start, we define empirical estimators of the rewards, initial state, and transition probabilities:
\begin{definition}\label{defn:estimators} For $(x,a,h) \in \calU_{\ell}^c$, define the estimators $\thetar(x,a,h)$, $\thetap(\cdot \mid x,a,h)$  as the empirical means of the first $\nlearn$ samples from the rewards and transitions at $(x,a,h)$. Specifically, if $\Klearn(x,a,h)$ denotes the set of the first $\nlearn$ hallucination episodes $\kexpl$ at which $(x,a,h) \in \tau_{\kexpl}$, we define
\begin{align*}
\thetar(x,a,h) &:= \textstyle\frac{1}{\nlearn}\sum_{k \in \Klearn(x,a,h)} r_{k;h}\\
\thetap(x' \mid x,a,h) &:= \textstyle\frac{1}{\nlearn}\sum_{k \in \Klearn(x,a,h)} \ind(x_{k;h+1} = x').
\end{align*}
Moreover, for all phases $\ell > \nlearn$, we also define the empirical estimate of the initial state distribution from the first $\nlearn$ samples, $\thetap(\cdot \mid 0)$, via
\begin{align*}
\thetap(x' \mid 0) &:= \textstyle\frac{1}{\nlearn}\sum_{\ell = 1}^{\nlearn} \ind(x_{\kexpl;1} = x').
\end{align*}
\end{definition}
Note that these estimators are not actually used by the algorithm; rather, these estimators are used as a surrogate to reason about Bayesian concentration. A couple additional remarks are in order.
\begin{itemize}
	\item $\theta_r\xah$ and $\theta_{\sfp}\xah$ are undefined for $(x,a,h) \in \calU_{\ell}$
	\item If $(x,a,h) \in \calU_{\ell}$, then $\thetar\xah,\thetap\xah$ remain the same for $\ell' \ge \ell$. In addition, $\thetap(x'\mid 0)$ remains fixed for all $\ell > \nlearn$.
	\item We construct the transition at $(x,a,H)$ to always transition to a terminal state $x_{H+1}$, so that $\thetap(\cdot\mid x,a,H) = \sfp_{\model}(\cdot \mid x,a,H) = \mathrm{dirac}_{x'}$ for any $x,a$.
\end{itemize}
Next, we establish the following \emph{frequentist}  concentration bounds for these estimators. For the remainder of this proof, we let $\epsr(\cdot)$ and $\epsp(\cdot)$ have explicity dependence on the failure probability argument $\delta$. Recall that elsewhere, we use only $\epsr := \epsr(\delta_0)$ and $\epsp := \epsp(\delta_0)$.
\begin{lemma}[Chernoff Concentration Bounds]\label{lem:conc_bounds} Recall the error bounds
\begin{align*}
\epsr(\delta) := \sqrt{ \frac{2\log(1/\delta)}{\nlearn}} \quad \text{ and } \epsp(\delta):= 2\sqrt{ \frac{2(S\log(5) + \log(1/\delta))}{\nlearn}}.
\end{align*}
Then, conditioned on any realization of $\modst$, the estimators $\theta_r(x,a,h)$, $\theta_\sfp(\cdot \mid 0)$, $\theta_{\sfp}(x,a,h)$ defined in \Cref{defn:estimators} sastisfy the following bound for any $\xah \in \calU_{\ell}^c$:
\begin{align*}
&\Pr[\xah \in \calU_{\ell}^c \cap \{|\theta_r(x,a,h) - \rmodst(x,a,h)| \ge \epsr(\delta)\} ] \le \delta \quad \text{and}\\
&\Pr[\xah \in \calU_{\ell}^c \cap \{\|\theta_{\sfp}(x,a,h) - \pmodst(\cdot \mid x,a,h)\|_{\ell_1} \ge \epsp(\delta)\} ] \le \delta
\end{align*}
Moreover, for any $\ell \ge \nlearn$, $\Pr[\|\thetap(\cdot \mid 0) - \pmodst(\cdot \mid 0)\|_{\ell_1} \ge \epsp(\delta)] \le \delta$.
\end{lemma}
The proof of \Cref{lem:conc_bounds} is given at the end of this subsection.

\paragraph{Bounding $\Expop\Prhall[\modvarhat \notin \modgoodl(\modst)] $}
\newcommand{\rmterm}{\mathrm{MainTerm}}
\newcommand{\modsttil}{\tilde{\model}_{\star}}

We consider an intermediate bound on the expectation $\Expop\Prhall[\modvarhat \notin \modgoodl(\modst)] $, which we will ultimately apply as an input to Markov's inquality. We establish
\begin{align}
&\Expop_{\modst,\ledhall}\Prhall[\modvarhat \notin \modgoodl(\modst)] \nonumber\\
&\qquad\le \frac{1}{\qpunish} \Expop_{\ledhonl}\Exp_{\modst, \modst' \iidsim  \Pr[\modst = \cdot \mid \ledhonl] } \left[\ind\{\modst \in \modclass_{\ell}(\epspunish) \text{ and } \modst' \notin \modgoodl(\modst) \}\right] \label{eq:expop},
\end{align}
where we use $\Expop_{\modst,\ledhall}$ to make clear that the randomness arises from $\modst$ and $\ledhall$. We render $\EvPun$ explicitly as $\EvPun = \{\modvarhat \in \modclass_{\ell}\}$ (resp. $\EvPun = \{\modst \in \modclass_{\ell}\}$), where we define the punishing model class
\begin{align}\label{eq:modclass_hall}
\modclasshall(\epspunish) :=
    \cbr{ \model \in \modtotal:
    \sfr_{\model}\xah \le \epspunish
    \;\;\text{for all fully-explored $\xah$ triples}
    }.
\end{align}
We shall also write $\modgoodl = \modgoodl(\modst)$ to elucidate the dependence of the set on $\modst$.
With the above notation, we write
\begin{align*}
&\Expop_{\modst,\ledhall}\Prhall[\modvarhat \notin \modgoodl(\modst)]  \\
&= \Expop_{\modst,\ledhall}\left[\Exp[\Prmucan[\modvarhat \notin \modgoodl(\modst) \mid \ledhall] \mid \ledcensl]\right]\\
&\overset{(i)}{=}\Expop_{\modst,\ledhall}\left[\Exp\left[\Prmucan[\modvarhat \notin \modgoodl(\modst) \mid \ledhonl] \mid \ledcensl, \EvPun\right]\right]\\
&= \Expop_{\modst,\ledhall}\frac{\Expop \left[\ind\{\modst \in \modclass_{\ell}(\epspunish)\}\Prmucan[\modvarhat \notin \modgoodl(\modst)\mid \ledhonl] \mid \ledcensl\right]}{\Pr[\modst \in \modclass_{\ell}(\epspunish) \mid \ledcensl]},
\end{align*}
where $(i)$ invokes \Cref{claim:ledhonl_cannonical} to replace the hallucinated ledger with the honest ledger, up to conditioning on $\EvPun$.
Since $\ledcensl$ is hygenic (\Cref{lem:data_hygiene}),  $\Pr[\modst \in \modclass_{\ell}(\epspunish) \mid \ledcensl] = \Prcan[\modst \in \modclass_{\ell}(\epspunish) \mid \ledcensl]$, which is equal to $\Prcan[\EvPun \mid \ledcensl]$, and at most $\qpunish$ by \Cref{rem:reduction-to-independence}. Hence, the above is at most
\begin{align*}
&\Expop_{\modst,\ledhall}\Prhall[\modvarhat \notin \modgoodl(\modst)]  \\
&\le\frac{1}{\qpunish}\Expop_{\modst,\ledhall}\Expop \left[\ind\{\modst \in \modclass_{\ell}(\epspunish)\}\Prmucan[\modvarhat \notin \modgoodl(\modst)\mid \ledhonl] \mid \ledcensl\right]\\
&= \frac{1}{\qpunish}\Expop_{\modst,\ledhonl} \left[\ind\{\modst \in \modclass_{\ell}(\epspunish)\}\Prmucan[\modvarhat \notin \modgoodl(\modst) \mid \ledhonl]\right],
\end{align*}
where in the last line we use the tower rule. By \Cref{lem:data_hygiene}, $\ledhonl$ is hygienic, so that $\Prmucan[\modst \notin \modgoodl(\modst) \mid \ledhonl] = \Prop_{\modst'}[ \modst' \notin \modgoodl(\modst) \mid \ledhonl]$, where $\modst'$ denotes a random variable with the same distribution as $\Prop[\modst \in \cdot \mid \ledhonl]$.   Hence, we may write
\begin{align*}
&\Expop_{\ledhonl, \modst} \left[\ind\{\modst \in \modclass_{\ell}(\epspunish)\}\Prmucan[\modvarhat \notin \modgoodl(\modst) \mid \ledhonl]\right] \\
&= \Expop_{\ledhonl}\Exp_{\modst, \modst' \iidsim  \Pr[\modst = \cdot \mid \ledhonl] } \left[\ind\{\modst \in \modclass_{\ell}(\epspunish)\}\ind\{\modst' \notin \modgoodl(\modst)\}\right]\\
&= \Expop_{\ledhonl}\Exp_{\modst, \modst' \iidsim  \Pr[\modst = \cdot \mid \ledhonl] } \left[\ind\{\modst \in \modclass_{\ell}(\epspunish) \text{ and } \modst' \notin \modgoodl(\modst) \}\right],
\end{align*}
where we represent the internal conditional probability$\Pr[\modst \notin \modgoodl \mid \ledhonl]$ as an expectation over an independt draw of $\modst'$ form the posterior given $\ledhonl$. This establishes \Cref{eq:expop}.

\paragraph{Concluding the Bayesian Chernoff Bound} Recall the estimators $\theta_r,\theta_{\sfp}$ from \Cref{defn:estimators}. These are entirely determined by the data in the honest ledger $\ledhonl$. Moreover by the triangle inequality, we have the inclusion of events
\begin{align*}
\{\modst \in \modclasshall  \text{ and } \modst' \notin \modgoodl(\modst)\} \subseteq \calE_{\mathrm{conc}}(\modst;\ledhonl) \cup  \calE_{\mathrm{conc}}(\modst';\ledhonl)
\end{align*}
where for $\model \in \{\modst,\modst'\}$, we define the concentration event
\begin{align*}
\calE_{\mathrm{conc}}(\model;\ledhonl) = \left\{\exists \xah \in \calU_{\ell} : \begin{matrix}  &\|\sfp_{\model}(\cdot \mid x,a,h) - \theta_{\sfp}(\cdot \mid x,a,h) \|_{\ell_1} \le \epsp(\delta_0) \\
&\text{ or }\lonenorm{\sfp_{\model}(\cdot \mid 0) - \theta_{\sfp}(\cdot \mid 0) } \le \epsp(\delta_0)\\
&\text{ or }  |\sfr_{\model}\xah - \theta_r\xah| \le  \epsr(\delta_0)
\end{matrix} \right\}.
\end{align*}
Hence,
\begin{align*}
&\Expop_{\ledhonl}\Exp_{\modst, \modst' \iidsim  \Pr[\modst = \cdot \mid \ledhonl] } \left[\ind\{\modst \in \modclass_{\ell}(\epspunish) \text{ and } \modst' \notin \modgoodl(\modst) \}\right]\\
&\le \Expop_{ \ledhonl}\Prop_{\modst,\modst' \iidsim \Prcan[\cdot \mid \ledhonl]}\left[\calE_{\mathrm{conc}}(\modst;\ledhonl) \cup  \calE_{\mathrm{conc}}(\modst';\ledhonl) \right]\\
 &\le 2\Prop_{ \ledhonl,\modst}[\calE_{\mathrm{conc}}(\modst;\ledhonl) ].
\end{align*}
where the last step uses a union bound, and the fact that $\modst,\modst'$ have the same distribution. Finally, by \Cref{lem:conc_bounds} and a union bound, we have that $\Prop_{ \ledhonl,\modst}[\calE_{\mathrm{conc}}(\modst;\ledhonl) ] \le 2SAH \delta_0$. Hence, retracing our steps
\begin{align}
\Expop_{\ledcensl}\Prhall[\modvarhat \notin \modgoodl(\modst)]  \le \frac{2\Prop_{ \ledhonl,\modst}[\calE_{\mathrm{conc}}(\modst;\ledhonl) ]}{\qpunish}=  \frac{4SAH \delta_0}{\qpunish}. \label{eq:delta_not_fin}
\end{align}
To conclude, recall the event (over $\ledcensl$) $\Egoodl := \{\Prhall[\modvarhat \notin \modgoodl] \le \epspunish\}$. Then, by a Markov's inequality argument and \Cref{eq:delta_not_fin},
\begin{align*}
\Prop_{\modst,\ledcensl}[\Egoodl] &= \Exp_{\modst,\ledcensl}\ind\{\Prhall[\modvarhat \notin \modgoodl(\modst)] \le \epspunish\}\\
&\le \Exp_{\modst,\ledcensl}\frac{\Prhall[\modvarhat \notin \modgoodl(\modst)]}{\epspunish} \le \frac{4SAH \delta_0}{\epspunish \qpunish}.
\end{align*}
 Choosing $\delta_0 := \frac{\deltafail\epspunish \qpunish}{4SAH}$ concludes.
\qed

\begin{proof}[Proof of \Cref{lem:conc_bounds}] Let us begin the proof assuming that the event $\{\xah \in \calU_{\ell}^c\}$ holds with probability one. We show how to remove this restriction at the end of the proof.

We begin with the concentration bound of $\thetar$.  Fix an $\xah$, and let $k_1,\dots,k_{\nlearn}$ denote the first $\nlearn$ hallucination episodes $k$ on which $\xah \in \traj_k$. Then, if $(\calF_{i})$ denote the filtration under which $\calF_i$ contains all information up the rollout of trajectory $\traj_{k_i}$, we see that $Z_i= r_{k_i;h} - \rmodst(x,a,h)$ is a martingale with respect to $(\calF_i)$, and $|Z_i| \le 1$. Hence, by the Azuma-Hoeffding inequality and a union bound over signs of the error,
	\begin{align*}
	|\thetar(x,a,h) - \rmodst(x,a,h)| = \left|\frac{1}{\nlearn} \sum_{i=1}^{\nlearn} Z_i\right| \le \sqrt{ \frac{2\log(2/\delta)}{\nlearn}}, \text{ with probability } 1-\delta.
	\end{align*}
	Recognizing the above error bound as $\epsr(\delta)$ concludes. Next, let us adress $\thetap(\xah)$. By Holder's inequality,
	\begin{align*}
	\lonenorm{\thetap(\cdot \mid \xah) - \pmodst(\cdot \mid \xah)} &= \max_{v\in [-1,1]^S} \langle v, \thetap(\cdot \mid \xah) - \pmodst(\cdot \mid \xah)\rangle.
	\end{align*}
	Now, let $\calN := \{-1,-1/2,0,1/2,1\}^S$ be a covering over $[-1,1]$. Then, any $v \in [-1,1]^S$ can be expressed as $v = v_1 + v_2$, where $v_1 \in \calN$, and $v_2 \in [-1/2,1/2]^S$. Hence,
	\begin{align}
	&\lonenorm{\thetap(\cdot \mid \xah) - \pmodst(\cdot \mid \xah)} \\
	&\quad \le \max_{v_1\in \calN} \langle v_1, \thetap(\cdot \mid \xah) \nonumber - \pmodst(\cdot \mid \xah)\rangle + \max_{v_2\in [-1/2,1/2]} \langle v_1, \thetap(\cdot \mid \xah) - \pmodst(\cdot \mid \xah)\rangle \nonumber\\
	&\qquad=  \max_{v\in \calN} \langle v, \thetap(\cdot \mid \xah) - \pmodst(\cdot \mid \xah)\rangle + \frac{1}{2}\lonenorm{\thetap(\cdot \mid \xah) - \pmodst(\cdot \mid \xah)}. \label{eq:cover_eq}
	\end{align}
	Rearranging, $\lonenorm{\thetap(\cdot \mid \xah) - \pmodst(\cdot \mid \xah)}  \le 2 \max_{v\in \calN} \langle v, \thetap(\cdot \mid \xah) - \pmodst(\cdot \mid \xah)\rangle$.
	Now, let $(\calF_i)$ and $k_i$ be as above. For each $v \in \calN$, let $v(x')$ denote its coordinates, and define $W_{i,v}= \sum_{x'\in S} v(x') \cdot (\ind(x_{k_i;h+1} =x') - \pmodst(x' \mid x,a,h))$. Then, we see that $W_{i,v}$ form a martingale with respect to $(\calF_i)$, and
	\begin{align*}
	\langle v, \thetap(\cdot \mid \xah) - \pmodst(\cdot \mid \xah)\rangle = \frac{1}{\nlearn}\sum_{i=1}^{\nlearn} W_{i,v}.
	\end{align*}
	We can also verify that $|W_{i,v}| \le \sum_{x'\in S} |v(x')| |(I(x_{k_i;h+1} =x') - \pmodst(x' \mid x,a,h)| \le\sum_{x'\in S} |v(x')| = 1 $.  Thus, again by Azuma-Hoeffding, we have that with probability $1 - \delta$,
	\begin{align*}
	\langle v, \thetap(\cdot \mid \xah) - \pmodst(\cdot \mid \xah)\rangle \le \sqrt{ \frac{2\log(1/\delta)}{\nlearn}}.
	\end{align*}
	By a union bound over all $v \in \calN$, which has $|\calN| = 5^{S}$, we have that with probability $1 - \delta$,
	\begin{align*}
	\langle v, \thetap(\cdot \mid \xah) - \pmodst(\cdot \mid \xah)\rangle \le \sqrt{ \frac{2(S\log(5) + \log(1/\delta))}{\nlearn}}.
	\end{align*}
	And thus, from $\Cref{eq:cover_eq}$, it holds that with probability $1- \delta$,
	\begin{align*}
	\lonenorm{\thetap(\cdot \mid \xah) - \pmodst(\cdot \mid \xah)} \le 2\sqrt{ \frac{2(S\log(5) + \log(1/\delta))}{\nlearn}} := \epsp(\delta).
	\end{align*}
	The argument for $\thetap(\cdot \mid 0)$ is analogous.
	\end{proof}

\paragraph{Handling the randomness of the event $\{\xah \in \calU_{\ell}^c\}$} To conclude, we address that the event $\{\xah \in \calU_{\ell}^c\}$ is random, and hence, $\thetar$ or $\thetap$ are well-defined is random. Here we explain why our reasoning above still remains valid. For simplicity, we explain how to modify the reasoning for $\thetar$; adjusting $\thetap$ is the same.

Let $n_{\ell}\xah$ denote the number of times triple $\xah$ is visited by phase $\ell$, and let $\calK_{\ell}\xah$ denote the \emph{epxloration} episodes on which $\xah \in \traj_{k}$ is visited. We extend the definition of $\thetar$ to cases where $\{\xah \notin \calU_{\ell}^c\}$ by defining
\begin{align*}
\tilde{\theta}_{\sfr}\xah = \begin{cases} \thetar\xah \text{ as in   \Cref{defn:estimators}} & \xah \in \calU_{\ell}^c \\
\frac{1}{\nlearn}\left(\sum_{k\in \calK_{\ell}\xah}r_{k;h} + \sum_{i=1}^{\nlearn - n_{\ell}\xah} \tilde{r}_i\right),  \tilde{r}_i \iidsim \mathsf{R}_{\modst}\xah & \text{otherwise}
\end{cases}
\end{align*}
In other words, when $\xah \in \calU_{\ell}$, one draws an additional $\nlearn - n_{\ell}\xah$ rewards from the reward distribution conditioned on $\modst$, i.e. $\mathsf{R}_{\modst}\xah $, and uses these to complete the estimator $\tilde{\theta}_{\sfr}\xah$. By construction, whenever $\xah \in \calU_{\ell}^c$, $\tilde{\theta}_{\sfr} = \theta_{\sfr}$, and hence
\begin{multline}
\Pr[\xah \in \calU_{\ell}^c \cap \{|\theta_r(x,a,h) - \rmodst(x,a,h)| \ge \epsr(\delta)\} ] \\
 = \Pr[\xah \in \calU_{\ell}^c \cap \{|\tilde{\theta}_r(x,a,h) - \rmodst(x,a,h)| \ge \epsr(\delta)\} ] \\
 \le \Pr[|\tilde{\theta}_r(x,a,h) - \rmodst(x,a,h)| \ge \epsr(\delta) ] \label{eq:in_view_of}
\end{multline}
Hence, it suffices to reason about the concentration of $\tilde{\theta}_r$. Moreover, $\tilde{\theta}_r$ satisfies the same concentration inequality we derived above for $\theta_r\xah$, because it also admits a similar martingale decomposition (with bounded increments). Hence, we can apply the above argument to reason about the concentration of $\tilde{\theta}_r$, and use this to reason about the concentration of $\thetar$ in view of \Cref{eq:in_view_of}.

\subsection{Proof of  \Cref{lem:visitation_comparison_general} \label{proof:visitation_comparison_general}}

	\newcommand{\barmod}{\bar{\model}}
	\newcommand{\barmodst}{\bar{\model}_{\star}}
	\newcommand{\rbar}{\bar{r}}
	\newcommand{\calUbar}{\bar{\calU}}

	{\textbf{The `in-particular' clause.}} Before diving into the proof, we establish the special case depicted in the ``in particular'' clause. This is achieved by taking the reward function $\rtil(x,a,h) := \ind\{(x,a,h) \in \calU\}$. Then, we have
	\begin{align*}
	\sum_{h=1}^H \rtil(\bmx_h,\bma_h,h) \ind\{\scrE_h\} &=  \sum_{h=1}^H  \ind\{(\bmx_h,\bma_h,h) \in \calU \text{ and } (\bmx_\tau,\bma_{\tau},\tau) \in \calU^c,~ \forall \tau < h\}.
	\end{align*}
	Note that these events inside the indicator are all disjoint, and their union is precisely the event that $\{\exists h: (\bmx_h,\bma_h,h) \in \calU\}$, which is precisely $\Evisitnol$.
	\\
	\\
	{\textbf{Main result.}} Next, we turn to the proof of the main lemma, we begin by establishing the special case of Markovian policies. Since we consider a fixed policy $\pi$ for both models $\model,\modst$, for simplicity, we may assume that $\model,\modst$ are Markov reward processes (i.e.  $A = 1$, and thus no policy), suppressing dependences on $\pi$ and $a$. For convenience, we remove the actions. For reasons that will become clear shortly, we also embed into $S+1$ states, and take the reward function $\rbar$ and set $\calU$ as
	\begin{align*}
	&\rbar(x,h) \gets \begin{cases}\rtil(x,\pi(x,h),h)&  x \in [S]\\
	0 & x = S+1 \end{cases}, \\
	&\calUbar \gets \{(x,h): (x,\pi(x,h),h) \in \calU\} \cup \{(S+1,h) : h \in [H]\},\\
	&\scrE_h = \{(\bmx_{\tau},\tau) \in \calUbar^c, ~\forall \tau < h\}
	\end{align*}
	With this setup, it suffices to show
	\begin{align}
	\binom{H}{2}\varepsilon &\ge   \left|\sfE^{\pi}_{\model}\left[\sum_{h=1}^H \rbar(\bmx_h,h) \ind\{\scrE_h\}\right] - \sfE^{\pi}_{\modst}\left[\sum_{h=1}^H \rbar(\bmx_h,h) \ind\{\scrE_h\}\right]\right|, \label{eq:MDP_desired}
	\end{align}
	To establish \Cref{eq:MDP_desired}, we construct reward processes  $\barmod,\barmodst$ over $S+1$ states which absorb the indicators $\ind\{\scrE_h\}$.
	\begin{align}\label{eq:equal_expectation}
	\sfE_{\barmod}\left[\sum_{h=1}^H \rbar(\bmx_h,h)\right] =
	\sfE^{\pi}_{\model}\left[\sum_{h=1}^H \rbar(\bmx_h,h) \ind\{\scrE_h\}\right],
	\end{align}
	and similarly with $\barmodst$ and $\modst$. Let us construct $\barmod$; the construction of $\barmodst$ is indentical.
	Let the initial state distribution be indentical: $\sfp_{\barmod}(x' \mid 0) \equiv \sfp_{\model}(x' \mid 0)$, and set the transition probabilities in $\barmod$ to be the MDP which coincides with $\model$ on states $(x,h) \in \calU^c$, but transitions to state $S+1$ on states $(x,h) \in \calU$:
	\begin{align*}
	\sfp_{\barmod}(x' \mid x,h) = \begin{cases} \sfp_{\model}(x' \mid x,h) & (x,h) \in \calUbar^c \\
	  \ind\{x' = S+1\} & (x,h) \in \calUbar \\
	\end{cases}
	\end{align*}
	We define $\barmodst$ analogously. By construction, and by our assumption on the initial state distribution and the transitions in $\calUbar^c$, we observe that
	\begin{align}
	&\|\sfP_{\barmodst}(\cdot \mid 0) - \sfp_{\barmod}(\cdot \mid 0)\|_{\ell_1} = \|\sfp_{\modst}(\cdot \mid 0) - \sfp_{\model}(\cdot \mid 0)\|_{\ell_1} \le \varepsilon \label{eq:close_init_dist}\\
	&\max_{x,h}\|\sfp_{\barmodst}(\cdot \mid x,h) - \sfp_{\barmod}(\cdot \mid x,h)\|_{\ell_1} = \max_{(x,h) \in \calUbar^c}\|\sfp_{\modst}(\cdot \mid x,h) - \sfp_{\model}(\cdot \mid x,h)\|_{\ell_1} \le \varepsilon.\label{eq:close_trans_dist}
	\end{align}
	We now verify that our construction satisfies \Cref{eq:equal_expectation}.

	\begin{claim}\label{claim:expected_reward_identity} For $\barmod,\model$ above, \Cref{eq:equal_expectation} holds, and similarly for $\barmodst,\modst$.
	\end{claim}
	\begin{proof}

		\newcommand{\barbmx}{\bar{\bmx}}
		We establish the equality for $\barmod,\model$. Let $(\bmx_1,1),(\bmx_2,2),\dots,(\bmx_H,H) \sim \sfP_{\model}$. Introduce the coupled sequence $\barbmx_1,\barbmx_2,\dots,\barbmx_H$ via
		\begin{align*}
		\barbmx_h := \begin{cases} S+ 1& (\bmx_{h-1},h-1) \in \calUbar \text{ or } \barbmx_{h-1} = S+1\\
		\bmx_h & \text{otherwise}\end{cases}
		\end{align*}
		We immediately see that $(\barbmx_1,1),(\barbmx_2,2),\dots,(\barbmx_H,H) \sim \sfP_{\barmod}$. Moreover, observe that,
		\begin{itemize}
			\item On the joint space on which the $(\bmx_h,\barbmx_h)$ sequence is defined, $\scrE_h^c = \{\exists \tau < h: (\bmx_{\tau},\tau) \in \calU\} =\{ \barbmx_h = S+1\}$
			\item If $\barbmx_h \ne S+1$, then $\barbmx_h = \bmx_h$.
		\end{itemize}
		Thus, letting $\sfE$ denote the law of the coupled sequences,
		\begin{align}
		\sfE\left[\sum_{h=1}^H \rbar(\barbmx_h,h)\right] &= \sfE\left[\sum_{h=1}^H \rbar(\barbmx_h,h)(1 - \ind\{ \barbmx_h = S+1\})\right] \tag{$\rbar(S+1,\cdot) \equiv 0$ by construction}\\
		&=\sfE\left[\sum_{h=1}^H \rbar(\barbmx_h,h)\ind\{ \barbmx_h \ne S+1\})\right] \nn\\
		&=\sfE\left[\sum_{h=1}^H \rbar(\bmx_h,h)\ind\{ \barbmx_h \ne S+1\})\right] \tag{if $\barbmx_h \ne S+1$, then $\barbmx_h = \bmx_h$}\\
		&= \sfE\left[\sum_{h=1}^H \rbar(\bmx_h,h)\ind\{\scrE_h\}\right] \tag{$\scrE_h^c = \{ \barbmx_h = S+1\}$}.
		\end{align}
		Since the coupled distribution of $\bmx_{1:H}$ and $\barbmx_{1:H}$ under $\sfE$ has marginals $\bmx_{1:H} \sim \sfP_{\model}$ and $\barbmx_{1:H} \sim \sfP_{\barmod}$, the indentity follows.
	\end{proof}

	To conclude, we invoke the ubiquitous performance difference lemma (see e.g. \citep[Lemma 5.3.1]{kakade2003sample}), specialized to Markov reward processes:
	\begin{lemma}[Performance Difference Lemma for MRPs]\label{lem:perf_diff} Let $\model_1,\model_2$ be two MRPs with state space $[S']$ and horizon $H$, and common reward function $r$. Define the value functions $V^{\model_i}_h(x) := \Exp^{\model_i}[\sum_{\tau=h}^{H} r(\bmx_h,h)  \mid \bmx_h = x]$. Then,
	\begin{align*}
	&\sfE_{\model_1}\left[\sum_{h=1}^H r(\bmx_h,h)\right] - \sfE_{\model_2}\left[\sum_{h=1}^H r(\bmx_h,h)\right] \\
	&\quad=  (\sfp_{\model_1}(\cdot \mid 0) - \sfp_{\model_2}(\cdot \mid 0)^\top V^{\model_2}_{1}(\cdot)  +  \sfE_{\model_1}\left[\sum_{h=1}^H (\sfp_{\model_1}(\cdot \mid \bmx_h,h) - \sfp_{\model_2}(\cdot \mid \bmx_h,h))^\top V^{\model_2}_{h+1}(\cdot)\right]
	\end{align*}
	\end{lemma}
	Applying \Cref{lem:perf_diff} and ,  we have
	\begin{align*}
	&\left|\sfE_{\barmod}\left[\sum_{h=1}^H \rbar(\bmx_h,h)\right] - \sfE_{\barmodst}\left[\sum_{h=1}^H \rbar(\bmx_h,h)\right]\right|\\
	&= \left|\sfp_{\barmod}(\cdot \mid 0) - \sfp_{\barmodst}(\cdot \mid 0 )^\top V^{\barmodst}_{h+1}(\cdot) + \sfE_{\barmod}\left[\sum_{h=1}^H (\sfp_{\barmod}(\cdot \mid \bmx_h,h) - \sfp_{\barmodst}(\cdot \mid \bmx_h,h))^\top V^{\barmodst}_{h+1}(\cdot)\right]\right|\\
	&\le \|\sfp_{\barmod}(\cdot \mid 0) - \sfp_{\barmodst}(\cdot \mid 0 )\|_{\ell_1} \max_{x \in [S+1]} |V^{\barmodst}_{1}(\cdot)| +
	\sfE_{\barmod}\left[\sum_{h=1}^H \|\sfp_{\barmod}(\cdot \mid \bmx_h,h) - \sfp_{\barmodst}(\cdot \mid \bmx_h,h)\|_{\ell_1} \cdot \max_{x \in [S+1]} |V^{\barmodst}_{h+1}(\cdot)|\right]\\
	&\le \varepsilon \sum_{h=1}^{H+1} \max_{x \in [S+1]} |V^{\barmodst}_{h+1}(\cdot)|,
	\end{align*}
	where the last inequality uses \Cref{eq:close_trans_dist,eq:close_init_dist}. Now, since the rewards $\rbar$ lie in $[0,1]$,we have $\max_{x \in [S+1]} |V^{\barmodst}_{h}(\cdot)| \le 1 + H - h$. Hence, $\sum_{h=1}^{H+1} \max_{x \in [S+1]} |V^{\barmodst}_{h+1}(\cdot)| \le \binom{H}{2}$, yielding
	\begin{align*}
	\binom{H}{2}\varepsilon &\ge  \left|\sfE_{\barmod}\left[\sum_{h=1}^H \rbar(\bmx_h,h)\right] - \sfE_{\barmodst}\left[\sum_{h=1}^H \rbar(\bmx_h,h)\right]\right| \\
	&= \left|\sfE^{\pi}_{\model}\left[\sum_{h=1}^H \rbar(\bmx_h,h) \ind\{\scrE_h\}\right] - \sfE^{\pi}_{\modst}\left[\sum_{h=1}^H \rbar(\bmx_h,h) \ind\{\scrE_h\}\right]\right|,
	\end{align*}
	where the last step uses \Cref{claim:expected_reward_identity}, thereby proving \Cref{eq:MDP_desired}.
	\qed

\subsection{Proof of \Cref{claim:value_upper_bound,claim:omega_star,claim:value_lower_bound,claim:value_diff,claim:value_diff_expectation} \label{sec:proof_of_prob_claims}}
	\begin{proof}[Proof of \Cref{claim:value_upper_bound} ]

		Recalling the $\sfP$-measurable event $\Evisit = \{\exists h  : (\bmx_h,\bma_h,h) \in \calU_{\ell}\}$, we have that for any $\pi \in \Pimarkov$ and $\model \in \modgoodl$,
		\begin{align*}
		\valuef{\pi}{\model} &= \sfE^{\pi}_{\model}\left[\sum_{h=1}^H \sfr_{\model}(\bmx_h,\bma_h,h) \right] \\
		&\le H \Pr^{\pi}_{\model}\left[ \Evisit \right] + H \max_{\xah \in \calU_{\ell}^c} \sfr_{\model}\xah \nn\\
		&\le H \Pr^{\pi}_{\model}\left[ \Evisit \right] + H (2\epsr + \epspunish) \tag*{(\Cref{eq:modgoodl})}\\
		&\le H \Pr^{\pi}_{\modst}\left[ \Evisit \right] + H (2\epsr + \epspunish) + \binom{H}{2}\cdot 2\epsp \tag*{(\Cref{lem:visitation_comparison_general})},
		\end{align*}
		where the last line uses \Cref{lem:visitation_comparison_general}
		 with $\calU \gets \calU_{\ell}$ and the $\ell_1$ bound of $2\epsp$ on difference in transitions and intial state probabilities from the definition of $\modgoodl$ in \Cref{eq:modgoodl}.
	\end{proof}

	\begin{proof}[Proof of \Cref{claim:omega_star}] First Point: the events $\{(\bmx_h,\bma_{h}) = (x,a) \text{ and } (\bmx_\tau,\bma_{\tau},\tau)\in \calU_{\ell}^c, \quad \forall \tau < h\}$ are disjoint, and $\Evisit$ holds precisely if and only if at least one of these events holds. Note that this  does not use any specific properties of $\calU_{\ell}$. Second Point: if $\calU_{\ell} \cap \reach_{\rho}(\modst)$ is non-empty, then there exists a triple $(x,a,h) \in \calU_{\ell}$ and a policy $\pi$ for which policy $\pi$ for which $\sfP_{\modst}^{\model}[(\bmx_h, \bma_h,h) = (x,a,h)] \ge \rho$. Moreover,
\begin{align*}
\rho &\le \sfP_{\modst}^{\pi}[(\bmx_h, \bma_h,h) = (x,a,h)] \\
&= \sfP_{\modst}^{\pi}[(\bmx_h, \bma_h,h) = (x,a,h) \text{ and } \forall \tau < h, (x,a,\tau) \in \calU_{\ell}^c]  \\
&\qquad+ \sfP_{\modst}^{\pi}[(\bmx_h, \bma_h,h) = (x,a,\tau) \text{ and } \exists \tau < h :  (x,a,h) \in \calU_{\ell}] \tag*{(total probability)}\\
&\le \sfP_{\modst}^{\pi}[\exists h' \ge h: (x,a,h')\in \calU_{\ell} \text{ and } \forall \tau < h, (x,a,\tau) \in \calU_{\ell}^c]  + \sfP_{\modst}^{\pi}[\exists \tau < h :  (x,a,\tau) \in \calU_{\ell}]\\
&= \sfP_{\modst}^{\pi}[\exists \tau  :  (x,a,\tau) \in \calU_{\ell}] := \sfP_{\modst}^{\pi}[\Evisit]\tag*{(disjoint union)}
\end{align*}
In particular, $\pi \in \Pi_{\ell}$, since $\sfP_{\modst}^{\pi}[\Evisit] \ge \rho \ge \rho_0$. The second part of the claim now follows from the first.
\end{proof}

	\begin{proof}[Proof of \Cref{claim:value_lower_bound}]

		Let $\model \in \modgoodl$. Adopting the notation of \Cref{lem:visitation_comparison_general} with $\calU \gets \calU_\ell$, set $\scrE_h := \{(\bmx_\tau,\bma_{\tau},\tau) \in \calU_\ell^c,~ \forall \tau < h\}$. Since the rewards are non-negative, we lower bound the value by a sum over rewards times indicators of $\scrE_h$, and invoke \Cref{lem:visitation_comparison_general} with $\rtil = \sfr_{\model}$ together with the definition of $\modgoodl$ in \Cref{eq:modgoodl}:
		\begin{align*}
		\valuef{\pi}{\model} &= \sfE^{\pi}_{\model}\left[\sum_{h=1}^H \sfr_{\model}\xah \right]\\
		&\ge \sfE^{\pi}_{\model}\left[\sum_{h=1}^H \sfr_{\model}\xah \ind(\scrE_h)\right]\\
		&\ge  \sfE^{\pi}_{\modst}\left[\sum_{h=1}^H \sfr_{\model}\xah \ind(\scrE_h)\right] - \binom{H}{2}\cdot 2\epsp
		\end{align*}
		Again, by nonnegativity of the rewards $\sfE^{\pi}_{\modst}\left[\sum_{h=1}^H \sfr_{\model}\xah \ind(\scrE_h)\right] = \sum_{x,a,h} \sfr_{\model}\xah \cdot \omegast\xah \ge \sum_{(x,a,h) \in \calU_{\ell}} \sfr_{\model}\xah \cdot \omegast\xah$. The bound bound follows.
	\end{proof}

	\begin{proof}[Proof of \Cref{claim:value_diff}]

		Let $\pi_1 \in \Pimarkov$, and $\pi_2 \in \Pi_{\ell}^c$. Fix a $\model \in \modgoodl$. By definition of $\Pi_{\ell}^c$, $\Pr^{\pi_2}_{\modst}\left[ \Evisit \right] \le \rho_0$,  \Cref{claim:value_upper_bound} ensures $\valuef{\pi_2}{\model} \le H \rho_0 + H \epspunish + 2 \epsr + H(H-1)\epsp$. On the other hand, \Cref{claim:value_lower_bound} ensures $\valuef{\pi_1}{\model}  \ge \sum_{(x,a,h) \in \calU_{\ell}} \sfr_{\model}\xah \cdot \omegast^{\pi_1}\xah - H(H-1)\epsp$, Hence,
		\begin{align*}
		&\valuef{\pi_1}{\model} - \valuef{\pi_2}{\model} \\
		&\quad\ge \sum_{(x,a,h) \in \calU_{\ell}} \sfr_{\model}\xah \cdot \omegast^{\pi_1}\xah -H \rho_0 - H \epspunish  - 2 \epsr -  2H(H-1)\epsp\\
		&\quad\ge \sum_{(x,a,h) \in \calU_{\ell}} \sfr_{\model}\xah \cdot \omegast^{\pi_1}\xah -H \left(\rho_0 + \epspunish  + 2H\epsp\right),
		\end{align*}
		where we  use that $\epsp \le \epsr$.
	\end{proof}

	\begin{proof}		To streamline the proof, we condense notation. Introduce the shorthand
		\begin{align*}
		Z^{\pi}_{\model} := \valuef{\pi}{\model} - \max_{\pi' \in \Pi_{\ell}^c}\valuef{\pi'}{\model} , \quad \text{and} \quad W^{\pi}_{\model} := \sum_{(x,a,h) \in \calU_{\ell}}\sfr_{\model}\xah \cdot \omegast^{\pi}\xah
		\end{align*}
		 and let $\varepsilon := H \left(\rho_0 + \epspunish  + 2H\epsp\right)$. Finally, recall the shorthand $\Exp_{\mathrm{hal}}[\cdot] := \Expsimhh \Expcan[\cdot \mid \ledhall]$ and set $\Pr_{\mathrm{hal}}[\cdot] := \Expsimhh \Prcan[\cdot \mid \ledhall]$. Now, \Cref{claim:value_diff} implies that, for all $\model \in \modgoodl$, $Z^{\pi}_{\model} \ge W^{\pi}_{\model} - \varepsilon$.  Since $Z^{\pi}_{\model} \in [-H,H]$ and $W^{\pi}_{\model} \in [0,H]$, we have
		\begin{align*}
		\Exphall[\valuef{\pi}{\modvarhat} - \max_{\pi' \in  \Pi^c} \valuef{\pi'}{\modvarhat} ] &= \Exphall[ Z^{\pi}_{\modvarhat} ]\\
		&\ge \Exphall[ \ind\{\modvarhat \in \modgoodl \} Z^{\pi}_{\modvarhat}  ] - H \Exphall[ \modvarhat \notin \modgoodl ] \\
		&\ge \Exphall[ \ind\{\modvarhat \in \modgoodl \} (W^{\pi}_{\modvarhat} - \varepsilon)  ] - H \Exphall[ \modvarhat \notin \modgoodl    ] \\
		&\ge \Exphall[  W^{\pi}_{\modvarhat} ] - \varepsilon - 2H \Exphall[ \modvarhat \notin \modgoodl  ].
		\end{align*}
		By \Cref{lem:hallucination_good_event}, we have that on $\Egoodl$, $\Exphall[ \modvarhat \notin \modgoodl  ] \le \epspunish$. Moreover, by linearity of expectation, and the fact that $\Exphall[\cdot]$ treats $\modst$ as deterministic,
		\begin{align*}
		\Exphall[  W^{\pi}_{\modvarhat} ] &= \sum_{(x,a,h) \in \calU_{\ell}} \Exphall[\sfr_{\model}\xah] \cdot \omegast^{\pi}\xah.
		\end{align*}
		Putting things together, we conclude,
		\begin{align*}
		&\Exphall[\valuef{\pi}{\modvarhat} - \max_{\pi' \in  \Pi^c} \valuef{\pi'}{\modvarhat} ] \\
		&\qquad\ge\sum_{(x,a,h) \in \calU_{\ell}} \Exphall[\sfr_{\model}\xah] \cdot \omegast^{\pi}\xah.  - 2\epspunish - \varepsilon \\
		&\qquad= \sum_{(x,a,h) \in \calU_{\ell}} \Exphall[\sfr_{\model}\xah] \cdot \omegast^{\pi}\xah  - H \left(\rho_0 + 3\epspunish  + 2H\epsp\right).
		\end{align*}
		Finally, we can lower bound
		\begin{align*}
		\sum_{(x,a,h) \in \calU_{\ell}} \Exphall[\sfr_{\model}\xah]  &\ge \min_{\xah \in \calU_{\ell}} \Exphall[\sfr_{\model}\xah] \cdot \sum_{(x,a,h) \in \calU_{\ell}} \omegast^{\pi}\xah\\
		&= \sfP_{\modst}^{\pi} \cdot\min_{\xah \in \calU_{\ell}} \Exphall[\sfr_{\model}\xah],
		\end{align*}
		where the last equality uses the indentity \Cref{claim:value_lower_bound}. Hence, setting $\varepsilon_0 :=  H \left(\rho_0 + 3\epspunish  + 2H\epsp\right)$, we have
		\begin{align*}
		\Exphall[\valuef{\pi}{\modvarhat} - \max_{\pi' \in  \Pi^c} \valuef{\pi'}{\modvarhat} ] &\ge \sum_{(x,a,h) \in \calU_{\ell}} \Exphall[\sfr_{\model}\xah] \cdot \omegast^{\pi}\xah  - \varepsilon_0\\
		&\ge \sfP_{\modst}^{\pi} \cdot\min_{\xah \in \calU_{\ell}} \Exphall[\sfr_{\model}\xah] -  \varepsilon_0. \qquad\qedhere
		\end{align*}
	\end{proof}

\subsection{Proof of \Cref{thm:main_prob_mdp} from \Cref{lem:main_prob_lemma} \label{proof:main_prob_mdp}}
	\newcommand{\bbar}{\widebar{b}}
	\newcommand{\Efinl}[1][\ell]{\calE_{\mathrm{fin};#1}}

	Let us first recall the relevant parameters and assumptions. Recall the per-phase failure probability $\deltafail$, set $\Delta_0 = \rho\ralt(\epspunish)/2$,
	\begin{align*}
	\textstyle n_0(\deltafail) := \frac{96 H^4((S\log(5) + \log(1/\delta_0))}{\Delta_0^2}, \quad \text{ where } \delta_0 = \frac{\deltafail \cdot \qpunish \cdot \epspunish}{4SAH}.
	\end{align*}
	Recall the exploration event $\Eexplorel := \{\exists h: (x\kh,a\kh,h) \in \calU_{\ell}, ~k := \kexpl\}$. Define the ``finishing'' event $\Efinl := \{\calU_\ell \cap \reach_{\rho}(\modst)\}$, on which we have successfully $(\rho,\nlearn)$ explored by phase $\ell$. We also define
	 Bernoull random variables $B_{\ell} := \ind\{\Eexplorel\}$, which are $\calF_{\ell}$-measurable, and their deviations from their conditional expectations $\bbar_{\ell} :=\Exp[B_{\ell} \mid \calF_{\ell}]$. From \Cref{lem:main_prob_lemma}, it thus holds that
	 \begin{align}
	 \Pr[ \bbar_{\ell} \le \rhoprog \text{ and not } \Efinl] \le 1 - \deltafail, \quad\text{where }\rhoprog := \frac{\Delta_0^2}{6H^2} = \frac{\rho^2\ralt(\epspunish)^2}{36H^2} \label{eq:main_prob_lemma_guar}
	\end{align}

	Now, suppose that for a given phase $L \ge 0$, we have that $\calU_{L+1} \cap \reach_{\rho}(\modst) \ne \emptyset$ (i.e., $\Efinl[L+1]$ fails). The same must be be true for all $\ell \le L+1$ since $\calU_{\ell}$ is non-decreasing across phases $\ell$. Hence, by a union bound
	\begin{align*}
	&\Pr[ \text{not } \Efinl[L+1] \text{ and } \sum_{\ell=\nlearn + 1}^{L} \bbar_{\ell} \le (L-\nlearn) \rhoprog] \\
	&\quad\le (L-\nlearn) \max_{\ell \in \{\nlearn+1,\dots,L\}}\Pr[  \text{not } \Efinl[L+1] \text{ and } \bbar_{\ell} \le  \rhoprog]\\
	&\quad\le (L-\nlearn) \max_{\ell \in \{\nlearn+1,\dots,L\}}\Pr[ \text{not } \Efinl[\ell] \cap \reach_{\rho}(\modst) \ne \emptyset \text{ and } \bbar_{\ell} \le  \rhoprog]\\
	&\quad \le (L-\nlearn) \deltafail \tag*{by \Cref{eq:main_prob_lemma_guar}}
	\end{align*}
	On the other hand, by the pidgeonhole principle, it must be the case that
	\begin{align}
	\sum_{\ell=1}^L B_{\ell} \le SAH \nlearn, \label{eq:B_l_bound}
	\end{align}
	since each time $B_{\ell} = 1$, one triple $\xah \in \calU_{\ell}$ is visited during a hallucination episode, and each triple $\xah \in \calU_{\ell}$ can only be visited a maximum of $\nlearn$ times during hallucination episodes. Hence, if we consider phase $L_0 := 4SAH\nlearn/\rhoprog$ and $L_1 = L_0 + \nlearn$ (incrementing $L_0$ by $\nlearn$, so $L_0 = L_1 - \nlearn$), and we have
	\begin{align*}
	\Pr[ \text{not } \Efinl[L_1+1] ] &= \Pr[ \text{not } \Efinl[L_1+1] \text{ and } \sum_{\ell=\nlearn + 1}^{L_1} \bbar_{\ell} \le L_0 \rhoprog] + \Pr[ \text{not } \Efinl[L+1] \text{ and } \sum_{\ell=\nlearn +1}^{ L_1} \bbar_{\ell} > L_0 \rhoprog]\\
	&\le L_0 \deltafail + \Pr\left[  \sum_{\ell=\nlearn+1}^{L_1} \bbar_{\ell} > L_0 \rhoprog\right]\\
	&= L_0 \deltafail + \Pr\left[  \sum_{\ell=1+\nlearn}^{L_1} \bbar_{\ell} > L_0 \rhoprog \text{ and } \sum_{\ell=1+\nlearn}^{L_1} B_{\ell} \le SAH \nlearn \right] \tag{\Cref{eq:B_l_bound}} \\
	&= L_0 \deltafail + \Pr\left[  \sum_{\ell=\nlearn+1}^{L_1} \bbar_{\ell} >  4SAH\nlearn \text{ and } \sum_{\ell=\nlearn+1}^{L_1} B_{\ell} \le SAH \nlearn  \right] \tag{Definition of $L_0$}.
	\end{align*}
	To conclude, let us use a concentration inequality to bound the probability on the final display. Applying a standard martingale Chernoff bound (Lemma F.4 in \cite{dann2017unifying}, with $X_\ell \gets B_{\ell}$, $P_{\ell} \gets \bbar_{\ell}$, and $W =  SAH \nlearn$) yields that
	\begin{align*}
	\Pr\left[  \sum_{\ell=\nlearn+1}^{L_1} \bbar_{\ell} >  4SAH\nlearn \text{ and } \sum_{\ell=\nlearn+1}^{L_1} B_{\ell} \le SAH \nlearn  \right] \le e^{-W} = e^{- SAH \nlearn},
	\end{align*}
	and thus, for $L_0 := 4SAH\nlearn/\rhoprog = 36SAH^3\nlearn/\Delta_0^2$, and $\nlearn = L_0 + \nlearn \le 37SAH^3\nlearn/\Delta_0^2$
	\begin{align}
	\Pr[ \text{not } \Efinl[L_1+1] ] \le L_0 \deltafail  + e^{- SAH \nlearn},
	\end{align}
	To conclude, we note that when $\Efinl[L_0+1]$, then we have $(\rho,\nlearn)$ explored by episode $K_0 = L_1 \cdot \nphase = 4SAH\nlearn \nphase = \frac{37SAH^3 \nlearn\nphase}{\Delta_0^2}$.

	\newcommand{\deltapun}{\delta_{\mathrm{pun}}}
	To conclude, let us select conditions on $\nlearn$ for which $L_0 \deltafail  + e^{- SAH \nlearn}$ is bounded by the target failure probability $\delta$. For this it suffices to choose $\deltafail$ such $\deltafail = \delta/2L_0$, and ensure $\nlearn \ge \log(2/\delta)$. This only affects our bound through the requirement $\nlearn \ge n_0(\delta_0) \vee \log(2/\delta)$, where
	\begin{align*}
	\delta_0 &= \frac{\delta \cdot \qpunish \cdot \epspunish}{4SAH} = \frac{\delta\cdot \qpunish \cdot \epspunish}{8SAHL_0} = \frac{\rhoprog\delta \cdot \qpunish \cdot \epspunish}{(4SAH)^2\nlearn}\\
	&= \frac{\frac{\Delta_0^2}{4H^2}\delta \cdot \qpunish \cdot \epspunish}{(4SAH)^2\nlearn}=  \frac{\delta \cdot \qpunish \cdot \epspunish}{(8SAH^2)^2\nlearn}\\
	&\le \frac{\delta}{\nlearn} \Delta_0^2 \cdot \delta_1^2, \text{ where } \delta_1 := \frac{\qpunish \cdot \epspunish}{24SAH}.
	\end{align*}
	 Hence, this requires
	\begin{align*}
	\nlearn \ge \frac{96 H^4 \log(\nlearn \cdot \frac{5^S}{\delta \cdot \delta_1^2 \Delta_0^2})}{\Delta_0^2}, \text{ or } \frac{\Delta_0^2}{96 H^4} \ge  \frac{\log(\nlearn \cdot \frac{5^S}{\delta \cdot \delta_1^2 \Delta_0^2})}{\nlearn}
	\end{align*}
	We appeal to the following claim to invert the logarithm
	\begin{claim} For all positive $a,b,t$ with $a/b \ge e$, $t \ge 2 \log(a/b)/b$ implies that $b \ge \log(a t)/t$.
	\end{claim}
	In particular, take $t = \nlearn$, $a = \frac{5^S}{\delta \cdot \delta_1^2 \Delta_0^2}$, and $b = \frac{\Delta_0^2}{96 H^4}$. Then, $a/b \ge e$, so it suffices that
	\begin{align*}
	\nlearn \ge \frac{192 H^4 \log(\frac{5^S 96 H^4}{\delta \cdot \delta_1^2 \Delta_0^4 })}{\Delta_0^2} \vee \log(2/\delta).
	\end{align*}
	Finally, we simplify
	\begin{align*}
	\log(\frac{5^S 96 H^4}{\delta \cdot \delta_1^2 \Delta_0^4 })  &= \log \frac{5^S 96 H^4 \cdot 16 \cdot 64 S^2A^2H^4  }{\delta \cdot (\qpunish \cdot \epspunish)^2 \rho^4 \ralt(\epspunish)^4 }\\
	&= S\log 5 + \log(1/\delta) +  \log \frac{96 H^4 \cdot 16 \cdot 64 S^2A^2H^4  }{\delta \cdot (\qpunish \cdot \epspunish)^2 \rho^4 \ralt(\epspunish)^4 }\\
	&\le S\log 5 + \log(1/\delta) +  4\log \frac{20 SAH^2 }{\rho \cdot \qpunish \cdot \epspunish  \ralt(\epspunish) }\\
	&:= S\log 5 + \log(1/\delta) +  \iota(\epspunish,\rho).
	\end{align*}
	Thus, it is enough to ensure $\nlearn \ge \frac{192 H^4 (S\log 5 + \log(1/\delta) +  \iota(\epspunish,\rho))}{\Delta_0^2} \vee \log(2/\delta)$. One can verify that the first term dominates the $\log(2/\delta)$.
	\qed

\newpage

\section{Proof of \Cref{prop:revelation}}
\label{app:exploit}

\newcommand{\calV}{\mathcal{V}}
\newcommand{\Vvis}{\calV_{\mathrm{vis}}}
\newcommand{\Vreach}{\calV_{\mathrm{rea}}}
\newcommand{\modtil}{\tilde{\model}}
\newcommand{\valuefrest}[2]{\valuename_{\mathrm{rstr}}(#1;#2)}

In this appendix, we restate and prove \Cref{prop:revelation}.
\proprevelation*

At a high-level, the proof requires two steps. First, we show that, with high probability, any model $\model$ drawn from the posterior given the signal $\signal$ has similar rewards and transitions to those of $\modst$ under all triples which are $\rho$-reachable under $\modst$. This step invokes a Bayesian Chernoff argument similar in spirit to those in \Cref{sec:proof_hall_good_event}. In the second step, we argue that similarity on $\rho$-reachable triples implies uniformly that for all Markovian policies $\pi$, $\valuef{\pi}{\model}$ and $\valuef{\pi}{\modst}$ are close by. As a consequence, we conclude that any BIC policy (one which optimizes $\Exp[\valuef{\pi}{\modst} \mid \signal]$) must be near optimal for $\modst$.

\begin{remark} Because assume the ledger $\ledger$ contain \emph{all} trajectories collected by the mechanism, all posteriors are cannonical.
\end{remark}

\paragraph{Preliminaries.}
We begin with a remark on notation. In the majority of the paper, we were concerned with sets of \emph{undervisited} triples $\xah$, notated $\calU$. In this section of the paper, we are concered more with sets of triples $\calV$ which we wish are sufficiently visited, or which we  wish to be so. The two sets of interest are
\begin{definition}[Reachable and Visited Set] Given $\rho > 0$ and a model $\modst$, $\Vreach(\rho,\modst)$ denote the sets of triples $\xah$ which are  $\rho$-reachable under $\modst$. We define $\Vreach(\modst)$ as the set of all reachabile triples $\xah$ for any postive $\rho$; i.e. $\Vreach(\modst) = \bigcup_{\rho > 0}\Vreach(\rho,\modst)$. \footnote{In other words, $\Vreach(>0,\modst)$ is the compliment of the set of triples which cannot be reached by \emph{any} policy under $\modst$.} Given $n \ge 0$ and ledger $\ledger_K$, we let $\Vvis(n)$ denote the set of triples $\xah$ which have been visited at least $n$ times in ledger $k$.
\end{definition}

We now recall the definition of transition-similiarity, modified to be stated in terms of $\calV$-sets.
\begin{restatable}[Transition-Similar]{definition}{defnsimilarV} \label{defn:similarity_V} Let $\|\cdot\|_{\ell_1}$ denote the $\ell_1$-distance between probability distributions.  Given $\calV \subset [S] \times [A] \times [H]$, we say two models $(\model,\modst)$ are $\varepsilon$-transition-similar on $\calV$ if (i) $\|\sfp_{\model}(\cdot \mid 0 ) - \sfp_{\modst}(\cdot \mid 0)\|_{\ell_1} \le \varepsilon $ \emph{(closeness of initial state distribution)}, and (ii) for each $(x,a,h) \in \calV$, $\|\sfp_{\model}(\cdot \mid x,a,h) - \sfp_{\modst}(\cdot\mid x,a,h)\|_{\ell_1} \le \varepsilon$ \emph{(closeness of transitions on $\calV$)}.
\end{restatable}

We also introduce the analogous notion of \emph{reward similarity}
\begin{definition}[Reward Similar]\label{defn:similarity_R}  Given $\calV \subset [S] \times [A] \times [H]$, we say two models $(\model,\modst)$ are $\varepsilon$-reward-similar on $\calV$ if for each $(x,a,h) \in \calV$, $|\sfr_{\model}(x,a,h) - \sfr_{\modst}( x,a,h)| \le \varepsilon$.
\end{definition}

\paragraph{Bayesian Concentration}
We begin by arguing that there exist accurate estimators $\thetar$ and $\thetap$ of the rewards and transitions which are well defined for all states visited at least $n$ times. The following is a modification of \Cref{lem:conc_bounds}, whose proof is similar and omitted in the interest of brevity.
\begin{lemma}[Chernoff Concentration Bounds]\label{lem:conc_bounds_two} Given an $n \ge 0$ and let $\ledger_K$ be an uncensored ledger containing at least $K \ge n$ trajectories. Define the error bounds
\begin{align*}
\textstyle \epsr(n,\delta) := \sqrt{ \frac{2\log(1/\delta)}{n}}, \quad \text{ and } \epsp(n,\delta):= 2\sqrt{ \frac{2(S\log(5) + \log(1/\delta))}{n}}.
\end{align*}
Then, there exist estimators $\theta_r(x,a,h)$, $\theta_{\sfp}(x,a,h)$, $\theta_\sfp(\cdot \mid 0)$,  of the rewards,  transition probabilities, initial state distribution, which are functions of the ledger $\ledger_K$  such that
\begin{align*}
&\Pr[\xah \in \Vreach(n) \cap \{|\theta_r(x,a,h) - \rmodst(x,a,h)| \ge \epsr(n,\delta)\} ] \le \delta \quad \text{and}\\
&\Pr[\xah \in \Vreach(n) \cap \{\|\theta_{\sfp}(x,a,h) - \pmodst(\cdot \mid x,a,h)\|_{\ell_1} \ge \epsp(n,\delta)\} ] \le \delta\\
&\Pr[\|\thetap(\cdot \mid 0) - \pmodst(\cdot \mid 0)\|_{\ell_1} \ge \epsp(n,\delta)] \le \delta.
\end{align*}
\end{lemma}
We now invoke the Bayesian concentration argument due to \cite{Selke-PoIE-ec21}. Let $\model'$ be a drawn from the posterior conditioned on $\ledger$, $\model' \sim \Pr[\cdot \mid \ledger]$. Then, $(\model',\ledger)$ and $(\modst,\ledger)$ have the same distribution. Hence, the estimators $\theta_r(\cdot)$ and $\theta_p(\cdot)$ (a function of only the ledger $\ledger$) also concentrate around $\sfr_{\model'}$ and $\sfp_{\model'}$, in the sense of \Cref{lem:conc_bounds_two}. Thus, by unions bounds and applications of the triangle inequality, it holds that
\begin{align*}
&\Pr[\xah \in \Vreach(n)\cap \{|\sfr_{\model'}(x,a,h) - \rmodst(x,a,h)| \ge 2\epsr(n,\delta)\} ] \le 2\delta \quad \text{and}\\
&\Pr[\xah \in \Vreach(n) \cap \{\|\sfp_{\model'}(x,a,h) - \pmodst(\cdot \mid x,a,h)\|_{\ell_1} \ge 2\epsp(n,\delta)\} ] \le 2\delta\\
&\Pr[\|\sfp_{\model'}(\cdot \mid 0) - \pmodst(\cdot \mid 0)\|_{\ell_1} \ge 2\epsp(n,\delta)] \le 2\delta.
\end{align*}
Recalling the definitions of transition- and reward-similarity, another union bound yields the following lemma:
\newcommand{\Esim}{\mathcal{E}_{\mathrm{sim}}}
\newcommand{\Ebarsim}{\bar{\mathcal{E}}_{\mathrm{sim}}}
\begin{lemma}\label{lem:Esim} Let $\modst$ denote the true model, and consider a sample $\model' \sim \Pr[\cdot \mid \ledger]$. Then for any $\delta \in (0,1)$, the following event $\Esim(\delta)$ holds with probability $1 - 6SAH\delta$ over all randomness in $(\modst,\model',\ledger_K)$:
\begin{align}
\Esim(n,\delta) := \left\{(\modst,\model') \text{ are }\,\, \left\{\begin{matrix} 2\epsr(n,\delta)\text{-reward-similar and}\\
2\epsp(n,\delta)\text{-transition-similar}\\
\end{matrix}\right\} \text{ on } \Vreach(n).\right\}
\end{align}
\end{lemma}

Lastly, we convert \Cref{lem:Esim} into a slightly more useful form to reason about sampling from the posterior conditioned on a fixed true model $\modst$ and ledger $\ledger_K$.
\begin{lemma}\label{lem:Esim_two} Fix a $\delta_1,\delta_2 \in (0,1)$, and define the event
\begin{align*}
\Ebarsim(n,\delta_1,\delta_2) := \{\Pr[\Esim(n,\delta_1) \mid \ledger_K,\modst] \ge 1 - \delta_2 \}
\end{align*}
Then, $\Pr[\Ebarsim(\delta_1,\delta_2)] \ge 1 - \frac{6SAH\delta_1}{\delta_2}$.
\end{lemma}
\begin{proof} We apply Markov's inequality.
\begin{align*}
\Pr[\Esim(\delta_1)^c] &= \Exp[\Pr[\Esim(\delta_1)^c \mid \ledger_K,\modst] ]\\
&\ge  \Exp[\delta_2 \cdot\ind\{\Pr[\Esim(\delta_1)^c \mid \ledger_K,\modst] \ge - \delta_2\}]\\
&:=  \Exp[\delta_2 \cdot \ind\{\Ebarsim(\delta_1,\delta_2)^c\}]\\
&= \delta_2 \cdot\Pr[\Ebarsim(\delta_1,\delta_2)^c].
\end{align*}
From \Cref{lem:Esim}, we know that $\Pr[\Esim(\delta_1)^c] \le 6\delta_1$. Therefore, $\Pr[\Ebarsim(\delta_1,\delta_2)^c] \le 6\delta_1/\delta_2$. The bound follows.
\end{proof}

\paragraph{Similarity implies close values.} With the above preliminaries in place, we first show that the simulation lemma which states that if two models $(\model,\modst)$ are both transition-similiar and reward-similar on $\Vreach(\rho,\modst)$, then \emph{all} policies have similar value.

Our first step is to show that the set of non-$\rho$-reachable triples $\Vreach(\rho,\modst)^c$ for $\modst$ is hard to reach under any $\model$ which is $\epsp$-transition-similar to $\modst$ on the $\rho$-reachable triples $\Vreach(\rho,\modst)$.
\begin{lemma}\label{lem:mass_of_reachable} Let $(\model,\modst)$ be $\epsp$-similar on $\Vreach(\rho,\modst)$. Then for policy $\pi \in \Pimarkov$,
\begin{align*}
\sfP^\pi_{\modst}[\exists (\bmx_h,\bma_h,h) \in \Vreach(\rho,\modst)^c] &\le \rho SH.\\
\sfP^\pi_{\model}[\exists h : (\bmx_h,\bma_h,h) \in \Vreach(\rho,\modst)^c] &\le \rho SH + H^2\epsp.
\end{align*}
\end{lemma}
\begin{proof} This first inequality follows from a union bound,
\begin{align*}
\sfP^\pi_{\modst}[\exists h : (\bmx_h,\bma_h,h) \in \Vreach(\rho,\modst)^c] &\le \sum_{x,a,h} \ind\{a = \pi(x,h)\} \ind\{(x,a,h) \in \Vreach(\rho,\modst)^c\}\sfP^\pi_{\modst}[ (\bmx_h,\bma_h,h) = (x,a,h)]\\
&\le \sum_{x,a,h} \ind\{a = \pi(x,h)\} \ind\{(x,a,h) \in \Vreach(\rho,\modst)^c\} \rho\\
&\le \rho SH.
\end{align*}
Moreover, by $\epsp$-similarity on $\Vreach(\rho,\modst)$, \Cref{lem:visitation_comparison_general} entails the following inequality, which proves our desired bound:
\begin{align*}
\left|\sfP^\pi_{\modst}[\exists h : (\bmx_h,\bma_h,h) \in \Vreach(\rho,\modst)^c] - \sfP^\pi_{\model}[\exists h : (\bmx_h,\bma_h,h) \in \Vreach(\rho,\modst)^c]\right| \le \binom{H}{2}\epsp \le H^2 \epsp. \quad\qedhere
\end{align*}
\end{proof}

\newcommand{\indreach}{\ind_{\mathrm{reach}}}
Using the above, we establish closeness of values:
\begin{lemma}[Simulation on Reachable Set]\label{lem:reachable_simulation} Fix $\rho > 0$, and suppose that  $(\model,\modst)$ are  $\epsp$-transition-similar and $\epsr$-reward-similar on $\Vreach(\rho,\modst)$. Introduce the indicatior
\begin{align*}
\indreach :=\begin{cases}0 & \Vreach(\rho,\modst) = \Vreach(\modst) = \Vreach(\model)\\
1 & \text{otherwise},
\end{cases}
\end{align*}
which is equal to $1$ unless the set of $\rho$-reachable triples under $\modst$ coincide with the set of reachable (for any $\rho'$) triples under either $\modst$ or $\model$).  Then, for any $\pi \in \Pimarkov$,
\begin{align}
\left|\valuef{\pi}{\model} - \valuef{\pi}{\modst}\right| \le  H^2 S \rho \cdot\indreach  + 2 H^3 \epsp + H\epsr. \label{eq:first_bound_reach}
\end{align}
\end{lemma}
\begin{proof}

Our strategy is to invoke the simulation lemma, \Cref{lem:visitation_comparison_general}, twice.  For $h \in [H]$, introduce the $\sfP$-events $\scrE_h := \{(\bmx_\tau,\bma_{\tau},\tau) \in \Vreach(\rho,\modst),~ \forall \tau < h\}$, and let us define the \emph{restricted value function} for any model $\model'$ via
\begin{align}
\valuefrest{\pi}{\model'} := \sfE^{\pi}_{\model'}\left[\sum_{h=1}^H \sfr_{\model'}(\bmx_h,\bma_h,h) \ind\{\scrE_{h+1}\}\right],
\end{align}
which only counts rewards accumulated on trajectories which remain on $\rho$-reachable states $\Vreach(\rho,\modst)$ up until time and \emph{including} step $h$. We can observe then that, for any model $\model'$,
\begin{align*}
\valuefrest{\pi}{\model'} \le \valuef{\pi}{\model'} &\le \valuefrest{\pi}{\model'}  +  H\sfP^{\pi}_{\model'}[\exists h: (\bmx_h,\bma_h,h) \in \Vreach(\rho,\modst)^c ]\\
&= \valuefrest{\pi}{\model'}  +  H\sfP^{\pi}_{\model'}[\exists h: (\bmx_h,\bma_h,h) \in \Vreach(\rho,\modst)^c ] \cdot \indreach,
\end{align*}
where we can multiply by the indicator $\indreach$ since if $\indreach = 0$,  $\sfP^{\pi}_{\model'}[\exists h: (\bmx_h,\bma_h,h) \in \Vreach(\rho,\modst)^c ]  = 0$. In light of \Cref{lem:mass_of_reachable}, we find then that $\modst$ and $\model$ satisfy,
\begin{align*}
\valuefrest{\pi}{\modst} \le \valuef{\pi}{\modst} \le \valuefrest{\pi}{\modst}  +  H^2 S \rho \cdot \indreach \\
\valuefrest{\pi}{\model} \le \valuef{\pi}{\model} \le \valuefrest{\pi}{\model}  +  H^2 S \rho \cdot \indreach + H^3 \epsp.
\end{align*}
Together, these bounds imply that
\begin{align}\label{eq:value_diff_inter_rest}
 |\valuef{\pi}{\model}  - \valuef{\pi}{\modst}| \le  H^2 S \rho \cdot \indreach + H^3 \epsp + |\valuefrest{\pi}{\modst} - \valuefrest{\pi}{\model}|.
\end{align}

It remains to bound the difference in restricted values $|\valuefrest{\pi}{\modst} - \valuefrest{\pi}{\model}|$. To do so, introduce an interpolating model $\modtil$ whose transitions $\modtil$ are the same as those in $\model$ ($\sfp_{\modtil} = \sfp_{\model}$), but whose rewards are the same as those in $\modst$ ($\sfr_{\modtil} = \sfr_{\modst}$). By the triangle inequality and \Cref{eq:value_diff_inter_rest},
\begin{equation}\label{eq:pennultimate_value_bound}
\begin{aligned}
 &|\valuef{\pi}{\model}  - \valuef{\pi}{\modst}| \le  H^2 S \rho \cdot \indreach + H^3 \epsp \\
 &\qquad+ \underbrace{|\valuefrest{\pi}{\modst} - \valuefrest{\pi}{\modtil}|}_{(i)} + \underbrace{|\valuefrest{\pi}{\model} - \valuefrest{\pi}{\modtil}|}_{(ii)}.
\end{aligned}
\end{equation}
To bound term $(i)$, define the event $\bar{\scrE}_{h} := \ind\{(x,a,h) \in \Vreach(\rho,\modst)\}$, so that $\scrE_{h+1} = \bar{\scrE}_{h} \cap \scrE_{h}$. Defining the reward $\rtil(x,a,h) := \ind\{\bar{\scrE}_h\}\sfr_{\modst}\xah$, we then have
\begin{align}
\ind\{\scrE_h\}\rtil(x,a,h) = \ind\{\scrE_h\}\ind\{\bar{\scrE}_h\}\sfr_{\modst}\xah = \ind\{\scrE_{h+1}\}\sfr_{\modst},
\end{align}
and since $\modst$ and $\modtil$ have the same reward function $\ind\{\scrE_h\}\rtil(x,a,h)= \sfr_{\modtil}(\bmx_h,\bma_h,h) \ind\{\scrE_{h+1}\}$. Therefore,
\begin{align*}
\valuefrest{\pi}{\modst} - \valuefrest{\pi}{\modtil} &= \sfE^{\pi}_{\modst}\left[\sum_{h=1}^H \sfr_{\modst}(\bmx_h,\bma_h,h) \ind\{\scrE_{h+1}\}\right] - \sfE^{\pi}_{\modtil}\left[\sum_{h=1}^H \sfr_{\modtil}(\bmx_h,\bma_h,h) \ind\{\scrE_{h+1}\}\right]\\
 &= \sfE^{\pi}_{\modst}\left[\sum_{h=1}^H \rtil(\bmx_h,\bma_h,h) \ind\{\scrE_{h}\}\right] - \sfE^{\pi}_{\modtil}\left[\sum_{h=1}^H \rtil(\bmx_h,\bma_h,h) \ind\{\scrE_{h}\}\right].
\end{align*}
Hence, the difference $\valuefrest{\pi}{\modst} - \valuefrest{\pi}{\modtil}$ takes the form of precisely the quantity bounded by \Cref{lem:visitation_comparison_general}, implying that (with the crude bound $\binom{H}{2} \le H^2$)
\begin{align}
(i) \le |\valuefrest{\pi}{\modst} - \valuefrest{\pi}{\modtil}| \le H^2 \epsp.
\end{align}
Next, we bound term $(ii)$, which, due to the fact that $\modtil$ and $\model$ have the ssame transitions, (and again $\modtil$ has the same rewards as $\modst$) takes the form
\begin{align*}
\valuefrest{\pi}{\model} - \valuefrest{\pi}{\modtil} &= \sfE^{\pi}_{\model}\left[\sum_{h=1}^H \sfr_{\model}(\bmx_h,\bma_h,h) \ind\{\scrE_{h+1}\}\right] - \sfE^{\pi}_{\modtil}\left[\sum_{h=1}^H \sfr_{\modtil}(\bmx_h,\bma_h,h) \ind\{\scrE_{h+1}\}\right]\\
 &=\sfE^{\pi}_{\model}\left[\sum_{h=1}^H (\sfr_{\model}(\bmx_h,\bma_h,h) - \sfr_{\modst}(\bmx_h,\bma_h,h))\ind\{\scrE_{h+1}\}\right].
\end{align*}
We further observe that $\scrE_{h+1} = 0$ unless $(\bmx_h,\bma_h,h) \in \Vreach(\rho,\modst)$, and when this occurs, $\epsr$-reward-similarly implies that $|(\sfr_{\model}(\bmx_h,\bma_h,h) - \sfr_{\modst}(\bmx_h,\bma_h,h))| \le \epsr$. Hence,
\begin{align*}
|\valuefrest{\pi}{\model} - \valuefrest{\pi}{\modtil}| \le  \sfE^{\pi}_{\model}\left[\sum_{h=1}^H \ind\{\scrE_{h+1}\}\epsr\right] \le H\epsr.
\end{align*}
In summary, we have bounded term $(i)$ by $H^2 \epsp$ and term $(ii)$ by $H\epsr$. From \Cref{eq:pennultimate_value_bound},
\begin{align*}
 &|\valuef{\pi}{\model}  - \valuef{\pi}{\modst}| &\le  H^2 S \rho \cdot \indreach+ H^3 \epsp + H \epsr + H^2 \epsp \\
 &\le H^2 S \rho \cdot \indreach + 2 H^3 \epsp + H\epsr. \qquad\qedhere
\end{align*}
\end{proof}

\paragraph{Concluding the proof. }
\newcommand{\Etrav}{\mathcal{E}_{\mathrm{trave}}}
To conclude, suppose that that the following two events hold:
\begin{itemize}
\item The true model $\modst$ has be $(\rho,n)$-traversed, i.e. $\Etrav := \{\Vvis(n) \supset \Vreach(\rho,\modst)\}$ holds.
\item Recall the event $\Ebarsim(n,\delta_1,\delta_2)$ from \Cref{lem:Esim_two}. In words, this is the event that, with probability $1 - \delta_2$, a draw $\model' \sim \Pr[\cdot \mid \ledger_K,\modst]$ is both $2\epsr(n,\delta_1)$-reward-similar and $2\epsp(n,\delta_1)$-transition-similar to $\modst$ on the visited triples $\Vvis(n)$. We take $\delta_2 = 1/n$ and $\delta_1 = \delta_2\cdot \delta/6SAH = \epsilon \delta/6SAH n$, where $\delta \in (0,1)$ is our target failure probability.
\end{itemize}
Observe that
\begin{align}
\Pr[\Etrav \cap \Ebarsim(n,\delta_1,\delta_2)] \ge  1 - \Pr[\Etrav] - \Pr[\Ebarsim] \overset{(i)}{\ge} 1 - \delta - \frac{6SAHn \delta_1}{\delta_2} \overset{(ii)}{\ge} 1 - 2\delta,
\end{align}
where $(i)$ uses that our algorithm satisfies $(\rho,n,\delta,K_0)$-\traversal and \Cref{lem:Esim_two}, and $(ii)$ replaes our chose of $\delta_1,\delta_2$.

To conclude, we assume $\Etrav \cap \Ebarsim(n,\delta_1,\delta_2)$ holds, and bound  $ |\valuef{\hat{\pi}}{\modst} - \valuef{\pi_{\star}}{\modst}|$,  where $\hat{\pi} \in \argmax_{\pi \in \Pimarkov}\Exp'[\valuef{\pi}{\model'} ]$, and $\pi_{\star} \in \argmax_{\pi \in \Pimarkov}\Exp[\valuef{\pi}{\modst} ]$.

To this end, $\Pr'[\cdot], \Exp'[\cdot]$ denote a shorthand expectation over a model $\model' \sim \Pr[\modst \in \cdot \mid \ledger_K]$ (treating $\modst$ and $\ledger_K$ as fixed).  Then on their intersection $\Etrav \cap \Ebarsim(n,\delta_1,\delta_2)$,
\begin{align}\label{eq:pEst}
\calE_{\star} := \left\{ \begin{matrix} \model' \text{ and }\modst \text{ are } 2\epsr(n,\delta_1)\text{-reward-similar and }\\
 2\epsp(n,\delta_1)\text{-transition-similar} \text{ on } \Vreach(\rho,\modst)\end{matrix}\right\} \text{ has } \Pr'[\calE_{\star}] \ge 1 - \delta_2.
\end{align}
On $\calE_{\star}$, \Cref{lem:reachable_simulation} implies that for any policy $\pi \in \Pi$,
\begin{equation}\label{eq:eprime_diff}
\begin{aligned}
|\valuef{\pi}{\model'}  - \valuef{\pi}{\modst}| &\le H^2 S \rho + 4 H^3 \epsp(n,\delta_1) + 2H\epsr(n,\delta_1) \\
&\le H^2 S \rho + 6H^3 \epsp(n,\delta_1),
\end{aligned}
\end{equation}
where we use $\epsp(n,\delta_1) \ge \epsr(n,\delta_1)$ as defined above.

Therefore, since $\Pr'[\calE_{\star}] \ge 1 - \delta_2$ on $\Etrav \cap \Ebarsim(n,\delta_1,\delta_2)$, it holds that any policy $\pi \in \Pimarkov$ and reachability lower bound $\rho_{\min}$,
\begin{align*}
&|\Exp'[\valuef{\pi}{\model'} ] - \valuef{\pi}{\modst}| \\
&= |\Exp'[\valuef{\pi}{\model'} - \valuef{\pi}{\modst}]| \\
&\le \Exp'[|\valuef{\pi}{\model'} - \valuef{\pi}{\modst}|]\\
&\le H\Exp'[\ind\{\calE_\star\}] + H\Exp'[\ind\{\calE_{\star}\}|\valuef{\pi}{\model'} - \valuef{\pi}{\modst}|]\\
&\overset{(i)}{\le}  H\delta_2 + H^2 S \rho \cdot \ind_{\{\rho > \rho_{\min}\}}+ 6H^3 \epsp(n,\delta_1)\\
&\overset{(ii)}{\le}H^2 S \rho \cdot \ind_{\{\rho > \rho_{\min}\}}+ 7 H^3 \epsp(n,\delta_1) := \bar{\varepsilon}.
\end{align*}
where $(i)$ last step uses \Cref{eq:pEst,eq:eprime_diff}, together with the fact that if $\rho_{\min}$ is a reachability bound, then for any $\rho < \rho_{\min}$, $\Vreach(\rho,\modst) = \Vreach(\modst) = \Vreach(\model')$ for all $\modst,\model'$ in the support of the prior $\prior$. In addition, $(ii)$ uses that $H\delta_2 = H/n \le 7 H^3 \epsp(n,\delta_1)$.

In particular, if $\hat{\pi} \in \argmax_{\pi \in \Pimarkov}\Exp'[\valuef{\pi}{\model'} ]$, and $\pi_{\star} \in \argmax_{\pi \in \Pimarkov}\Exp[\valuef{\pi}{\modst} ]$, we conclude that on $\calE_{\star}$.
\begin{align*}
 |\valuef{\hat{\pi}}{\modst} - \valuef{\pi_{\star}}{\modst}| &\le 2\bar{\varepsilon} \\
 &= {\textstyle \mathcal{O}(H^2)\cdot\left(S \rho \cdot \ind_{\{\rho > \rho_{\min}\}}+ H\sqrt{\frac{S + \log(SAHn/\delta)}{n}}\right)}.\qquad\qed
\end{align*}

\end{document}